\documentclass[twoside,11pt]{article}

\usepackage{blindtext}

\usepackage{jmlr2e_arxiv}

\usepackage{lastpage}
\jmlrheading{??}{2026}{1-\pageref{LastPage}}{7/26}{?/??}{??-????}{Cai, Namkoong, Russo, Zhang}

\ShortHeadings{Active Exploration via Autoregressive Generation of Missing Data}{Cai, Namkoong, Russo, Zhang}
\firstpageno{1}

\usepackage[inline]{enumitem}
\usepackage{floatrow}
\usepackage{tikz}
\usetikzlibrary{positioning}
\usepackage{wrapfig}

\newcommand{\cA}{\mathcal{A}}
\newcommand{\cH}{\mathcal{H}}
\newcommand{\what}[1]{\widehat{#1}} %

\newcommand{\Dtrain}{\MC{D}^{\TN{offline}}}

\newcommand{\Ahist}{\MC{A}^{\TN{offline}}}
\newcommand{\Aeval}{\MC{A}}

\newcommand{\psar}{{\pi_{\TN{TSAR}}}}

\usepackage{subcaption}
\usepackage[most]{tcolorbox}

\usepackage{commands}
\usepackage{comment}
\usepackage{algorithm}
\usepackage{algpseudocode}

\newtheorem{assumption}{Assumption} %
\newtheorem{prop}{Proposition}

 \newcommand{\kwz}[1]{
  {\color{violet} [{#1}]} %
 }

\usepackage{amsmath}
\usepackage{array}

\newtcolorbox{softbox}{
    enhanced,
    breakable,
    enhanced jigsaw,
    arc=7pt,
    colback=blue!5,
    colframe=white,
    left skip=1cm,
    right skip=1cm,
    boxrule=0pt,
    before upper={\setlength{\parindent}{0.5em}\everypar{\hspace*{\parindent}\ignorespaces}},
}

\providecommand{\qedsymbol}{\ensuremath{\square}}
\providecommand{\qedhere}{\ifmmode\quad\qedsymbol\else\hfill\qedsymbol\fi}

\usepackage{array}
\newcolumntype{L}[1]{>{\raggedright\arraybackslash}p{#1}}

\begin{document}

\title{Active Exploration via Autoregressive Generation of Missing Data}

\author{\name Tiffany (Tianhui) Cai \email tc3100@columbia.edu \\
       \addr Department of Statistics\\
       Columbia University\\
       New York, NY 10025, USA
       \AND
       \name Hongseok Namkoong \email hn2369@columbia.edu \\
       \addr Decision, Risk, and Operations Division\\
       Columbia Business School\\
       New York, NY 10025, USA
       \AND
       \name Daniel Russo \email djr2174@columbia.edu \\
       \addr Decision, Risk, and Operations Division \\
       Columbia Business School\\
       New York, NY 10025, USA
       \AND
       \name Kelly W. Zhang \email kelly.zhang@imperial.ac.uk \\
       \addr Department of Mathematics\\
       Imperial College London\\
       London SW7 2AZ, United Kingdom
       }
\editor{}
\maketitle

\begin{abstract}%
\noindent We pose uncertainty quantification and exploration in online decision-making as a problem of training and generation from an autoregressive sequence model, an area experiencing rapid innovation. Our approach rests on viewing uncertainty as arising from missing future outcomes that could be revealed through action choices, rather than from unobservable latent parameters of the environment. This reformulation aligns naturally with modern machine learning capabilities: we can i) train generative models through next-outcome prediction rather than fit explicit priors, ii) assess uncertainty through autoregressive generation rather than sampling latent parameters from posteriors, and iii) adapt to new information by extending the sequence model's context rather than explicit posterior updating. 
Our main theoretical result establishes a reduction from online decision-making to offline next-outcome prediction: Bayesian regret is controlled directly by the sequence model's offline prediction loss, without requiring an explicit latent-variable posterior. Experiments, including a semi-synthetic news recommendation task, show that autoregressive generation produces calibrated epistemic uncertainty and enables effective exploration by using article text as prior information to focus exploration on resolving remaining uncertainties.

\end{abstract}


\section{Introduction}
\label{sec:intro}

Foundation models have moved machine learning from modeling concise numerical data to natively handling unstructured data---text, audio, video, sensor streams---that convey richer information and are often more natural to collect. 
Yet these models seem to lack a capability that has long been studied in statistics: 
quantifying epistemic uncertainty, the uncertainty that comes from limited data and would resolve if more data were gathered, as opposed to aleatoric uncertainty, which comes from irreducible noise.

We view foundation models through the lens of the sequence-prediction task that they are often trained on: predicting the distribution of the next outcome given previous ones.  
Even a perfectly accurate next-outcome predictor does not obviously expose epistemic uncertainty on the underlying environment.
In this paper, we explore whether this capability for next-outcome prediction is nevertheless enough to support epistemic uncertainty quantification.
We use a multi-armed-bandit style sequential decision-making setting as a demanding testbed---one that requires not only recognizing remaining uncertainty, but also acting to resolve it and updating beliefs as new information is incrementally gathered. We aim to power this entire loop using capabilities native to foundation models:

\begin{softbox}
\hspace{-5mm} 
Given a sequence model that accurately predicts the next-outcome distribution conditional on past observations,
how can we use it to quantify epistemic uncertainty and drive active exploration?
\end{softbox}

This turns out to be surprisingly subtle. Methods that summarize only the next-step predictive distribution---for instance, by sampling a single next outcome and acting greedily, as proposed in recent work \citep{MullerHoArGrHu22, NguyenGr22, HollmannMuPuKrKoHoScHu25, HegselmannBuLaAgJiSo23}---can suffer linear regret, eventually badly outperformed by simple algorithms with no foundation model access at all (see Section~\ref{subsec:failure_of_other_generation}, Figure~\ref{fig:altSampling}).  The issue is that good decisions require more than predicting the next outcome; this fact has been reinforced by recent work demonstrating the challenges of using large language models (LLMs) for exploration \citep{nie2025evolve,krishnamurthy2024can,harris2025should}. 
 The agent must reason about how future outcomes, once observed, would resolve epistemic uncertainty relevant to later decisions. A one-step predictive distribution marginalizes over both the epistemic and aleatoric uncertainty and therefore cannot drive exploration on its own. How, then, can next-outcome prediction be used to represent the uncertainty that future observations would resolve?

A classical approach to this sequential decision-making problem
posits latent parameters governing aspects of the unknown environment, such as parameters in a reward model. In a Bayesian view, these latents are random variables, and the decision-maker tracks a posterior distribution 
over these latents 
as new data arrives. 
This posterior over the latents represents epistemic uncertainty, i.e. %
uncertainty about the environment that should guide exploration toward observations that would be most informative. 
But maintaining such a posterior is often a critical computational bottleneck when observations are high-dimensional or unstructured, and 
explicit posterior distributions are not a primitive that modern foundation models natively support. %

In the decision-making settings we study, this latent parameter approach %
is also not strictly necessary. The decision-maker ultimately cares about the outcomes that would result from potential future actions---for example, %
which action would have the highest reward on average, or which observations would be informative for future decisions---not about the latent parameter itself. If the relevant future outcomes were exactly known under each possible sequence of actions, %
the decision problem would be resolved. This motivates a missing-data view of uncertainty: rather than represent epistemic uncertainty through an unobserved latent parameter, we represent it through the potentially observable outcomes that have not yet been revealed, in the style of \citet{rubin:1987}. 

\begin{figure}[t]
\centering
     \includegraphics[width=\textwidth]{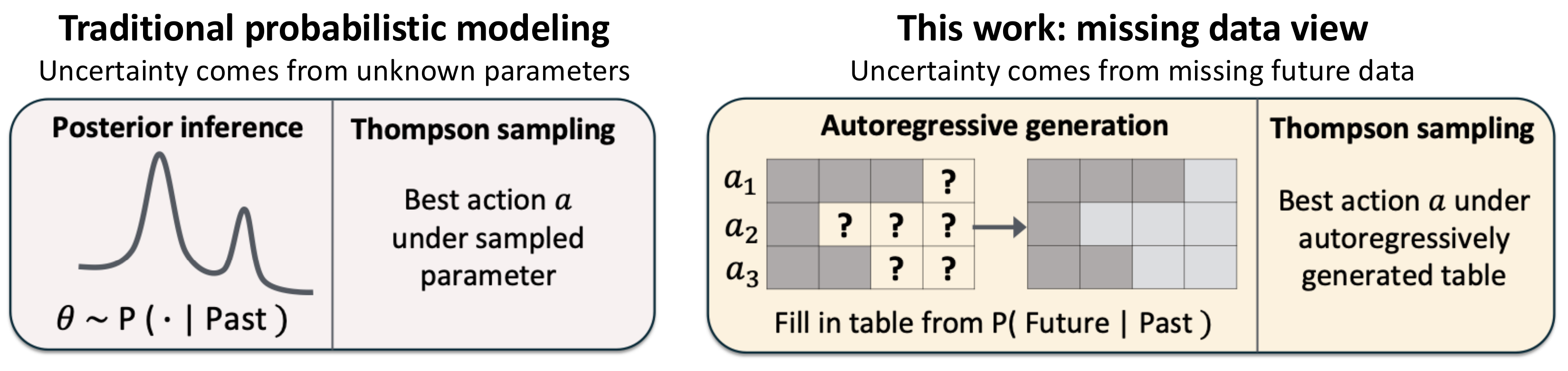}
    \caption{Traditional implementation of active exploration algorithms such as Thompson sampling requires probabilistic models over latent parameters that get updated as more data is gathered. Instead, we view the source of uncertainty in decision-making as \emph{missing data} and use autoregressive generation of missing outcomes as the basic unit of probabilistic inference. This view continues the predictive-inference philosophies of \citet{dawid1984present} and \citet{geisser2017predictive}, and the missing-potential-outcomes view of causal inference advocated by \citet{rubin:1987}; modern sequence models make it practical without requiring exchangeability of $Y_1, \ldots, Y_T$, and with unstructured prior information entering as a prompt/context. \vspace{-3mm}}
\label{fig:missing-data-view}
\end{figure}

\paragraph{Our approach.} Operationalizing this missing-data viewpoint yields concrete replacements for the subroutines that classical Bayesian exploration algorithms rely on; see Figure \ref{fig:missing-data-view}. These algorithms typically require i) constructing informed priors over latent parameters, ii) performing probabilistic inference under the current model of uncertainty, and iii) updating beliefs as more data is gathered---each of which becomes difficult with high-dimensional or unstructured data. Our approach instead i) trains a generative model through next-outcome prediction, ii) assesses uncertainty through autoregressive generation of missing  outcomes---including those not yet observed, and iii) adapts to new information by extending the sequence model’s context \citep{bengio2003neural} rather than explicitly updating a posterior. Marginal next-outcome prediction thus becomes the basic primitive for epistemic uncertainty quantification and active exploration: the same sequence-prediction machinery that already handles rich, unstructured inputs can drive exploration, without an explicit prior, likelihood, or posterior update over latent environment parameters. This recipe is not tied to any single algorithm---autoregressive generation serves as a drop-in replacement for posterior sampling of latents in Thompson sampling, Bayes UCB, knowledge gradient, information-directed sampling, and related methods 
\citep{thompson1933likelihood, kaufmann2012bayesian, russo2018learning, ryzhov2012knowledge}.

\paragraph{Outline.} Section~\ref{sec:key_ideas} develops the missing-data view of uncertainty in a passive-observation setting and shows that autoregressive generation serves as a drop-in replacement for a posterior sampling step that appears in  classical Bayesian exploration algorithms 
such as Thompson sampling, Bayes UCB, knowledge gradient, and information-directed sampling. %
Section~\ref{section:approximate} itemizes our contributions in full. Sections~\ref{sec:meta-bandit-formulation} and~\ref{sec:psar-algorithm} formalize the online decision-making problem and the use of autoregressive generation inside exploration algorithms, with a special focus on %
Thompson sampling %
\citep{thompson1933likelihood, russo2020tutorial}; we call our Thompson sampling based algorithm Thompson sampling via autoregressive generation (TSAR).  Section~\ref{sec:psar_ptheta} proves a regret bound that reduces online decision-making to offline next-outcome prediction. Section~\ref{sec:experiments} validates these ideas in synthetic settings and in a semi-synthetic news recommendation task that requires leveraging article headline text to focus exploration.

\section{Key ideas: Missing data view of uncertainty}
\label{sec:key_ideas}
While later sections of the paper study active exploration problems, we start by developing key ideas in a setting where observations accrue passively. A few examples help to explain the setting. Consider a news recommendation platform that has just released a new article. From the headline text, the platform can form an initial guess about how engaging the article will be, but substantial uncertainty remains---uncertainty 
that gets 
progressively resolved 
as the article is recommended to one user after another 
\citep{li2010contextual}. A similar phenomenon arises in political science, where a previous survey may suggest an increased probability of approval on a certain issue, but considerable uncertainty remains \citep{OfferWestortCoGr20}. As more individuals are polled, the overall approval becomes increasingly clear.

In both cases, we face a fundamental inferential question: given limited observations so far, what can we confidently conclude, and how much uncertainty remains?
In Section~\ref{section:classical}, we review classical Bayesian inference, which approaches these problems by representing unknowns through latent parameters---the engagement rate of an article or the approval rate of a particular policy, for example. %
In Section~\ref{section:missing-data}, we present our alternative approach that views missing data---the user feedback or surveys not yet gathered---as the source of uncertainty. %
Under this view, the inferential question becomes: given limited observations so far, what can we already infer about these missing but observable quantities, and what remains uncertain? 
Correct probabilistic imputation of missing data powers this style of uncertainty quantification without any explicit reference to latent variables.

\paragraph{A note on prior knowledge.}  Both approaches to uncertainty quantification require substantial probabilistic modeling. In the classical Bayesian setting, the prior and likelihood together specify the joint distribution over latents and observables. In our missing-data-based formulation, we require a sequence model that specifies the distribution of the next observation conditioned on %
previous observations. We begin by assuming that respective probabilistic models are available and focus on the resulting inference problems in Sections~\ref{section:classical} and~\ref{section:missing-data}.
In Section~\ref{section:approximate}, we return to where these probabilistic models come from and illustrate how the missing data view of uncertainty allows using foundation models to integrate complex prior knowledge (e.g., text and video %
information). 

\subsection{Classical Bayesian modeling}
\label{section:classical}

A standard way to formalize uncertainty is to posit a Bayesian model with a latent parameter, say, 
an underlying quality $\eta$ that governs observed outcomes $Y_1,\ldots,Y_T$. Here, one can treat available prior information as shaping initial beliefs about $\eta$, and then one updates those beliefs as observations arrive. 
Throughout, we explicitly denote any initial information the modeler may have as $Z$, to highlight the significance of such information in modern application scenarios (e.g., article text or past clinical notes). Throughout, we use the phrase ``prior information'' to refer to $Z$, as it is information about the task prior to observing outcomes $Y$.

Standard Bayesian approaches model exchangeable observations $Y_1, Y_2, \ldots, Y_T$ as being drawn independently and identically distributed %
conditioned on $\eta$. Given a subset of observations $Y_{1:t}$, %
the agent forms the posterior belief

\begin{equation}\label{eq:posterior_formula}
p(\eta \mid Z, Y_1, \ldots, Y_t) \propto \prod_{\tau=1}^{t} p(Y_\tau \mid \eta) \cdot p(\eta \mid Z),
\end{equation}
which summarizes everything that can (and cannot) be inferred about $\eta$ given the observed information. The prior distribution $p(\eta \mid Z)$ of $\eta$ given $Z$ and likelihood $p(Y_\tau \mid \eta)$ are assumed to be known.
\begin{example}
    To make these ideas concrete, consider a news recommendation platform that has just 
released a new article. As users 
interact with the article throughout the day, the agent observes feedback $Y_t$---clicks, dwell times, shares, comments---that reduces uncertainty about each article's true engagement level $\eta$. 
In this setting, the agent has substantive prior information about the article before it is shown to any user, such as headlines, text, metadata, which we denote explicitly as $Z$.
\end{example}

Typically, quantities of interest can be expressed as a function of $\eta$. 
One example is the  conditional mean $f(\eta)=\E[Y_\tau \mid \eta]$ where we assume $Y_\tau$ is a real-valued random variable. %
We formalize the problem of performing posterior inference about any quantity of interest $f(\eta)$ below. Grouping together prior information $Z$ and observations as $Y_{1:t} = (Y_1, \ldots, Y_t)$ %
to infer the distribution of $f(\eta)$ given $(Z, Y_{1:t})$
we can repeatedly draw $\widehat{\eta}$ from the posterior distribution~\eqref{eq:posterior_formula}, and compute $f(\widehat{\eta})$. %
\begin{tcolorbox}[colback=gray!5!white, colframe=gray!75!black, title=Inference under uncertainty induced by latent parameters]
 Given a latent $\eta$ and a quantity of interest $f(\eta)$, draw samples $\widehat{\eta}$ from the posterior distribution $\mathbb{P}(\eta = \cdot \mid Z, Y_{1:t})$ and compute $f(\widehat{\eta})$ on these samples.  
\end{tcolorbox} 
\noindent While elegant for simple settings, this framework faces practical challenges when the data is rich, and complex models are used for the prior and likelihood. 
Specifying an appropriate prior distribution $p(\eta \mid Z)$ for $Z$ with complex, unstructured information (like article text) involves incorporating knowledge from foundation models to inform the distribution of $\eta$, which is non-trivial. 
Additionally, when outcomes $Y$ are themselves high-dimensional or unstructured, the latent $\eta$ itself may be complex (e.g., the parameters in a neural network), making posterior computation difficult, particularly as data accrues over time $t$.

\subsection{Uncertainty as missing data}
\label{section:missing-data}

This paper takes an alternative approach that sidesteps latent parameters entirely. Instead of focusing on inferring a latent, intangible quality of the ``true'' environment, $\eta$, we consider a quantity of interest that is a function $f(Y_1, \ldots, Y_T)$ of the full trajectory of outcomes $Y_1, \ldots, Y_T$ sampled from some joint distribution $p(\cdot \mid Z)$; in contrast to Section~\ref{section:classical}, here we do not assume that $Y_{1:T}$ are necessarily exchangeable given $Z$. For example, practically the quantity of interest $f(Y_1, \ldots, Y_T)$ %
could be the average engagement if a news article were shown broadly, i.e., $f(Y_1, \ldots, Y_T) = \frac{1}{T} \sum_{t=1}^T Y_t$. At any timepoint $t$, the uncertainty about this quantity of interest arises from missing future outcomes $Y_{t+1}, \ldots, Y_T$. That is, \emph{uncertainty does not stem from a latent environment $\eta$ parameter}, but from outcomes $Y_{t+1}, \ldots, Y_T$ that are yet to be seen. 

Our uncertainty quantification task is as follows: 
given the prior information $Z$ and the outcomes observed so far $Y_1,\ldots,Y_t$, what is the distribution of the quantity of interest? Even among instances with similar prior information $Z$ and the same early outcomes, there can be substantial variability in the eventual final outcome, so this distribution is genuinely non-degenerate. We can formalize the problem in a manner that parallels the problem of inference under an uncertain latent $\eta$ to a striking degree. 
\begin{tcolorbox}[colback=gray!5!white, colframe=gray!75!black, title=Inference under missing data]
Given a quantity of interest $f(Y_1, \ldots, Y_T)$, and current observations %
$\{Z, Y_{1:t} \}$, 
draw samples $\widehat{Y}_{t+1}, \ldots, \widehat{Y}_T$ from the %
conditional distribution
$\mathbb{P} \big( (Y_{t+1}, \ldots, Y_T) = \cdot \mid Z, Y_{1:t} \big)$ and compute
$f(Y_1, \ldots, Y_t, \widehat{Y}_{t+1}, \ldots, \widehat{Y}_T)$ on these samples.
\end{tcolorbox} 

While so far we have focused on passively observed outcomes, when extended to the decision-making setting, this view of missing future outcomes as the source of uncertainty is particularly aligned with active exploration problems, where gathering missing observations 
tends to be a concrete choice the agent faces. 
This viewpoint is also very general. In particular, it does not
require that the data is conditionally i.i.d. given some $\eta$; in fact, this viewpoint does not even require %
exchangeability of $Y_1, \dots, Y_T$. %

\paragraph{Autoregressive generation of missing data.}
Even with a correctly specified model of the joint \emph{distribution} $Y_1 \dots, Y_T \mid Z \sim p(\cdot \mid Z)$, quantifying posterior uncertainty over the \emph{realization} $f(Y_1, \ldots, Y_T)$ after observing $(Z, Y_{1:t})$ generally is not available in closed form, and so a sampling-based approach to estimation is recommended. As shown in Figure~\ref{fig:autoregressiveGen}, our inferential approach  uses the variation in the estimated $f(Y_1, \ldots, Y_T)$ across probabilistic imputations of the missing outcomes to quantify uncertainty.

In particular, we probabilistically impute the next unobserved outcome $Y_{t+1}$ given the prior information and observed information so far $Z, Y_{1:t}$ using $\widehat{Y}_{t+1} \sim p(Y_{t+1} \mid Z, Y_{1:t})$. We then impute $Y_{t+2}$ via autoregressively sampling the next conditioned on the same history $(Z, Y_{1:t})$ in addition to %
the sampled value $\widehat{Y}_{t+1}$, i.e., $\widehat{Y}_{t+2} \sim p(Y_{t+2} \mid Z, Y_{1:t}, Y_{t+1} = \widehat{Y}_{t+1})$. We continue this process until all remaining outcomes $\widehat{Y}_{t+1,T}$ are probabilistically imputed; then, evaluating $f$ on the imputed trajectory produces one draw from the induced posterior over $f$. 
We call the model of conditional distributions of the form 
$p(Y_{t+1} \mid Z, Y_1, \ldots Y_{t})$ 
a \emph{sequence model}.
By autoregressively sampling from the correct conditional distribution $p$, we can sample from the correct joint distribution over missing outcomes.

\begin{figure}[t]
    \centering
    \includegraphics[width=0.95\linewidth]{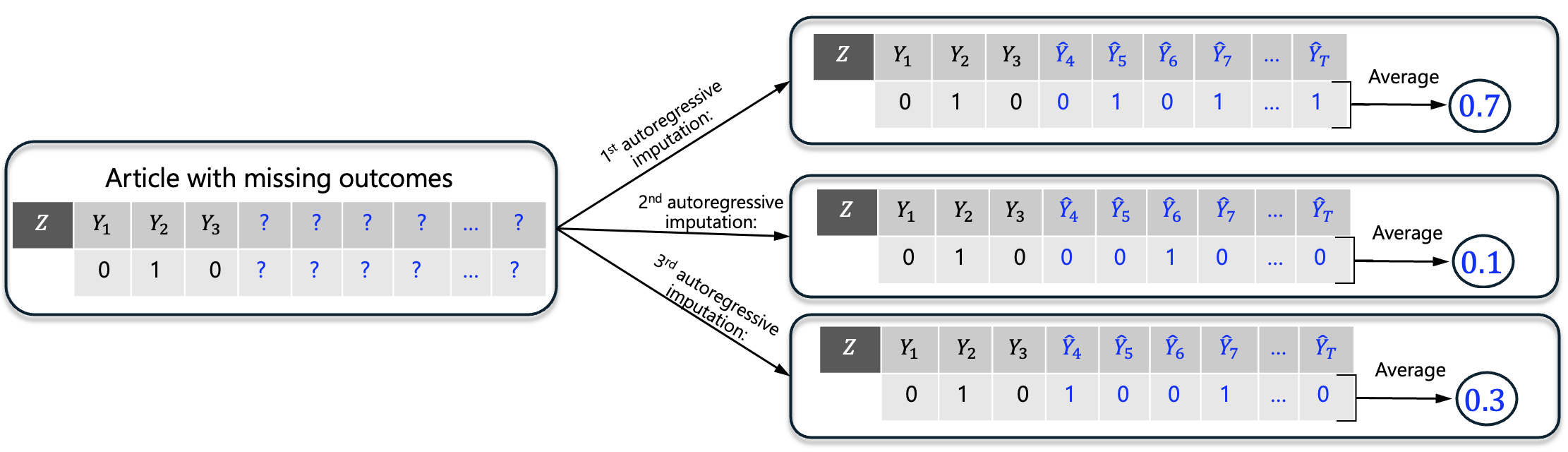}
    \caption{ Given prior information $Z$ and observed outcomes $Y_1, Y_2, Y_3,$ we generate multiple completions of the sequence. Variability across the imputed means $\frac{1}{T} \big( Y_1 + Y_2 + Y_3 + \sum_{t=4}^T \widehat{Y}_t \big)$ reflects remaining uncertainty. As more real outcomes $Y_t$ are observed, we expect this variability to shrink.} 
    \label{fig:autoregressiveGen}
\end{figure}

 \begin{lemma}
 \label{lemma:posteriorSample}
 After imputing missing data according to the procedure above with a correctly specified model $p$ of the joint distribution, the quantity of interest $f(Y_1, Y_2, \ldots, Y_T)$ computed from the imputed dataset is a valid ``posterior sample'' of $f(Y_1, Y_2, \ldots, Y_T)$ in that
 $$\mathbb{P} \Big(f(Y_1, \ldots Y_t, \widehat{Y}_{t+1}, \ldots, \widehat{Y}_T)  = \cdot \mid Z, Y_{1:t} \Big) = \mathbb{P}\Big(f(Y_1, \ldots, Y_T) = \cdot \mid Z, Y_{1:t} \Big).$$
 \end{lemma}

\noindent The autoregressive imputation approach described above does not require explicitly modeling a latent $\eta$, but is compatible with such models from Section~\ref{section:classical}: one can simply write the distribution of future outcomes as posterior predictives using the law of total probability. In this special case, our imputation approach matches the sampling methodology used in the multiple imputation literature \citep{rubin2018multiple,gelman2013bayesian}.

\begin{tcolorbox}[colback=gray!5!white, colframe=gray!75!black, title=Two Views of Uncertainty]
\begin{tabular}{@{}L{0.19\textwidth} L{0.38\textwidth} L{0.38\textwidth}@{}}
& \textbf{Classical View} & \textbf{Missing Data View} \\[0.8em]
\hline \\[-0.8em]

\textbf{Source of uncertainty} 
& Latent parameters $\eta$
& Future outcomes $Y_{t+1}, \dots, Y_T$ \\[2em]

\textbf{Model} 
& Prior and likelihood 
& Sequence model over observables \\[2em]

\textbf{Inference} 
& Draw latent from the posterior distribution 
& Sample $\widehat{Y}_{t+1}, \widehat{Y}_{t+2}, \ldots$ from sequence model, conditioned on $Z$ and previous observations \\[3em]

\textbf{Key challenge} 
& Encoding rich prior information $Z$ and handling high-dimensional outcomes
& Training a good sequence predictor that directly leverages foundation models \\[0.3em]
\end{tabular}
\end{tcolorbox}

\paragraph{Why the joint distribution matters, not just the next-step marginal.}
The autoregressive procedure does more than predict the next outcome; it posits a joint distribution over the remaining sequence.  To see why this matters, consider two distributions over a sequence of coin tosses.  The first is a fair coin with i.i.d. tosses of mean $1/2$; the second is a deterministic coin whose side is unknown but, once observed, is fixed for every subsequent toss.  The two have identical marginal distributions over the next toss but very different joint distributions: a single observation resolves all uncertainty in the second world and none in the first; see \cite[Figure 1]{osband2022neural}. Methods that summarize only the one-step distribution---such as generating a single next outcome \citep{MullerHoArGrHu22, NguyenGr22, GarneloRoMaRaSaShTeReEs18} or averaging several independent next-outcome draws---cannot tell the two apart, and consequently fail to drive exploration correctly when used inside a bandit algorithm.  Section~\ref{subsec:failure_of_other_generation} demonstrates these failure modes empirically.

\subsection{Approximate sequence models}
\label{section:approximate}

So far we have treated the sequence model as given, assuming it correctly specifies the conditional distribution of each outcome. We now return to a question deferred at the start of Section~\ref{sec:key_ideas}: where does this probabilistic model come from? 
The autoregressive imputation approach to uncertainty quantification opens up a path to leveraging foundation models for probabilistic uncertainty quantification in ways that appear difficult with the traditional Bayesian modeling of a latent parameter.

Because marginal next-outcome prediction is precisely the task that modern sequence models are trained on, a practitioner can form an
\emph{approximate} sequence model $\widehat{p}$ for predicting next outcomes to quantify uncertainty---without any additional Bayesian inference machinery. 
There are several possibilities for forming such a $\widehat{p}$:
\begin{itemize}[leftmargin=*]
\item \emph{Prompting.} 
Relying on their strong zero-shot performance, one can simply prompt an LLM to predict the next observation given prior information %
$Z$ and past outcomes %
$Y_1, \ldots, Y_{t}$.
\item \emph{In-context learning.} One can augment prompting by including in the prompt historical examples of prior information $Z$ and their corresponding past outcomes $Y_1,\ldots,Y_{t}$ \citep{brown2020language}.
\item \emph{Fine-tuning.} When domain-specific outcome data is available, one can fine-tune a pretrained model to predict sequences of outcomes conditioned on prior information $Z$.
\item \emph{Training from scratch.} For organizations with abundant historical data (e.g., Google~\citep{rajput2023recommender}), one can train a sequence model entirely from scratch over past outcome sequences. We discuss our procedure for this in Section \ref{sec:pretrainHistorical}, which we use in a semi-synthetic news recommendation experiment.
\end{itemize}

However $\widehat{p}$ is formed, it is inherently approximate. A central challenge we focus on in this work is how this approximation affects uncertainty quantification and, ultimately, exploration.

\subsection{Our contributions}
\label{section:contributions}

The ideas in Section~\ref{section:missing-data} are not themselves novel, but the formulation is atypical in the literature and sets up the remainder of the paper. Its distinctive feature is the focus on functionals $f(Y_1, \ldots, Y_T)$ of a finite population of outcomes, under which all results follow from the chain rule of probability with no further assumptions. The intellectual lineage we build on (Section~\ref{section:related-work}) is instead rooted in de Finetti's theorem, a deep result which essentially implies that observing an infinite \emph{exchangeable} sequence of $Y$'s reveals a true latent data-generating distribution. In our formulation, the joint distribution of outcomes can be entirely general, and no latent parameter needs to be considered---or even necessarily exists.

This paper applies these ideas to \emph{online learning} problems where present actions impact data observed in the future.
A theme throughout is that a capability for prediction \emph{suffices}: the sequence model that drives the algorithm is never required to be exchangeable, to encode a conditional-i.i.d.\ structure, or to correspond to a latent-variable model---only to predict well.
The rest of the paper develops three contributions that, together, turn the missing-data view into a working algorithm with formal guarantees.
\begin{enumerate}[leftmargin=*]
    \item \bo{Posterior sampling via autoregressive generation.} Uncertainty quantification alone does not prescribe \emph{action}; an agent must translate uncertainty about missing outcomes into decisions that balance exploitation of current knowledge against exploration to resolve remaining uncertainties. To study this problem, Section~\ref{sec:meta-bandit-formulation} extends the missing data formulation to bandit-style decision-making tasks. Then, in Section \ref{sec:psar-algorithm}, we adapt classical Bayesian exploration algorithms by replacing posterior sampling over latent parameters with autoregressive generation of missing outcomes conditioned on prior information $Z$ and outcomes observed so far in the current task.  Section~\ref{sec:psar-algorithm} formally interprets this generation process as posterior sampling whenever the sequence model can sample exactly from the relevant conditional distributions, providing a drop-in replacement in algorithms such as Bayes UCB, knowledge gradient, information directed sampling, and Thompson sampling \citep{kaufmann2012bayesian, russo2018learning, ryzhov2012knowledge}.  We develop and analyze the Thompson sampling version most fully. The results of Sections~\ref{sec:psar-algorithm} and~\ref{sec:screening}, just as in Section~\ref{section:missing-data}, hold for an arbitrary joint distribution over the potential outcomes table, with no exchangeability or stationarity assumption.

    \item \bo{Regret bound: success in offline sequence prediction translates into effective online decision-making.}  
    In online bandit problems, the agent's actions determine which observations it gathers, errors in an approximate $\widehat{p}$ can compound: mistaken beliefs may lead an agent to choose actions consistent with those beliefs, causing it to rarely collect data that would correct them.
    \citet{simchowitz2021bayesian} make this vivid by showing that Thompson sampling with a misspecified prior can suffer \emph{worst-case} cumulative regret that scales \emph{quadratically} in the horizon $T$, and a reflexive imitation-learning argument would have offline prediction error enter the regret bound multiplied by a distribution-shift coefficient~\citep{rossImitation2010}.  Section~\ref{sec:psar_ptheta} rules out such multiplicative cascades. We prove that, averaged over the prior that actually generates the bandit task, the agent's (Bayes) regret is directly controlled by the sequence model's offline next-outcome prediction error---with no multiplicative distribution-shift or concentrability coefficient, despite adaptive data collection---, scaling at worst \emph{linearly} in $T$ and recovering the standard $\sqrt{T}$ Thompson sampling rate as offline predictions get better. This establishes a clean reduction from online decision-making to offline next-outcome prediction. The main structural condition in this analysis is exchangeability of the \emph{true} data-generating process, used solely to identify left-to-right log-loss as the correct offline training target; the learned model $\widehat{p}$ itself may be arbitrary---misspecified, non-exchangeable, and with no latent-variable interpretation.

    \item \bo{Experimental validation in synthetic settings and a semi-synthetic news recommendation task.}  We demonstrate that our theoretical insights bear out in simulations.  Our methods not only match oracle Bayesian procedures in synthetic settings, but do so without relying on an explicit latent parameter specification.  Moreover, our approach scales to a news article recommendation setting in which the best performance requires leveraging unstructured prior information like article headlines.  Even when incorporating and fine-tuning a language model, we find our autoregressive generation procedure implements accurate uncertainty quantification (credible intervals) and drives exploration with low regret.
\end{enumerate}
\bo{Relationship to a follow-up paper.}  
A follow-up paper by the same authors extends one component of this work---the Thompson sampling regret analyses in Sections~\ref{sec:regret-correctly-specified} and \ref{section:regret-bounds}---to contextual bandits with general policy classes, adding VC-dimension-based entropy arguments specific to that setting~\citep{zhang2025contextual}. 
The present paper develops the broader framework:  the missing-data formulation of epistemic uncertainty discussed in this section, posterior sampling via autoregressive generation for exploration algorithms in Section \ref{sec:psar-algorithm}, a novel regret bound on screening actions using high-dimensional action features in Section~\ref{sec:screening},
general-purpose reductions of regret to offline prediction for multiple adaptive policies in Section~\ref{sec:psar-ts}, 
a formal comparison with previous analyses of Thompson sampling under prior misspecification in Section~\ref{sec:misspecifiedTS}, and a full empirical study that includes evaluation of the epistemic uncertainty quantification in Section~\ref{sec:experiments}.

\section{Preliminaries: Online decision-making problem}
\label{sec:meta-bandit-formulation}

In this section, we provide a formal definition of the decision-making problem.
We consider a bandit-style decision-making task in which an agent sequentially selects actions from a finite set $\Aeval$ across time periods $t\in \{1, 2, \ldots, T\}$, adapting future actions in response to feedback. Formally, we say a  bandit task  $\mathcal{B} = (Z, \tau)$ consists of prior task information $Z$ and a two-dimensional array or ``table'' of potential outcomes \begin{equation}
\tau  = (Y_{1}^{(a)}, \ldots, Y_{T}^{(a)})_{a \in \Aeval},
\end{equation}
where each $Y_t^{(a)}$ represents what would be observed if action $a \in \MC{A}$ were taken at time $t$. This table is the multi-action analogue of the outcome sequence $Y_1, \ldots, Y_T$ from the passive setting of Section~\ref{section:missing-data}: each row represents the potential outcomes for one action, and, exactly as in Section~\ref{section:missing-data}, what the agent does not yet know about the task is encoded in the entries of $\tau$ that have not been revealed---not 
in any latent parameter. Prior information $Z$ can be made up of rich, unstructured data, and could comprise a set of action-specific information  $\{Z^{(a)}\}_{a\in \mathcal A}$. Prior information lets us model problems where rich information about the task is available before any outcomes are observed, and where that information should be used to guide exploration. For example, in a news recommendation problem, $Z$ could represent news article text.

The pair $\mathcal B=(Z,\tau)$ defines a bandit problem as follows: at each timestep $t$, the agent selects action $A_t \in \Aeval$, observes a vector of outcomes $Y_t$, and associates with this a reward $R(Y_t) \in [0,1]$ by applying a fixed, known function $R(\cdot)$. 
In particular, $Y_t = Y_t^{(A_t)}$. We index rows of the table by actions and columns by time, writing $\tau[a,t] = Y_t^{(a)}$.
The agent can make decisions using its own \emph{task-specific history} from that day: %
\begin{equation}
	\HH_{t-1}\triangleq \big( %
    Z, \, (A_1, Y_1),\ldots, (A_{t-1}, Y_{t-1}) \big), 
	\label{eqn:history}
\end{equation} 
which includes all historical observations available at the start of time $t$ and 
prior information $Z$. 
We call this an \emph{informed bandit} problem as there is prior information $Z$
associated with each task. %
Contrast this with contextual bandit problems, where fresh context arrives at each timestep and describes the current decision-making situation rather than the entire task. 
While existing works have studied meta-bandits that utilize task features \citep{wan2021metadata}, to our knowledge, existing bandit algorithms have not previously studied this kind of informed bandit problem with high-dimensional, unstructured $Z$ (e.g., text).

The goal, informally, is to select a rule for adaptively selecting actions that maximizes cumulative expected reward. 
We assume that the bandit task $\mathcal{B}$ is sampled i.i.d. from the \emph{task distribution}:
\begin{align}
\label{eq:pstar}
    \mathcal B \sim p_{\TN{tasks}}^*.
\end{align}
This is the direct extension of the model in  Section~\ref{section:missing-data}, where we assumed $Z$ and the single sequence $Y_1, \ldots, Y_T$ are drawn jointly from some distribution.

\paragraph{Online decision-making objective.} In the passive setting of Section~\ref{section:missing-data}, inference centered on a trajectory-level quantity of interest $f(Y_1, \ldots, Y_T)$. Similarly, in this active multi-action setting, various quantities of interest to the decision-maker can be written as functions of the potential outcomes table $\tau$. 
The long-run mean reward of action $a$ is
\begin{equation}\label{eq:mean_reward_tau}
\mu_{a}(\tau)= \frac{1}{T} \sum_{t=1}^{T} R \big( Y_{t}^{(a)} \big).
\end{equation}
This is essentially the ``population mean'' for a large finite population of size $T$, a common object of interest in statistics \cite[see e.g.][Chapter 7]{rice2007mathematical}. Similarly, the action with maximal reward is
\begin{equation}\label{eq:optimal_action_tau}
A^*(\tau) \in \argmax_{a\in \Aeval}  \, \mu_{a}(\tau).
\end{equation}
Above, we assume ties are broken in some deterministic fashion. A decision-maker who knew $\tau$ upfront could immediately implement the action $A^*(\tau)$ with highest average reward throughout all periods, with no need to engage in costly exploration. 

The cost of exploration is typically measured through \emph{regret}: the gap in rewards accrued relative to an omniscient agent who knew from the beginning the best fixed action over the horizon, $A^*(\tau)$. To define this formally, we let $\pi$ denote some adaptive policy---i.e. any (possibly randomized) rule that specifies which action to select next, given $Z$ and the history of past actions and $Y$'s. %
The per-period regret incurred by $\pi$ on bandit instance
$\mathcal{B}$ is given by\footnote{We benchmark against the best action in hindsight, which is very common in adversarial formulations. Less stringent versions of regret that move the maximum outside a certain conditional expectation are sometimes called pseudo-regret \citep{bubeck2012regret,LattimoreSz19}. In versions of our formulation where both are well-defined, the distinction between these are insignificant (Appendix \ref{app:finite_pop_TS}).}
\[
\Delta(\pi, \mathcal{B})  = \E_{\pi}\left[ \max_{a\in \mathcal{A}} \frac{1}{T} \sum_{t=1}^{T} R(Y_t^{(a)}) - \frac{1}{T} \sum_{t=1}^{T} R(Y_t^{(A_t)}) \, \bigg| \,\mathcal{B} \right].
\]
The expectation integrates over any internal randomness used in $\pi$. For deterministic policies, no expectation is needed.

We measure performance through expected regret  
\[
\Delta(\pi) \triangleq \E \left[\Delta(\pi, \mathcal{B}) \right],
\]
which averages over the draw of the bandit task $\mathcal{B}$. In our framework, this has the interpretation (by the law of large numbers) of being the long-run regret a fixed policy $\pi$ would incur on average across many i.i.d.  draws of bandit tasks.

\section{Active exploration via generative implementations of ``posterior sampling''}
\label{sec:psar-algorithm}

We now formally introduce our decision-making algorithm. Throughout Section \ref{sec:psar-algorithm}, we assume the algorithm has knowledge of the task distribution $p_{\TN{tasks}}^*$ defined in \eqref{eq:pstar}, from which $\mathcal B$ is drawn.
Specifically, knowledge of $p_{\TN{tasks}}^*$ means that the agent knows at all times the conditional distribution of any entry in potential outcomes table, written as
\begin{equation}
    \label{eq:missing_entry_distribution_general}    
    p_{\TN{tasks}}^* \big(\tau[a, i] = y \mid Z, (\tau[a',i'])_{(a',i') \in \mathcal{O}} \big)
\end{equation} 
for any tuple $(a, i)$ and any set of observed tuples $\mathcal{O}$. The above conditional distribution can be thought of as a sequence model, as it is the probability of one outcome, given a previous sequence of outcomes, and task information $Z$. 
For now, assume $p^*_{\rm{tasks}}$ is known; we discuss the case in which we only have an approximate sequence model later in Section \ref{sec:psar_ptheta}.
In practice, we will make simplifying assumptions on the structure of $p^*_{\TN{tasks}}$, which will make implementation easier (see Section~\ref{subsubsec:assumptions}).

Note %
that \emph{even if the task distribution $p_{\TN{tasks}}^*$ is known, the decision-making problem for our particular task $\MC{B}$ is non-trivial: $p_{\TN{tasks}}^*$ specifies %
the distribution of tasks, but not our particular task $\MC{B}$.} For example, we may know that bandit tasks with certain features $Z$ on average have certain characteristics, but we do not know the realization of
the  potential outcomes $\tau$ for the task %
at hand; this is the missing-data uncertainty introduced in Section~\ref{section:missing-data}.

\subsection{Posterior sampling by imputation when $p_{\TN{tasks}}^*$ is known} %
\label{sec:our-algorithm}
While the agent knows the conditional distributions of missing outcomes according to $p_{\TN{tasks}}^*$, the agent does not know the realization of $\tau$. All decisions are made under uncertainty about missing entries in this \emph{potentially} observable data table $\tau$. As decisions are made sequentially, each observed outcome exactly reveals one entry in $\tau$, resolving uncertainty about that entry while potentially informing predictions about the remaining unobserved entries. The agent may also resolve substantial uncertainty by leveraging prior information $Z$. One way to reason about this uncertainty is to imagine what the full table could be, drawing a hypothetical realization of $\tau$ that is consistent with what has already been observed. We call this a \emph{posterior sample}, which we define formally in Definition~\ref{lemma:posteriorSample} below. %

\begin{definition}\label{def:posterior_sample}
	 A random variable $\what{\tau}_t$ is an \bo{\emph{exact posterior sample}} of the potential outcomes table $\tau$ at time $t$ if $\what{\tau}_t$ and $\tau$ have %
     the same conditional distributions given all observations available before that time:
	\begin{equation}\label{eq:posterior-sampling-def}
		 \mathbb{P}\left(\what{\tau}_t \in \cdot \mid \HH_{t-1} \right) = \mathbb{P}\left(\tau \in \cdot \mid \HH_{t-1}\right) \quad  \text{a.s.}
	\end{equation}
    If the generated table $\what{\tau}_t$ is an exact posterior sample, functions of the generated table, like the associated mean-rewards $\mu_{a}(\what{\tau}_t)$ or optimal action $A^*(\what{\tau}_t)$ from \eqref{eq:mean_reward_tau} and \eqref{eq:optimal_action_tau} respectively, are also exact posterior samples of those quantities. In other words, if \eqref{eq:posterior-sampling-def} holds, then for any measurable function $f$, %
    \begin{align*}
        \mathbb{P}\big( f(\what{\tau}_t) \in \cdot \mid \HH_{t-1} \big) = \mathbb{P}\big( f(\tau) \in \cdot \mid \HH_{t-1}\big) \quad  \text{a.s.}
    \end{align*}
\end{definition}

An ability to draw posterior samples enables us to approximate the posterior distribution of any function of $\tau$. For instance, at timestep $t$, the posterior probability some action $a$ is optimal on average across the $T$ rounds could be approximated by Monte-Carlo simulation by repeatedly drawing posterior samples $\what{\tau}_t$, then calculating and recording the fraction of draws under which $A^*(\what{\tau}_t)=a$.

Posterior sampling requires sampling from the joint distribution of the missing entries. This can be implemented by iterating over them sequentially, at each step drawing a value from the conditional distribution \eqref{eq:missing_entry_distribution_general}  given the previously sampled entries, as previewed in Section~\ref{sec:key_ideas}.  Algorithm \ref{alg:posterior_sample_generic} shows pseudocode where we begin with a partially revealed potential outcomes table, then repeatedly select a missing entry (line \ref{line:generic_psar_pick}), sample its value conditioned on \emph{both observed entries and previous %
samples} (line \ref{line:generic_psar_sample}), and then record it into the table (line \ref{line:generic_psar_record}). This is an extension of the imputation procedure in Section~\ref{section:missing-data}, where the procedure from Figure~\ref{fig:autoregressiveGen} is carried out over the potential outcomes table $\tau$ rather than a single sequence.

\begin{algorithm}[h]
\caption{Generic posterior sampling via autoregressive generation (PSAR)}
\label{alg:posterior_sample_generic}
\begin{algorithmic}[1]
\Require $p_{\TN{tasks}}^*$, action set $\Aeval$, time $t$, history $\mathcal{H}_{t-1}$
\State Initialize $\what{\tau}_t$ as an $|\Aeval| \times T$ array with all observed $Y$'s %
from $\mathcal{H}_{t-1}$ filled in.
\State Initialize the set of filled in entries in $\what{\tau}_t$ as $F=\{(A_i,i) : i \leq t-1 \}$. %
\While {missing entries remain} 
\State Pick an index $(a,i)$ of a missing entry.
\label{line:generic_psar_pick}
\State Sample $\what{Y}^{(a)}_i \sim p_{\TN{tasks}}^*\left(\tau[a,i] =\cdot \mid Z, \{\tau[a,i]\}_{(a,i)\in F}\right)$.
\label{line:generic_psar_sample}
\State Record $\what{Y}^{(a)}_i$ in  $\what{\tau}_t$ and  add $(a,i)$ to the set of filled-in entries $F$.
\label{line:generic_psar_record}
\EndWhile\\
\Return $\what{\tau}_t$
\end{algorithmic}
\end{algorithm}
Correctness of Algorithm \ref{alg:posterior_sample_generic} is an immediate consequence of the chain rule of %
probability, 
which factorizes a joint probability distribution into a product of %
conditional distributions. We state the result and omit the proof.  
\begin{lemma}[Correctness]\label{lem:exact_psar}
	 The output $\what{\tau}_t$ returned by Algorithm \ref{alg:posterior_sample_generic} is an exact posterior sample (Definition \ref{def:posterior_sample}).
\end{lemma}

\subsection{A generative implementation of Thompson Sampling} %
\label{sec:psar_pstar}
Here we start by introducing our decision-making algorithm that uses posterior sampling by autoregressive generation %
(Algorithm~\ref{alg:posterior_sample_generic}, ``PSAR'') as a subroutine. 
We will show that the decision-making algorithm implements an analogue of Thompson sampling (i.e., probability matching) that replaces the usual step of posterior sampling of latent parameters $(\eta)$ for a posterior sample of the potential outcomes table $(\tau)$. %

The algorithm maintains an internal copy of the potential outcomes table, initialized with all values set to missing (line \ref{line:generic_tspsar_init}). As actions are selected and outcomes are observed, they are recorded in the table. To select an action in any period $t$, the algorithm generates a posterior sample of the potential outcomes table $\what{\tau}_t$ by copying observed entries, and sequentially imputing missing ones (line \ref{line:generic_tspsar_impute}%
), calculates the action with maximal average reward under that table and chooses to explore that action (line \ref{line:generic_tspsar_calculate_select}). Note that imputations generated at time $t$ are used only to select an action for that timestep and are not used thereafter.

\begin{algorithm}[h]
	\caption{Generic implementation of Thompson sampling via autoregressive generation (TSAR).} %
	\label{alg:Thompson_sampling_generic}
	\begin{algorithmic}[1]
		\Require Posterior sampling subroutine satisfying \eqref{eq:posterior-sampling-def}, known reward function $R(\cdot)$, action set $\Aeval$, time horizon $T$. 
		\State Initialize internal copy of $\tau$ as $|\Aeval| \times T$ array with all values set to missing. \label{line:generic_tspsar_init} 
        \State Initialize history as $\mathcal{H}_{0} \gets \{ Z^{(a)}: a \in \Aeval\}$. %
		\State 	Observe unstructured prior information $Z$.
		\For{$t=1,\ldots, T$}
		\State Generate posterior sample $\what{\tau}_t$ given $\HH_t$, satisfying 
        \eqref{eq:posterior-sampling-def}, e.g., using PSAR Alg. \ref{alg:posterior_sample_generic}. 
        \label{line:generic_tspsar_impute}
        \State Select action $A_t \gets A^*(\what{\tau}_t)$ that is optimal under $\what{\tau}_t$ and then discard %
        table $\what{\tau}_t$.
        \label{line:generic_tspsar_calculate_select}
        \State Observe outcome $Y_t \gets Y_t^{(A_t)}$ from action $A_t$.
        
		\State Update history $\HH_{t} \gets \HH_{t-1} \cup \{ (A_t, Y_t)\}$ and record outcome $Y_t$ in internal copy of $\tau$.
		\EndFor
	\end{algorithmic}
\end{algorithm}

We call this procedure Thompson sampling because it satisfies the probability matching property that is often taken as the definition for Thompson sampling \citep{thompson1933likelihood, scott2010modern,russo2016information}, which is that %

\begin{center}
    \textit{
The probability any action is chosen is equal to the posterior probability that it is optimal.
    }
\end{center}
\begin{lemma}[Probability matching, i.e., Thompson Sampling]
\label{lem:prob_matching} Under Algorithm \ref{alg:Thompson_sampling_generic}, for every time $t$ and action $a$, %
$\PP\left(A_t = a \mid \HH_{t-1}\right) = \PP\left(A^*(\tau)=a \mid \HH_{t-1}\right)$ a.s.
\end{lemma}
\begin{proof}
	We have $\PP\left(A_t = a \mid \HH_{t-1}\right) = \PP\left(A^*(\what{\tau}_t)=a \mid \HH_{t-1}\right) =\PP\left(A^*(\tau)=a \mid \HH_{t-1}\right)$, where the first equality is line \ref{line:generic_tspsar_calculate_select} of Algorithm \ref{alg:Thompson_sampling_generic} and the second equality uses the definition of a posterior sample \eqref{eq:posterior-sampling-def}. 
\end{proof}

\subsubsection{Regret bounds for online decision-making using $p^*$ (Algorithm~\ref{alg:Thompson_sampling_generic}).}
\label{sec:regret-correctly-specified}
The next result applies pre-existing analyses to bound the regret of this generative 
Thompson sampling variant. In particular, per-period regret vanishes if the number of interactions per action, $T/|\Aeval|$ is large, which are settings with ample opportunity for exploration. Note that Proposition \ref{prop:regret_TS_exact} below does not make a stationarity assumption on the reward distribution over time. %
\begin{prop}[Regret of exact Thompson sampling via autoregressive generation]
\label{prop:regret_TS_exact} Let $\pi_{\rm TSAR}$ denote Algorithm \ref{alg:Thompson_sampling_generic}. Assume rewards are bounded and normalized so that $R(y)\in [0,1]$ for any outcome $y$. Then, 
	\[
	\Delta(\pi_{\rm TSAR}) \leq \sqrt{\frac{|\Aeval| \log(|\Aeval|)}{2T} }.
	\]
\end{prop}
\begin{proof}
This result is an immediate consequence of the information-theoretic analysis of Thompson sampling developed by \citet{russo2016information}. 
Note \citet{bubeck2015bandit} and \citet{bubeck2016multi} observed that the proof approach used in \citet{russo2016information} applies without modification beyond the stationary bandit setting, as long as the probability matching property in Lemma \ref{lem:prob_matching} holds. See Appendix~\ref{sec:app_exact_ts} for more details.
\end{proof}

\subsubsection{The benefit of prior information: Regret with effective screening}\label{sec:screening}

The value of prior information $Z$ in our framework is that it lets the agent largely avoid exploring unpromising actions.  Probability matching makes this concrete: %
under Thompson sampling with a correctly specified model, 
since the probability an action is played equals the model's posterior probability that it is optimal, an action assigned prior probability below $1/T$ of being optimal is played less than once in expectation over the whole horizon. %
In regimes that motivate informed exploration---many candidate actions relative to the horizon---this screening is what makes strong performance possible at all.

The following result shows that this screening effect is captured automatically by our regret bound, without any modification to the algorithm.  Although the intuition is natural---Thompson sampling should not waste effort on arms the prior already rules out---we are not aware of a similar result in the literature, and provide our own proof based on the information-ratio technique of \cite{russo2016information}.

For a threshold $\epsilon > 0$, define the \emph{candidate set}
\[
  S(Z) \;=\; \bigl\{a \in \cA : \PP(A^*\!=\!a \mid Z) > \epsilon\bigr\},
\]
the actions that the prior information $Z$ leaves plausibly optimal.  Note that the size of the set $|S(Z)|$ is a random variable determined by $Z$; we write $\E[|S(Z)|]$ for its expectation over the task distribution.

\begin{corollary}[Regret of Thompson sampling with effective screening]\label{cor:screening}
Let $\pi_{\rm TSAR}$ denote Algorithm~\ref{alg:Thompson_sampling_generic} applied with exact posterior samples satisfying~\eqref{eq:posterior-sampling-def}.  For threshold $\epsilon = 1/T$,
\begin{equation}\label{eq:screening-bound}
  \Delta\bigl(\pi_{\rm TSAR}\bigr)
  \;\le\;
  \frac{|\cA|-\E[|S(Z)|]}{T}
  \;+\;
  \sqrt{\frac{\E[|S(Z)|]\cdot\log|\cA|}{2T}}\,.
\end{equation}
\end{corollary}

\noindent
Compared with Proposition~\ref{prop:regret_TS_exact}, which gives $\Delta(\pi_{\rm TSAR}) \le \sqrt{|\cA|\log|\cA|/(2T)}$, the bound~\eqref{eq:screening-bound} replaces $|\cA|$ in the leading $\sqrt{T}$ term with the expected number of actions that survive screening, $\E[|S(Z)|]$.  The arms screened out by $Z$ contribute only a $T$-independent additive cost of $(|\cA| - \E[|S(Z)|])/T$ per period.  When $Z$ is highly informative---so that $\E[|S(Z)|] \ll |\cA|$---this represents a substantial improvement. The proof in Appendix \ref{app:screening} utilizes information theoretic techniques to analyze Thompson sampling from \citet{russo2016information}.

\subsubsection{Online decision-making beyond Thompson sampling}\label{subsec:beyond_ts} %
The missing-data view of uncertainty applies more generally to other active exploration algorithms. In particular, our posterior sampling via autoregressive generation approach (Algorithm \ref{alg:posterior_sample_generic}) can be used as a replacement for posterior sampling of latent variables in many common decision-making algorithms.
For example, we can implement upper confidence bound algorithms by simulating several imputed average rewards for each action  and taking an action that maximizes their quantile; see Algorithm \ref{alg:bayesUCB} for pseudocode for BayesUCB \citep{kaufmann2012bayesian}.
Similarly, we can implement look-ahead procedures (e.g., Knowledge Gradients \citep{ryzhov2012knowledge}) by imputing an outcome from an action taken this timestep, and appending it to the context of the sequence model to simulate outcomes for future timesteps. Our missing data perspective allows one to extend these existing exploration algorithms to utilize potentially complex prior information through the pretrained sequence model.

\begin{algorithm}[h]
	\caption{%
    BayesUCB with posterior sampling via autoregressive generation} %
	\label{alg:bayesUCB}
	\begin{algorithmic}[1]
		\Require $p_{\TN{task}}$, actions $\Aeval$, prior information $\{ Z^{(a)}\}_{a \in \Aeval}$, number of generations $k$.
		\State Initialize history as $\mathcal{H}_{0} \gets \{ Z^{(a)}: a \in \Aeval\}$
		\For{$t \in \{1, \ldots, T\}$}
		\State Autoregressively generate $k$ conditionally i.i.d. data completions $\what{\tau}_{t,1}, \dots, \what{\tau}_{t,k}$ by repeatedly applying Algorithm \ref{alg:posterior_sample_generic} using $p_{\TN{task}}$, $\Aeval$ and $\mathcal{H}_{t-1}$. %
		\State Set quantile level $q_t \gets 1 - \frac{1}{t \log T}$.
		\State Let $q_{t,k} :=\lceil q_t\cdot k \rceil$, i.e., $q_t \cdot k$ rounded up to the nearest integer.
		\State For each $a \in \Aeval$ set $\TN{UCB}_{t,a}$ to the $q_{t,k}^{\TN{th}}$ smallest value in $\big\{ \mu_{a}(\what{\tau}_{t,j}) \big\}_{j=1}^k$.
		\State Select action $A_t \gets \TN{argmax}_{a \in \Aeval} \big\{ \TN{UCB}_{t,a} \big\}$. 
		\State Observe outcome $Y_t \gets Y_t^{(A_t)}$ from action $A_t$.
		\State Update history $\HH_{t} \gets \HH_{t-1} \cup \{ (A_t, Y_t)\}$.
		\EndFor 
	\end{algorithmic}
\end{algorithm}

\subsection{Autoregressive posterior sampling under structural assumptions}
\label{subsubsec:assumptions}

In this section, we present a specific instantiation of our posterior sampling procedure, Algorithm \ref{alg:posterior_sample_generic}, under two specific structural assumptions: \textit{independence across actions} (Assumption \ref{assump:articlesiid}) and \textit{exchangeability over time} (Assumption \ref{assump:exchangeable}). 
Additionally, we assume that the prior information $Z=\{Z^{(a)}\}_{a\in \Aeval}$ %
consists of prior information for each action. %
Since these structural assumptions make implementing our algorithm more straightforward, the remainder of this paper will focus on the setting in which these two assumptions hold.

Assumption \ref{assump:articlesiid} implies that given $p_{\TN{tasks}}^*$ is known, 
the prior information and potential outcomes for action $a$, $\big( Z^{(a)}, Y_{1:T}^{(a)}\big)$, do not provide any additional information about $\big( Z^{(a')}, Y_{1:T}^{(a')}\big)$ for a different action $a'$. %
This assumption is reasonable in certain domains (like product recommendations) and simplifies the generative posterior sampling procedure substantially because it allows us to draw inferences independently across actions.

\begin{algorithm}[b!]
\caption{Posterior sampling via autoregressive generation under Assumptions \ref{assump:articlesiid} and \ref{assump:exchangeable}}
\label{alg:posterior_sample}

\begin{algorithmic}[1]
\Require $p_{\TN{per-action}}$, action set $\Aeval$, time $t$, history $\mathcal{H}_{t-1}$
\State Initialize $\what{\tau}_t$ as an $|\Aeval| \times T$ array with all observed $Y$'s %
filled in.
\For{action $a \in \Aeval$}
    \State Initialize the set of indices for filled-in entries for action $a$ as $F^{(a)} = \{(a,i) : i\leq t-1, A_i=a\}$.
    \While{missing entries remain, i.e. $|F^{(a)}| < T$ }
    \State Pick a missing %
    index $(a,i)$ not in %
    $F^{(a)}$. 
    \State Sample $\what{Y}^{(a)}_i \sim p_{\TN{per-action}}\left(\tau[a,i] =\cdot \mid Z^{(a)}, \{\tau[a,i]\}_{(a,i)\in F^{(a)}}\right)$.
    \State Record $\what{Y}^{(a)}_i$ in $\what{\tau}_t$ and  add $(a,i)$ to the set of filled-in entries $F^{(a)}$.
    \EndWhile
\EndFor \\
\Return $\what{\tau}_t$
\end{algorithmic}
\end{algorithm}

\begin{assumption}[Independence across actions]
	\label{assump:articlesiid}   
    There exists some distribution $p_{\TN{per-action}}^*$ such that 
	\begin{align*}
		\big( Z^{(a)}, Y_{1:T}^{(a)}\big) \sim p_{\TN{per-action}}^* \quad \TN{independently across} \quad a \in \Aeval.
	\end{align*}
    This means $p^*_{\TN{tasks}}$ can be decomposed as follows: $p^*_{\TN{tasks}} \big( \{ Z^{(a)}, Y_{1:T}^{(a)} \}_{a \in \MC{A}} \big) = \prod_{a \in \MC{A}} ~ p^*_{\TN{per-action}} \big( Z^{(a)}, Y_{1:T}^{(a)} \big)$.
\end{assumption}

\begin{figure}[b!]
    \centering
    \vspace{-4mm}
      \begin{subfigure}{0.85\linewidth}
    \centering
    \includegraphics[width=\linewidth]{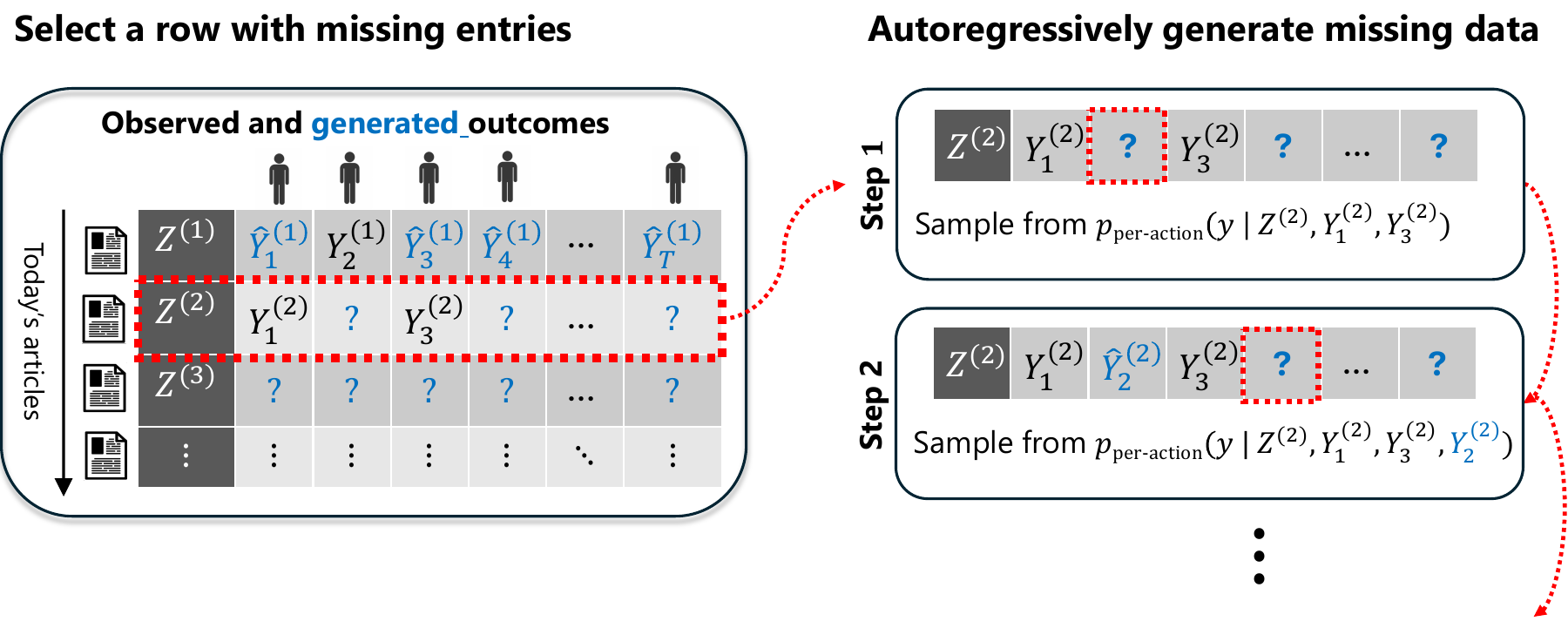}
    \caption{Imputing potential outcomes: posterior sampling via autoregressive generation (PSAR)}
    \label{subfig:psar}
  \end{subfigure}

  \begin{subfigure}{0.85\linewidth}
    \centering
    \includegraphics[width=\linewidth]{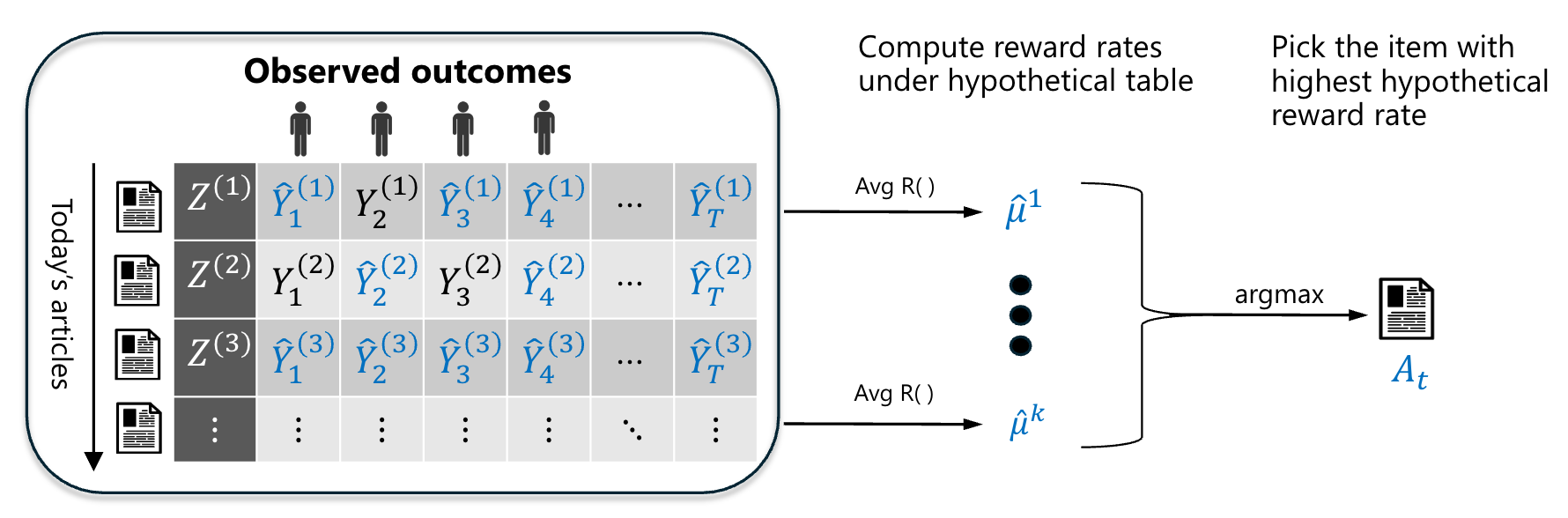}
    \caption{Making a decision given imputed potential outcomes: use the imputed potential outcomes table to choose the action with the largest average outcome over users.}
    \label{subfig:tsar}
  \end{subfigure}
    
    \caption{\bo{Thompson Sampling via Autoregressive Generation.} Subfigure (a) depicts imputing potential outcomes via autoregressive generation using Algorithm~\ref{alg:posterior_sample}. Subfigure (b) depicts making a decision using Thompson Sampling (probability matching) with the imputed potential outcomes  (Algorithm~\ref{alg:Thompson_sampling_generic}). 
}
    \label{fig:autoregressive_generation}
\end{figure}

Next, we make an exchangeability over time assumption, which means 
that the distribution of the potential outcomes table is invariant under permutation of its columns (timesteps). In our setting, where the columns correspond to the time periods in a bandit problem, this amounts to a \emph{stationarity} assumption that is inherent to most studies of multi-armed bandit problems\footnote{%
Somewhat remarkably, Thompson sampling satisfies a regret bound (relative to the best action in hindsight $A^*$) even in nonstationary settings where the exchangeability assumption fails (see Proposition \ref{prop:regret_TS_exact} and the references therein), but other standard bandit algorithms are only studied in stationary settings.}. 
\begin{assumption}[Exchangeability]
	\label{assump:exchangeable}
	For any $a\in \Aeval$ and permutation $\sigma$, \\
    $\big( Z^{(a)},  Y_{\sigma(1)}^{(a)}, Y_{\sigma(2)}^{(a)}, \ldots, Y_{\sigma(T)}^{(a)} \big)$ has the same joint distribution as $\big( Z^{(a)},  Y_1^{(a)}, Y_2^{(a)}, \ldots, Y_T^{(a)} \big)$. 
\end{assumption}
One example of when exchangeability holds is when the outcomes $(Y_1^{(a)}, Y_2^{(a)}, \ldots, Y_T^{(a)})$ are independent and identically distributed (i.i.d.), but assuming outcomes that are i.i.d. is too strong for our purposes.
For example, in the classical Bayesian setting described in Section \ref{section:classical}, the outcomes are not i.i.d.---they are potentially highly correlated due to their common dependence on the latent factor $\eta$---but they are exchangeable.
Exchangeability encodes a temporal stability that ensures early observations from playing an action are informative about future performance, allowing for a stable ``best action'' $A^*$ that an agent could hope to identify and consistently play over time.

Under Assumptions \ref{assump:articlesiid} and \ref{assump:exchangeable}, the %
procedure in Algorithm \ref{alg:posterior_sample_generic} reduces to Algorithm \ref{alg:posterior_sample} (depicted in Figure \ref{fig:autoregressive_generation}).  The main difference between these algorithms is that Assumptions \ref{assump:articlesiid} and \ref{assump:exchangeable} allow us to generate outcomes $\what{Y}^{(a)}_i$ independently across actions, and outcomes for action $a$ to depend on %
the prior information for only that action, $Z^{(a)}$.

\section{Autoregressive posterior sampling with approximate sequence models}
\label{sec:psar_ptheta}
In this section, we provide theoretical guarantees for our autoregressive generation approach to forming posterior samples when we do not know the true task distribution $p_{\TN{tasks}}^*$, but we instead have access to an approximation $\widehat{p}_{\TN{task}}$. Throughout, we work under Assumptions \ref{assump:articlesiid} and \ref{assump:exchangeable}. Specifically, we provide two main theoretical results:
\begin{enumerate}[leftmargin=*]
    \item Forming posterior samples with a sequence model $\widehat{p}_{\TN{task}}$ approximates posterior sampling up to an error term that depends on the excess prediction loss of $\widehat{p}_{\TN{task}}$ compared to $p_{\TN{tasks}}^*$ (Section \ref{sec:psar-ts}). Notably, this  guarantee holds even when sampling is conditioned on a history produced by an \emph{arbitrary}  adaptive data-collection policy.
    \item For the regret of online decision-making, we first prove a general reduction that holds for \emph{any} policy: regret against the true environment is controlled by regret in an environment consistent with $\widehat{p}_{\TN{tasks}}$, plus a penalty for the excess offline prediction loss. Specializing this reduction to Thompson sampling via autoregressive generation (TSAR) then yields a concrete regret bound in terms of the offline prediction loss of $\widehat{p}_{\TN{tasks}}$ (Section~\ref{section:regret-bounds}). %
\end{enumerate}

\subsection{Theory Part I: Autoregressive generation from approximate sequence models approximates posterior sampling}
\label{sec:psar-ts}
When can we trust that autoregressive generation will produce faithful posterior samples to guide exploration? While it is natural to expect high-quality samples when the sequence model exactly matches the data-generating distribution, this section establishes a far more general guarantee: \textit{the quality of generated posterior samples is controlled by the model's offline prediction performance, even when sampling is performed conditional on strategically collected observations}. This provides a concrete foundation for our algorithm's core premise—that accurate next-outcome prediction enables reliable posterior sampling during online exploration.

We quantify the quality of a sequence model $\widehat{p}_{\TN{per-action}}$ by its log-loss,  %
\begin{align}
    \label{eq:pop_loss}
    \ell(\widehat{p}_{\TN{per-action}}) &\triangleq  
    \E\bigg[ -\sum_{t=1}^{T} \log \, \widehat{p}_{\TN{per-action}} \Big( Y_t^{(a)} \mid Z^{(a)}, Y_{1:t-1}^{(a)} \Big) \bigg],
\end{align}
where the expectation above is  taken over $(Z^{(a)},Y^{(a)}_{1:T})\sim p_{\TN{per-action}}^*$.
We use this metric since sequence prediction models have had immense success in optimizing this exact quantity. %

First, note the following identity, which expresses the excess loss of a sequence model, $\ell(\widehat{p}_{\TN{per-action}})-  \ell(p_{\TN{per-action}}^*)$, as the average KL divergence between the prediction of $\widehat{p}_{\TN{per-action}}$ and those of the true data-generating model $p_{\TN{per-action}}^*$. \
\begin{multline}
    \ell(\widehat{p}_{\TN{per-action}}) = \ell(p_{\TN{per-action}}^*) \\
    + \E_{Z^{(a)} \sim P_Z} \left[D_{\rm KL}\left( p_{\TN{per-action}}^* \big(Y_1^{(a)}, \ldots, Y_T^{(a)} \mid Z^{(a)} \big) \; \Big\| \; \widehat{p}_{\TN{per-action}} \big( Y_1^{(a)}, \ldots, Y_T^{(a)} \mid Z^{(a)} \big)  \right) \right]. 
    \label{eq:log-loss-is-kl}
\end{multline}
For completeness, see Lemma \ref{lem:log-loss-is-kl} in Appendix \ref{app:log-loss-is-kl} for the derivation of the above identity. Informally, we say a sequence model has been successfully \emph{optimized} for offline predictions if $\ell(p_{\TN{per-action}}^*)-\ell(\widehat{p}_{\TN{per-action}}) \leq \epsilon$, for some small $\epsilon > 0$.

Our main result formalizes how offline prediction quality controls the quality of generated samples, even under adaptive data collection. Recall that  $\HH_{t-1} \triangleq \big\{ Z, \, (A_1, Y_1),\ldots, (A_{t-1}, Y_{t-1}) \big\}$  denotes the history of observations available at the start of time $t$, including unstructured prior information %
$Z=\{Z^{(a)}\}_{a\in \Aeval}$. The proof is given in Appendix \ref{app:proof_of_posterior_sample}.
\begin{prop}\label{prop:posterior_sample}
    Let Assumptions \ref{assump:articlesiid} and \ref{assump:exchangeable} hold. If, after employing some policy $\pi$ for $t-1$ time periods, Algorithm \ref{alg:posterior_sample} is applied to the history $\HH_{t-1}$ with sequence model $\widehat{p}_{\TN{per-action}}$ to generate a potential outcomes table  $\widehat{\tau}_t$, then, 
    \[
    \underbrace{ \E_\pi\big[ D_{\rm KL}\big(\mathbb{P}(\widehat{\tau}_t \in \cdot  \mid \HH_{t-1}) \, \| \, \mathbb{P}\left(\tau \in \cdot  \mid \HH_{t-1} \right) \big) \big] }_{\TN{Discrepancy %
    in online generated samples}} ~~ \leq ~~ \underbrace{|\Aeval| \left[\ell(\widehat{p}_{\TN{per-action}}) - \ell(p_{\TN{per-action}}^*) \right].}_{\TN{Discrepancy 
    in offline predictions}}
    \]
    Moreover, the data processing inequality implies that for any function $f$,
        \[
 \E_\pi\big[ D_{\rm KL}\big(\mathbb{P}(f(\widehat{\tau}_t) \in \cdot \mid \HH_{t-1} ) \, \| \, \mathbb{P}(f(\tau) \in \cdot \mid \HH_{t-1} \big) \big] ~~ \leq ~~  |\Aeval| \left[ \ell(\widehat{p}_{\TN{per-action}}) - \ell(p_{\TN{per-action}}^*) \right]
    \]
\end{prop}
Note that the expectation in the proposition is taken over the history under $\pi$. This result relates offline and online performance. When the model achieves near-optimal prediction loss, the bound implies that in an information-theoretic sense it is difficult to distinguish our algorithm's samples from true posterior samples. The second part of the proposition extends this guarantee to any derived quantity, like mean rewards or optimal actions, that we might compute from these samples.

This result formalizes %
two deep properties of 
our algorithm. First, it shows that incremental errors in autoregressive generation do not severely compound across iterations. Second, it shows that a single measure of offline sequence prediction loss directly controls the quality of posterior generations for \textit{any} given process for generating the history $\HH_{t-1}$. This may be surprising, as theory in reinforcement learning typically requires certain distributions under the policy used to collect offline data to nearly match those under the online policy \citep{kakade2002approximately, antos2008learning, ross2010efficient}.

\subsection{Theory Part II: Exploration using approximate sequence models $\widehat{p}_{\TN{per-action}}$ results in low regret}
\label{section:regret-bounds}
This section formally shows how successfully optimizing next-outcome prediction on an offline dataset reliably enables online exploration with low regret. Recall from Section \ref{sec:meta-bandit-formulation}, we measure the agent's performance through expected per-period regret:
\begin{equation}\label{eq:regret}
    \Delta(\pi; \MC{B})= \E_{\pi}\left[ \max_{a\in \Aeval} \frac{1}{T} \sum_{t=1}^{T} \left(R(Y_t^{(a)}) - R(Y_t^{(A_t)})\right)\right].
\end{equation}
The expectation above is taken over the joint draw of unstructured prior information and potential outcomes $(Z^{(a)}, Y_{1:T}^{(a)})_{a\in \Aeval}$ from $p^*_{\TN{per-action}}$ as well as any additional randomness in the actions selected by policy $\pi$.

\subsubsection{A general reduction: optimized simulators reflect true regret}
Our analysis relies on interpreting an approximate sequence model as inducing a simulator for bandit tasks, as in the definition below. One could use such a simulator to train policies offline using reinforcement learning techniques, or, as we do, to perform online planning.

\begin{definition}[Bandit simulator]
\label{def:simulator} Let Assumptions \ref{assump:articlesiid} and \ref{assump:exchangeable} hold. A (possibly approximate) sequence model $\widehat{p}_{\TN{per-action}}$, together with the true marginal distribution $Z^{(a)}$, induces a stochastic simulator of bandit policies as follows. First draw  $Z^{(a)}$ independently for each $a \in \Aeval$. Then, at any time $t$, given the history $\HH_{t-1}$, algorithm $\pi$ selects some action $A_t$ and observes an outcome
    \[ 
    Y_t \mid \HH_{t-1}, A_t \sim \widehat{p}_{\TN{per-action}}\Big( \cdotspace \mid Z^{(A_t)}, \, (Y_j)_{j\in F_t^{(A_t)}} \Big),
    \]
    which is added to the history. Above, we use $F_t^{(a)} = \{(a,i) : i\leq t-1, A_i=a\}$ to denote the ordered sequence of time indices at which action $a$ has been chosen in the history $\HH_{t-1}$.
\end{definition}
It is not hard to show that, according to this definition, the simulator induced by $p_{\TN{per-action}}^*$ perfectly matches the true data-generating process. 
Let $\Delta(\pi; \widehat{p}_{\TN{per-action}})$ denote the expected regret policy $\pi$ incurs under the simulated environment and reserve $\Delta(\pi; p_{\TN{per-action}}^*)$ for the true regret defined earlier in \eqref{eq:regret}. 
\begin{prop}\label{prop:sim-to-real-gap} For any policy $\pi$ and any sequence model $\widehat{p}_{\TN{per-action}}$, 
    \begin{align}
    \underbrace{ \Delta \big( \pi; p_{\TN{per-action}}^* \big) }_{\TN{Deployment regret}} 
    ~~ \leq ~~ \underbrace{ \Delta \big( \pi; \widehat p_{\TN{per-action}} \, \big) }_{\TN{Regret under simulator}} 
    + \underbrace{ \sqrt{ (|\Aeval| / 2) \left\{ \ell(\widehat{p}_{\TN{per-action}}) - \ell(p_{\TN{per-action}}^*) \right\} } }_{\TN{Penalty for sub-optimal simulator}}.
    \label{eqn:deploymentRegret}
\end{align}
\end{prop}
\noindent The proof is in Appendix  \ref{app:proof_of_sim_to_real_gap}. 
Similar statements apply to simulated quantities beyond regret.  %
Informally, if one were training a sequence model $\hat{p}_{\TN{per-action}}$ from scratch, we expect $\ell(\widehat{p}_{\TN{per-action}}) - \ell(p_{\TN{per-action}}^*)$ goes to zero as the size of the dataset used to fit $\widehat{p}_{\TN{per-action}}$ grows, as long as the sequence model class chosen is sufficiently flexible.

As a corollary of the general reduction above, we attain a regret guarantee for our algorithm when deployed with an offline-optimized sequence model. 
\begin{corollary}
    \label{thm:psarRegret}
    Under Algorithm \ref{alg:Thompson_sampling_generic} applied with $\widehat{p}_{\TN{per-action}}$ (denoted $\psar(\widehat{p}_{\TN{per-action}})$),
    \begin{align*}
        \Delta \big( \psar(\widehat{p}_{\TN{per-action}} ); \, p_{\TN{per-action}}^* \big) 
        \leq 
        \underbrace{ \sqrt{ \frac{ |\Aeval| \log (|\Aeval|) }{2 T} } }_{\TN{Regret bound for TS}}
        ~ + ~~ \underbrace{ \sqrt{ \frac{ |\Aeval| }{2} \big\{ \ell(\widehat{p}_{\TN{per-action}}) - \ell(p_{\TN{per-action}}^*) \big\} } }_{\TN{Penalty for sub-optimal prediction}}.
    \end{align*}
\end{corollary}
\textit{Proof sketch.} Corollary \ref{thm:psarRegret} follows from Proposition \ref{prop:sim-to-real-gap} after bounding \\
$\Delta \big( \psar(\widehat{p}_{\TN{per-action}} ); \, \widehat{p}_{\TN{per-action}} \big)$. Our proof interprets this as studying the regret of correct probability matching, 
$\mathbb{P}(A_t = a \mid \HH_{t-1}) = \mathbb{P}( A^*=a \mid \HH_{t-1})$ applied in the environment $\widehat{p}_{\TN{per-action}}$ defined in Definition \ref{def:simulator}. Although the simulated environment $\widehat{p}_{\TN{per-action}}$ could be complex and nonstationary, such a regret bound is already known in the literature. \citet{bubeck2015bandit} and \citet{bubeck2016multi} observed that the information theoretic analysis of \citet{russo2016information}, originally derived for more restricted class of models, applies essentially without modification to bound regret in this setting. See Appendix \ref{app:proof_of_TS_regret} for details. \hfill $\blacksquare$ \\

\citet{wen2021predictions} prove a regret bound that superficially looks similar to ours, however their result only applies to binary outcomes and requires a specific conditional KL divergence to be small, which does not appear to follow from training a model with low validation loss.
Moreover, we use a very different proof technique. Our Corollary \ref{thm:psarRegret} relies on Proposition \ref{prop:sim-to-real-gap}, which is a result that may be of independent interest.

\subsubsection{The price of informed exploration: prior sensitivity of Thompson sampling and the tightness of our regret bound.}
\label{sec:misspecifiedTS}

Under the simulator induced by $\widehat{p}_{\TN{per-action}}$, actions that the model believes are unlikely to be optimal given prior information $Z$ are therefore effectively screened out and receive little direct exploration. As formalized in Section \ref{sec:screening}, in the regimes that motivate informed exploration---many candidate actions relative to the horizon, such as recommendation systems---this screening effect is precisely what can enable strong performance. 

The natural question, then, is what this aggressive screening costs when the model's initial beliefs are wrong. An action that is screened out generates no data that could correct the model's mistaken beliefs about it, so such errors are never washed out by additional interactions, even over long horizons. The persistent impact of misspecified priors differs markedly from passive learning, where prior effects typically fade with more data. This intuition helps explain why the penalty term in \eqref{eqn:deploymentRegret} does not decay with $T$, and why per-period regret need not vanish under misspecification. Indeed, this phenomenon is closely related to the known sensitivity of Thompson sampling to prior misspecification. Classical frequentist regret bounds for Thompson sampling establish sublinear cumulative regret under fixed priors with full support \citep{agrawal2012analysis}, which may suggest that the cost of prior misspecification is limited because the influence of the prior is eventually overcome by data. However, this interpretation can be misleading: worst-case analyses show that the sensitivity of Thompson sampling to prior misspecification can scale quadratically with the horizon \citep{liu2016prior,simchowitz2021bayesian}.\footnote{For example, \citet{liu2016prior} show that cumulative regret in the bad-prior regime scales as $O(\sqrt{T/\rho})$, where $\rho$ is the prior mass assigned to the true model. Thus, if $\rho=1/T^3$, this dependence yields a $T^2$ cumulative regret penalty.} The apparent tension is resolved by noting that classical regret guarantees fix a problem instance and study performance as the horizon grows, whereas worst-case misspecification analyses allow the hardest problem instance to change with the horizon. These hardest instances can place progressively less prior mass on the truth, increasing the regret incurred before the data can overcome the prior.

To contextualize our result within this literature, and to verify that the non-vanishing penalty in Proposition~\ref{prop:sim-to-real-gap} reflects the problem rather than slack in our analysis, we specialize to the classical setting of Bayesian bandits with a misspecified prior. In this special case, autoregressive generation reduces to sampling from a posterior predictive distribution akin to \eqref{eq:posterior_formula}, and Corollary~\ref{thm:psarRegret} yields a regret bound for Thompson sampling under prior misspecification. This provides a short information-theoretic route to a prior-sensitivity guarantee. In contrast, \citet[Corollary~3.1]{simchowitz2021bayesian} study prior sensitivity for algorithms including a ``$k$-shot'' variant of Thompson sampling through an intricate analysis, showing that $\varepsilon$ misspecification in total variation distance can inflate per-period regret by a term growing linearly with the horizon $T$. In a few lines, our reduction shows that $\varepsilon$ misspecification in KL divergence inflates per-period regret by a term independent of $T$.

\begin{example}[Prior Sensitivity of Thompson sampling]\label{ex:prior-sensitivity}   
Consider a pair of Bayesian generative models that differ only in their prior distributions:
    \begin{enumerate}[leftmargin=*]
        \item {\bf True Bayesian model:} A sequence $(Y_{1}^{(a)}, \ldots, Y_{T}^{(a)}) \mid Z^{(a)} \sim p_{\TN{per-action}}^*( \cdotspace \mid Z^{(a)})$ is generated by first drawing a latent parameter $\eta^{(a)} \sim \nu_*( \cdotspace \mid Z^{(a)})$, then sampling i.i.d. outcomes $Y_{1}^{(a)}, \ldots, Y_{T}^{(a)} \mid \eta^{(a)} \overset{\text{i.i.d.}}{\sim} q(\cdotspace \mid \eta^{(a)})$.
        \item {\bf Model with an incorrect prior:} A sequence $(Y_{1}^{(a)}, \ldots, Y_{T}^{(a)}) \mid Z^{(a)} \sim \widehat{p}_{\TN{per-action}}( \cdotspace \mid Z^{(a)})$ follows the same generative process but with prior $\eta^{(a)} \sim \nu_{\theta}( \cdotspace \mid Z^{(a)})$, while maintaining an identical likelihood $q(\cdotspace \mid \eta^{(a)})$.
    \end{enumerate}
    Then,\begin{align}\label{eq:mispecified-prior-regret}
        \Delta \big( \psar(\widehat{p}_{\TN{per-action}} ); \, p_{\TN{per-action}}^* \big) \leq 
        \underbrace{ \sqrt{ \frac{ |\Aeval| \log (|\Aeval|) }{2 T} } }_{\TN{Regret bound for Thompson sampling}}
        + \underbrace{ \sqrt{ \frac{ |\Aeval| }{2} D_{\rm KL}(  \nu_{*}  \; \big\| \; \nu_{\theta}  )}  }_{\TN{Penalty for misspecified prior}}.
    \end{align}
    Above, \eqref{eq:mispecified-prior-regret} holds by Corollary \ref{thm:psarRegret} due to the following bound on $\ell(p_{\TN{per-action}}^*)-\ell(\widehat{p}_{\TN{per-action}})$:
    \begin{align*}
    \ell(p_{\TN{per-action}}^*)&-\ell(\widehat{p}_{\TN{per-action}}) = \E_{Z^{(a)}\sim P_Z}\left[ D_{\rm KL}\left( \mathbb{P}_{p_{\TN{per-action}}^*}\big(Y_{1:T}^{(a)} \mid Z^{(a)} \big) \; \Big\| \; \PP_{\widehat{p}_{\TN{per-action}}} \big(  Y_{1:T}^{(a)}  \mid Z^{(a)}\big)  \right) \right]\\
    &\leq  \E_{Z^{(a)}\sim P_Z}\left[D_{\rm KL}\left( \mathbb{P}_{p_{\TN{per-action}}^*} \big(\eta^{(a)}, Y_{1:T}^{(a)}  \mid Z^{(a)}\big) \; \Big\| \; \mathbb{P}_{\widehat{p}_{\TN{per-action}}}\big(\eta^{(a)},   Y_{1:T}^{(a)} \mid Z^{(a)} \big)  \right) \right] \\ 
    &= \E_{Z^{(a)}\sim P_Z}\left[D_{\rm KL}\left(  \nu_{*}(\cdotspace \mid Z^{(a)}) \; \Big\| \; \nu_{\theta}(\cdotspace \mid Z^{(a)})    \right)\right] 
    \end{align*}
    The inequality is the data-processing inequality and the final equality uses the chain rule together with the fact that both generative models share a common likelihood. To simplify further, assume prior information $Z^{(a)}$ takes on only one possible value.
\end{example}

The bound \eqref{eq:mispecified-prior-regret} allows the per-period
expected regret to remain inflated by a non-vanishing KL divergence term. An upper bound alone cannot show that this persistence is real rather than an artifact of the analysis; our lower bound does. In Appendix \ref{app:lower_bound_instance}, we provide a lower bound example
where even with minimal prior misspecification (differing only by a $\log(T)$
term), the per-period regret remains constant, aligning with known lower bounds
in related settings \citep{liu2016prior}. In this sense the
$T$-independence of the penalty term in Corollary \ref{thm:psarRegret} is
tight; we do not claim the bound is tight in all settings.

\section{Related work}
\label{section:related-work}
\vspace{-1.5mm}

\paragraph{Core ideas from statistical sequence prediction.} As discussed in Section \ref{sec:key_ideas} and Figure \ref{fig:missing-data-view}, our work builds on a classical vision in statistics where inference is performed solely through potentially observable quantities, rather than through latent parameters. This perspective traces back to de Finetti's work on exchangeable sequences \citep{DeFinetti33} and appears across multiple literatures \citep{Dawid84, Geisser70, BertiReRi98, FortiniLaRe00, FortiniPe14, HahnMaWa18, BertiDrPrRi21, BertiDrLePrRi22, FongHoWa23, osband2022neural, lee2023martingale, osband2024epistemic}.  Among these, our initial efforts on this direction were especially inspired by \citet{ osband2024epistemic}, who uses sequence prediction loss to \emph{evaluate} the quality of  epistemic uncertainty quantification. Several authors recently proposed neural network architectures for learning the joint distribution / posterior predictives~\citep{NguyenGr22, MullerHoArGrHu22, GarneloRoMaRaSaShTeReEs18}. While some of these papers include Bayesian bandit experiments, they do not use autoregressive generation to quantify epistemic uncertainty (see Section \ref{subsec:failure_of_other_generation} for further discussion). %

\vspace{-1.5mm}
\paragraph{Sequential decision-making with foundation models.} Recent work leveraging foundation models for sequential decision-making broadly falls into three categories \citep{yang2023foundation}. The first involves training sequence models on expert demonstrations. Some of these works try to directly imitate the typical expert, while others uses goal-conditioned action selection to improve upon demonstrated behavior \citep{janner2021offline,decisionTransformer,ding2019goal}; while effective in some settings, this approach can provably fail when reward signals are noisy \citep{brandfonbrener2022does,malenica2023causality,russo2026success}. 
The second category explores variants of in-context decision making---decision-making where a fixed sequence model adapts by conditioning on its interaction history with the current environment---where pretrained foundation models are carefully prompted without additional fine-tuning
\citep[see e.g.][]{huang2022language, yao2023react,song2026reward}; these include papers recent work that discuss the challenges of using LLMs for in-context
exploration in bandit problems \citep{nie2025evolve,krishnamurthy2024can,harris2025should}. Another closely related work by
\citet{tan2026pfn} leverages TabPFN \citep{HollmannMuPuKrKoHoScHu25} for exploration; our approach differs from theirs because they quantify uncertainty using a sub-sampled central limit theorem, rather than through our autoregressive generation approach. %
The third category develops generative models as learned environment simulators \citep[see e.g.][]{ha2018world, kaisermodel,hafner2021mastering}. Section \ref{section:regret-bounds} interprets our sequence models as simulators, but our work distinctively uses these simulator models for online autoregressive generation to quantify epistemic uncertainty and enable principled exploration.

\vspace{-1.5mm}
\paragraph{Sequence modeling in bandits.} 
A particularly relevant approach in this space is the Decision Pretrained Transformer (DPT) \cite{lee2023incontext}, which implements Thompson sampling by pretraining a sequence model to predict the reward-maximizing action. This elegant approach works by observing that optimal in-context 
prediction of the reward-maximizing action necessarily models its posterior distribution. Our work develops a broader framework for uncertainty quantification through autoregressive generation of missing outcomes, an idea that is applicable well beyond Thompson sampling. While the works share connections in implementing Thompson sampling, their core insights and implementations differ substantially, each with distinct advantages. DPT learns a ``generative model of behavior'' \cite[Sec.~3.2]{yang2023foundation}---predicting the best action based on past histories. %
In contrast, we learn a ``generative model of the world'' \cite[Sec.~3.3]{yang2023foundation}---predicting reward/outcome sequences for each action individually. We conjecture this is key to our unique theoretical guarantees linking \emph{offline} sequence modeling to \emph{online} performance.
 Given these complementary strengths, it would be interesting to explore hybrid approaches, where our models could serve as offline simulators to distill DPT-style policies, which reduce online computation.

Autoregressive predictive models are also employed by \citet{liu2023nonstationary}, but with very different goals and insights. Their work recognizes that Thompson sampling can over-explore in nonstationary environments, where constant environmental changes introduce fresh uncertainty that may not be worth resolving. They address this through a clever conditioning of imagined reward trajectories in their Bayesian model. While our work shares the use of generating potential rewards, we do so with an entirely different objective---replacing explicit Bayesian inference with a more scalable approach rather than moderating Thompson sampling's exploration. Later works that approximate predictive sampling with neural networks \citep{zhu2023non} rely on deep ensembles trained via online stochastic gradient descent, whereas our methods do not train/fine-tune the sequence model at decision-making time; our sequence models are fixed prior to the line decision-making phase and learn in-context. 

\vspace{-1.5mm}
\paragraph{Other approaches to neural network inference in bandit problems.} 
Beyond approaches that leverage sequence modeling, there are several approaches to implement Thompson sampling with neural networks. The first places a Bayesian prior on the neural network weights themselves, either through Bayesian linear regression on the final layer \citep{riquelme2018deep,snoek2015scalable} or full Bayesian neural networks \citep{zhang2020neural}. A second thread uses ensembles of neural networks to approximate posterior sampling \citep{osband2018randomized,lu2017ensemble,ensembleSampling}, including more efficient variants like Epinets \citep{osband2024epistemic,zhu2023scalable,osband2023approximate} and HyperModels \citep{dwaracherla2020hypermodels,hyperagent} that reduce computational overhead. Our work offers a different perspective entirely: rather than placing uncertainty on network parameters or using ensembles, we model uncertainty through autoregressive generation of missing outcomes, allowing us to leverage the full power of modern sequence models while maintaining principled uncertainty quantification. Moreover, we emphasize informed meta-problems where leveraging rich priors from unstructured information is crucial. This setting poses fundamental challenges for existing approaches. Ensemble methods typically rely on bootstrapping data to inject variance, but at the start of a decision problem, there is no data to bootstrap, leaving no basis for principled exploration. While randomized prior functions \cite{osband2018randomized} attempt to address this through random initialization, our approach is fundamentally different---extracting meaningful priors implicitly through training and scaling laws, and updating beliefs by through in-context learning 
rather than incremental gradient descent.

\vspace{-1.5mm}
\paragraph{Meta-bandit learning.} 
Prior work on meta-bandit %
learning has primarily focused on parametric settings, studying how to fit hyperparameters for Thompson sampling \citep{bastani2022meta, kveton2021meta, metalearningSong} or linear UCB algorithms \citep{cella2020meta}. While these approaches demonstrate convergence to oracle performance given sufficient meta-examples, they differ fundamentally from our work which tackles informed exploration with unstructured information and relies on autoregressive generation rather than parameter inference. An alternative approach learns exploration policies from scratch using reinforcement learning, potentially discovering novel exploration strategies but current approaches require more complex training procedures and larger sample sizes than our sequence prediction framework \citep{duan2016rl,wang2016learning,oh2020discovering,oh2025discovering}.

\section{Experiments}
\label{sec:experiments}
We evaluate our algorithm in a synthetic setting and in a semi-realistic news recommendation setting. 
We focus on settings with binary outcomes and rewards ($R(Y_t)=Y_t \in \{0, 1\}$) 
since the news dataset we build on 
has binary click/no-click outcomes. 
Both settings satisfy the structural assumptions introduced in Section~\ref{subsubsec:assumptions}, so that the version of TSAR implemented is Algorithm~\ref{alg:Thompson_sampling_generic} using Algorithm~\ref{alg:posterior_sample} as a subroutine.
Additionally, in both settings, the methodology for obtaining $\widehat p$ is as described in Section~\ref{sec:pretrainHistorical}. %
Additional details and results are in Appendix \ref{sec:exp_details}. 

In Section~\ref{sec:synthetic_experiments} we use a synthetic setting to validate our approach and see that with our pretraining strategy, our algorithm almost exactly matches the decision-making performance of a classical version of Thompson sampling with a correctly specified Bayesian model. %
Then, in Section~\ref{sec:MIND}, we demonstrate our method scales to %
a news recommendation setting where articles are associated with text headlines, and where $\widehat p$ incorporates a language model. %

\subsection{Learning an approximate sequence model $\widehat{p}_{\TN{per-action}}$ from historical data}
\label{sec:pretrainHistorical}

In this section, we present an approach to learning an approximate sequence model $\widehat{p}_{\TN{per-action}}$ using historical bandit task data under Assumptions~\ref{assump:articlesiid} (independence across actions) and \ref{assump:exchangeable} (exchangeability) as described in Section~\ref{subsubsec:assumptions}. We will use this approach throughout our experiments. %
\emph{Note that this is just one approach among many to form $\widehat{p}_{\TN{per-action}}$.} Some additional ways to obtain $\what p_{\text{per-action}}$ are discussed in Section~\ref{section:approximate}.

We assume the agent has access to a large amount of historical data $\Dtrain$ from past bandit tasks, with data also drawn from $p_{\TN{per-action}}^*$. 
This historical data is used to form $\what p_{\text{per-action}}$ in an offline learning phase, prior to decision-making. %
Note that this historical data is distinct from recent history within the current online bandit task, $\mathcal H_{t-1}$, as defined in \eqref{eqn:history}. %
Specifically, this historical data consists of
\begin{align}
    \Dtrain \triangleq \big\{ \big( Z^{(a)}, Y_{1:T}^{(a)}\big) \}_{a\in \mathcal \Ahist } %
    \qquad \TN{where} \qquad 
    \big( Z^{(a)}, Y_{1:T}^{(a)}\big) \iidsim p_{\TN{per-action}}^* %
\end{align}
and $a\in \Ahist$ indexes a large historical set of previous actions, their features $Z^{(a)}$, and their potential outcomes $Y_{1:T}^{(a)}$, with actions $a$ sampled across many previous bandit tasks. In practice, to augment the training data, and also to promote exchangeability (Assumption~\ref{assump:exchangeable}), we also train on bootstrapped samples $\tilde{Y}_{1:T}^{(a)}$ formed as follows:
\begin{align}
    \tilde{Y}_1^{(a)}, \dots, \tilde{Y}_T^{(a)} \mid Y_{1:T}^{(a)} \iidsim \frac{1}{T} \sum_{t=1}^T \delta_{Y_t^{(a)}}.
    \label{eqn:bootstrap}
\end{align}
Note that by assuming we have access to $Y_{1:T}^{(a)}$ for each action $a$, we are assuming we have access to ``complete'' historical bandit task data, in the sense that we observe all potential outcomes for any given action. If such ``complete'' historical bandit data is not available, in practice, during %
training, we can impute a full dataset by bootstrap resampling rewards across timesteps $t$, where the bootstrap resampling is with replacement and done independently for each action.

Using $\Dtrain$, we fit a sequence model $\widehat p_{\TN{per-action}}$ to approximate $p_{\TN{per-action}}^*$ by minimizing the empirical version of the %
log-loss of $\widehat{p}$ from \eqref{eq:pop_loss}: 
\begin{align}
    \label{eq:train_loss}
    \ell(\widehat p_{\TN{per-action}}; \Dtrain) \triangleq - \sum_{a \in \Ahist} \sum_{t=1}^T \log \widehat p_{\TN{per-action}} \big( Y_t^{(a)} \mid Z^{(a)}, Y_{1:t-1}^{(a)} \big).
\end{align}
\noindent Figure 
\ref{fig:masking}
visualizes the training objective, and Algorithm~\ref{alg:offline} describes how to optimize $\widehat p_{\TN{per-action}}$ via stochastic gradient descent-based approaches. %

\begin{figure}[h]
    \centering
    \includegraphics[width=0.9\linewidth]{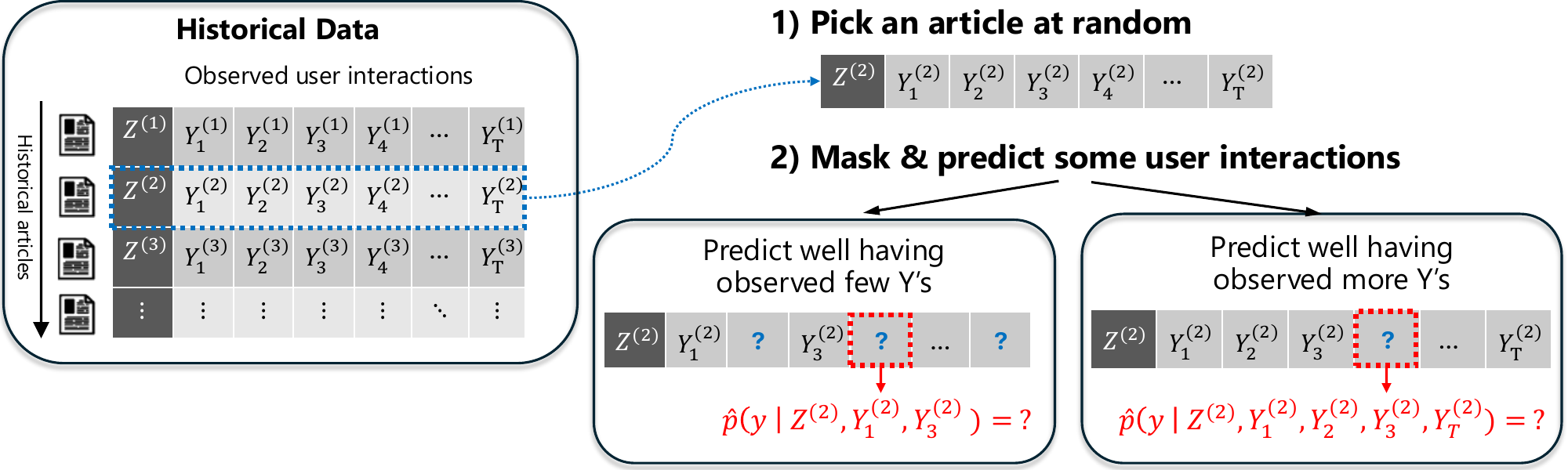}
    \caption{\bo{Training objective.} To attain low prediction loss, the sequence model must \textit{implicitly} perform Bayesian inference, i.e., the model must learn to balance information from a ``prior'' based on the text $Z$ and information as it observes more $Y$'s. Intimate connections between sequence prediction and Bayesian inference are studied in \citep{dawid1984present, barron1998minimum, garnelo2018neural,jha2022neural,NguyenGr22,MullerHoArGrHu22,lee2023martingale}. }
    \label{fig:masking}
\end{figure}

\begin{algorithm}[h]
  \caption{Offline training of a sequence model}
  \label{alg:offline}
  \begin{algorithmic}[1]
    \Require Training data $\Dtrain$, Sequence model class $\{\widehat p_\theta\}_{\theta\in\Theta}$
    \For{epoch in $1,\ldots,E$}
        \State Sample a mini-batch $\MC{A}^{\TN{mini-batch}} \subset \Ahist$
        \State Compute loss \eqref{eq:train_loss} with $\big\{ Z^{(a)}, Y_{1:T}^{(a)} \big\}_{a \in \MC{A}^{\TN{mini-batch}}}$ (or bootstrapped samples $\tilde{Y}_{1:T}^{(a)}$ from \eqref{eqn:bootstrap})
        \State Backpropagate and take a gradient descent step to update $\theta$ in $\widehat p_\theta$ %
     \EndFor
  \end{algorithmic}
\end{algorithm}
\subsubsection{Sequence model architectures}
\label{sec:seq_model_architectures}
In our experiments, we focus on settings where $Z^{(a)}$ consists of either numeric or text features (e.g., headlines in a news recommendation setting). The choice of architecture for the sequence model $\widehat p_{\TN{per-action}}$ is flexible. Our theoretical results in Sections \ref{sec:psar-ts} and \ref{section:regret-bounds} establish that any architecture achieving strong predictive performance is suitable. Below, we describe two architectural approaches that we use, which offer different tradeoffs; this is by no means an exhaustive list of potential implementations for $\what p_{\text{per-action}}$ %
(for more examples, see Section~\ref{section:approximate}).  %

\paragraph{Structured Bayesian models.}  When strong domain knowledge exists about the outcome distribution, we can incorporate it %
into the model architecture. For binary outcomes ($Y_{t}^{(a)}\in \{0,1\}$), we can use a Beta-Bernoulli model where a neural network maps $Z^{(a)}$ (e.g. text features) to the Beta prior parameters $\widehat \alpha(Z^{(a)})$ and $\widehat \beta(Z^{(a)})$: %
\begin{equation}\label{eq:ppd}
\widehat p_{\TN{per-action}} \big( Y_{t+1}^{(a)} = 1 \mid Z^{(a)}, Y_{1:t}^{(a)} \big) = \frac{\widehat \alpha(Z^{(a)}) + \sum_{j=1}^{t} Y_j^{(a)}}{\widehat \alpha(Z^{(a)}) + \widehat \beta(Z^{(a)}) + t}
\end{equation}
Here, $\widehat\alpha(Z^{(a)})$ and $\widehat \beta(Z^{(a)})$ are learned functions that transform text features into Beta prior parameters and \eqref{eq:ppd} is the corresponding ``posterior predictive'' distribution. This architecture explicitly encodes conjugate Bayesian updating while allowing the prior to be fit from data in the style of empirical Bayes methods \citep{casella1985introduction}. For a more thorough discussion on the connection to Bayesian sequence modeling, see Appendix \ref{subsec:bayes-seq-models}.

\paragraph{Simple neural network architecture.} 
One approach %
uses %
separate 
encoders for prior information $Z^{(a)}$ 
and outcome sequences $Y_{1:t}^{(a)}$, whose outputs are combined by a final multi-layered perceptron (MLP) to 
form predictions. The encoder for $Z^{(a)}$ depends on the input type: a language model (e.g. DistilBERT, \citep{sanh2019distilbert}) 
for text, or an MLP for numeric inputs. This encoder and the final MLP are fine-tuned end-to-end. 
In contrast, the encoder for outcomes is a fixed function extracting sufficient statistics, 
which is appropriate because the outcome distribution has known sufficient statistics. 
Figure~\ref{fig:babysequence} visualizes this architecture for the case where 
$Z^{(a)}$ is text (news experiments, Section~\ref{sec:MIND}); the synthetic 
experiments (Section~\ref{sec:synthetic_experiments}) in which $Z^{(a)}$ is a vector use the same architecture 
except with an MLP as the $Z^{(a)}$ encoder.

\begin{figure}[t!]
    \centering
    \includegraphics[width=0.55\linewidth]{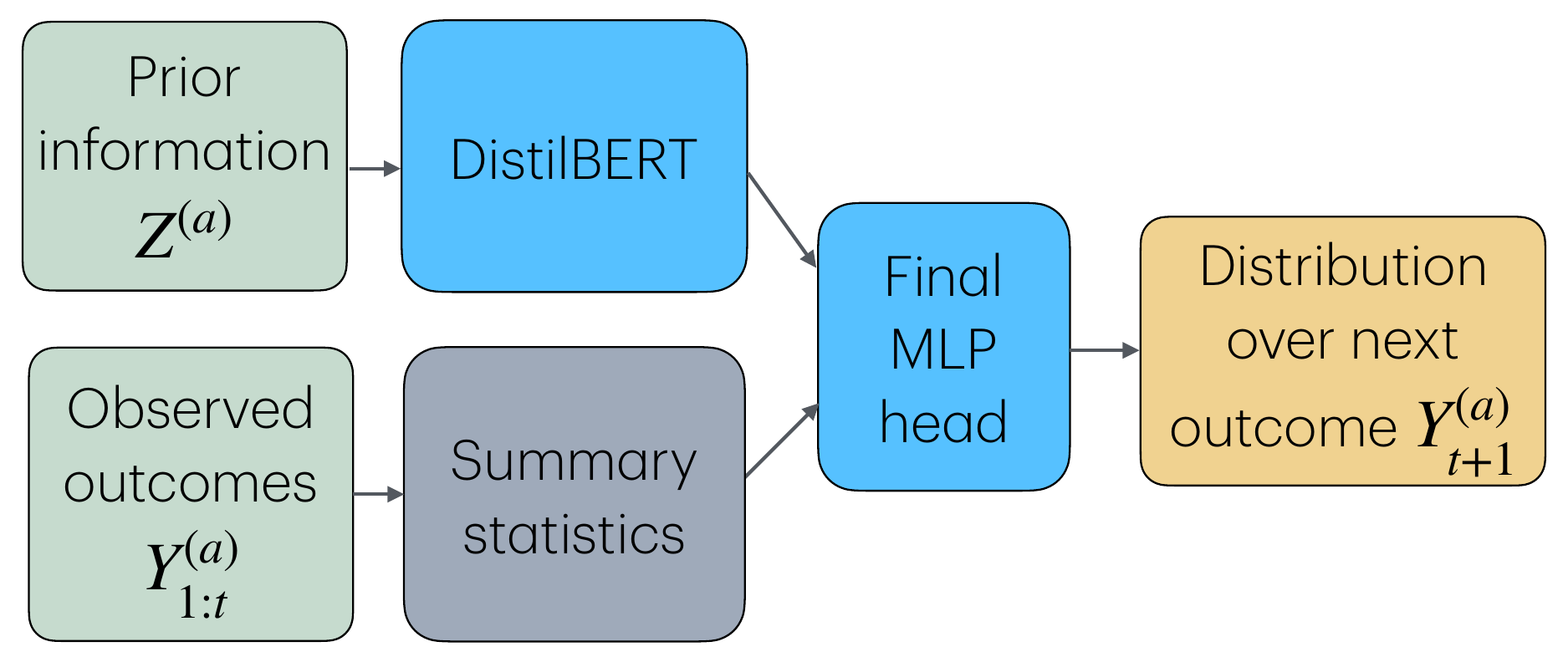}
    \caption{Simple neural network sequence model architecture for when $Z^{(a)}$ is text. %
    Each box, aside from ``Outcome encoder'', represents a neural network with trainable parameters; the outcome encoder is a fixed function that extracts sufficient statistics of the history. %
    Arrows indicate data flow between inputs/outputs and components of the neural network model. %
    The components depicted in blue boxes (models) are end-to-end fine-tuned in the offline learning phase. \vspace{-3mm}
    }
    \label{fig:babysequence}
\end{figure}

\subsection{Synthetic setting: Mixture Beta-Bernoulli}
\label{sec:synthetic_experiments}
Despite never defining a Bayesian model with explicit latent parameters, we find that our Thompson sampling via autoregressive generation (TSAR) algorithm matches the performance of an oracle Thompson sampling method that has %
access to the true prior and likelihood. We show this in a synthetic environment where the true data-generating process is known. The true data-generating process uses a Bernoulli likelihood for outcomes, with a mixture of Beta distributions as the prior, where the Beta parameters are determined by $Z^{(a)} \in \real^2$. See Appendix~\ref{app:synthetic_appendix} for more details. 

\subsubsection{Decision-making algorithms}

\paragraph{\textsc{TSAR} variants.} 
We implement \textsc{TSAR} %
with two different types of sequence models as described in Section~\ref{sec:seq_model_architectures}:
\begin{enumerate*}[label=(\roman*)] 
    \item \textsc{Flexible NN} is an MLP that takes  $Z^{(a)}$ and a summary of past outcomes for the action as input to predict probabilities, and
    \item \textsc{Beta-Bernoulli NN} is a model matching the Beta-Bernoulli posterior predictive form, where a neural network maps  $Z^{(a)}$ to the two Beta parameters. %
\end{enumerate*}
Both models are trained offline using gradient descent on loss~\eqref{eq:train_loss} and later used to power online decision-making under new bandit tasks (details in Appendix~\ref{app:synthetic_appendix}).

Our primary comparison is between \textsc{TSAR Flexible NN}, which uses the trained \textsc{Flexible NN} for autoregressive generation, and \textsc{TS Oracle}, which performs Thompson sampling with exact Bayesian inference using the true generative model.

\paragraph{Truncating generation lengths.} 
For large population sizes $T$, generating all $T$ missing outcomes during online decision-making (Algorithm~\ref{alg:posterior_sample}) is computationally expensive. We therefore modify PSAR to generate only $m=500$ missing outcomes per action and use their average to estimate the mean reward $\mu^{(a)}(\tau)$ from \eqref{eq:mean_reward_tau}.
Our experiments show this approximation remains effective when $m$ is sufficiently large (see  Section~\ref{sec:truncationExp}). %

\paragraph{Baselines.} 
We consider two additional Thompson Sampling (TS) based baselines, none of which match the performance of the oracle Bayesian procedure.
\begin{enumerate*}[label=(\roman*)] 
    \item \textsc{TS Beta-Bernoulli (Uniform Prior)} performs exact Thompson sampling with a misspecified uniform prior that ignores $Z^{(a)}$;
    \item %
    \textsc{TS Neural Linear} implements Thompson sampling with a Gaussian linear regression, where the regression features consist of the output from a neural network trained to predict outcomes given $Z^{(a)}$ as input, adapting ideas from neural contextual bandits  \citep{riquelme2018deep}.
\end{enumerate*}
Finally, we also consider the \textsc{UCB} algorithm from  Section 6 of \citet{Abbasi-YadkoriPaSz11}. 
Full details on baselines are in Appendix~\ref{sec:bandit_baselines}. 

\begin{figure}[t]
    \vspace{-4mm}
\vspace{-8pt}
    \centering
  \begin{minipage}[b]{0.45\textwidth}
    \begin{minipage}[b]{\linewidth}
      \centering
      \includegraphics[trim=5pt 0pt 0pt 0pt, clip,width=\linewidth]{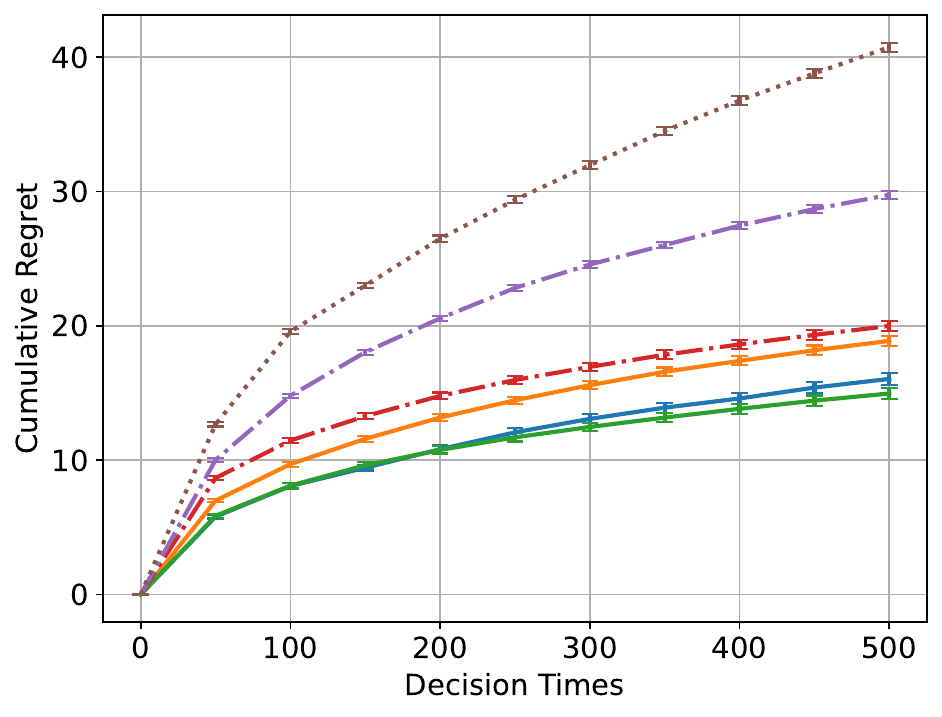}
    \end{minipage}

  \end{minipage}
      \begin{minipage}[b]{0.53\textwidth}
    \centering
    \includegraphics[width=\linewidth]{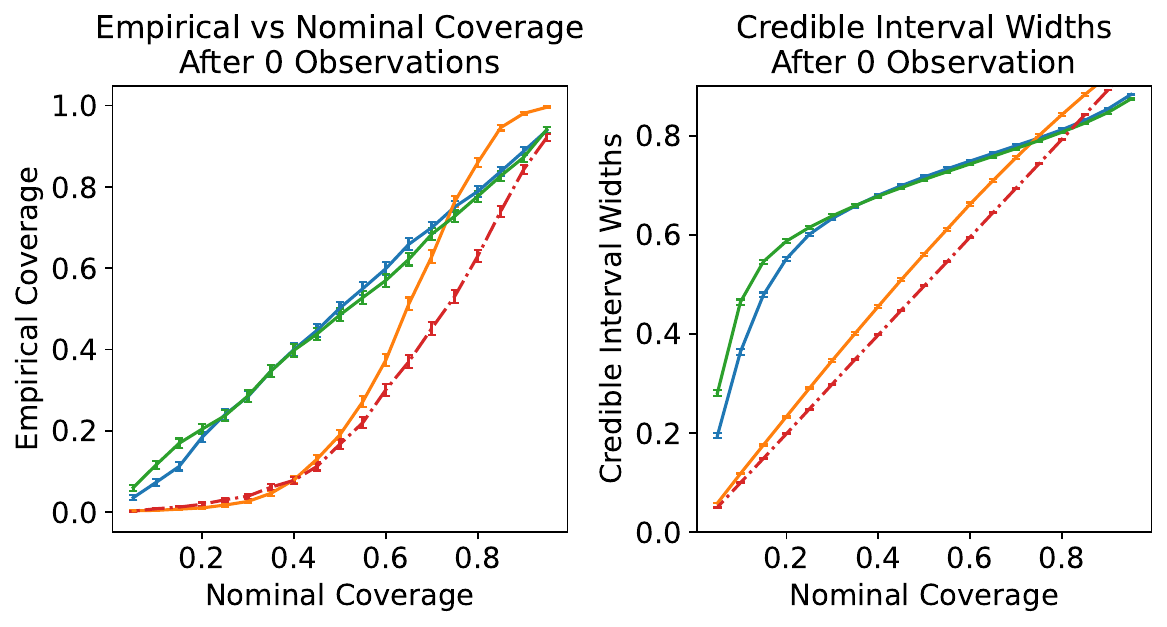}
    \begin{minipage}[b]{\linewidth}
      \centering
      \includegraphics[width=\linewidth]{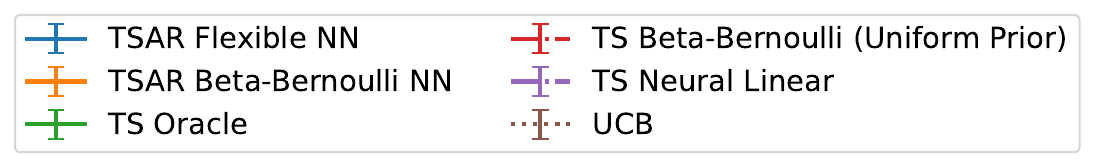}

    \end{minipage}
    
  \end{minipage}
    \caption{\bo{Mixture Beta-Bernoulli setting.}
    Left: cumulative regret averaged over $500$ repetitions ($|\Aeval| = 10$). 
    Right: evaluating uncertainty quantification (coverage and interval width) averaged over $1000$ actions not seen in training.
    Error bars are $\pm 1$ s.e. \vspace{-3mm}
    }
    \label{fig:synthetic_evals}
\end{figure}

\subsubsection{Simulation Results}

\paragraph{TSAR with a well-fitted sequence model has regret close to the TS Oracle: Figure \ref{fig:synthetic_evals} (Left).}
\textsc{TS Oracle} achieves the lowest regret, with \textsc{TSAR Flexible NN} performing nearly as well. Methods using misspecified unimodal priors---\textsc{TSAR Beta-Bernoulli NN} and \textsc{TS Beta-Bernoulli (Uniform Prior)}---achieve intermediate performance, significantly better than both \textsc{UCB} and Thompson sampling with a misspecified Bayesian model. %

\paragraph{Our posterior sampling via autogressive generation (PSAR, Algorithm~\ref{alg:posterior_sample}) procedure with a well-fitted sequence model accurately quantifies epistemic uncertainty: Figure \ref{fig:synthetic_evals} (Right).} 
For each of 1000 unseen articles, we generate 250 posterior samples $\mu^{(a)}(\widehat{\tau}_1)$ of the mean action reward $\mu^{(a)}(\tau)$ by autoregressively sampling outcomes conditioned $Z^{(a)}$. Using quantiles of these samples, %
 we construct credible intervals and measure coverage, i.e., how often the true $\mu^{(a)}(\tau)$ falls within these intervals. The \textsc{Flexible NN} nearly matches  \textsc{TS Oracle}'s theoretically ideal performance in both coverage and interval width. Methods using misspecified Beta priors show substantially worse coverage, even with fitted parameters. While these results reflect inference conditioned only on $Z^{(a)}$ without observations of $Y$, similar patterns hold when conditioning on %
 additional observations
 (see Appendix~\ref{app:synthetic_appendix}).

\subsection{News recommendation setting}
\label{sec:MIND}

Next, we build a semi-realistic news recommendation task using the MIcrosoft News Dataset (MIND) \citep{wu2020mind}. Our Thompson sampling via autoregressive generation algorithm %
scales effectively to scenarios where performance depends on unstructured information like article headlines. The method maintains accurate uncertainty quantification and low-regret exploration even when incorporating and finetuning a language model.

For each article $a$, $Z^{(a)}$ represents either headline text (in our best methods) or categorical information like ``politics'' or ``sports'' (in simpler variants). The dataset contains approximately 11,000 articles after preprocessing, with binary click/no-click rewards. 

\subsubsection{Decision-making algorithms}
\paragraph{TSAR variants.} 
We consider TSAR with two different types of sequence models: %
(i) \textsc{Flexible NN (Text)} and (ii) \textsc{Beta-Bernoulli NN (Text)} are analogous to those from Section \ref{sec:synthetic_experiments}, but we modify them to use article text $Z^{(a)}$ embedded using DistilBERT \citep{sanh2019distilbert}, which is fine-tuned end-to-end during the offline stage. 
(iii) \textsc{Flexible NN (Category)} treats $Z^{(a)}$ as a one-hot encoding of category information (e.g. ``politics" or ``sports") instead of headline text. As in Section~\ref{sec:synthetic_experiments}, we generate $m=500$ missing outcomes per action and use their average to estimate $\mu^{(a)}(\tau)$, instead of generating all $T$ missing outcomes to reduce computation costs.

\paragraph{Baselines.} 
To our knowledge, this kind of informed bandit problem is new and lacks established baselines. %
We therefore adapt several popular approaches from related areas. %
\textsc{TS Neural Linear} implements Thompson sampling with a Gaussian linear regression, where the regression features consist of the output from a neural network trained to predict outcomes given DistilBERT article embeddings as input, adapting ideas from neural contextual bandits  \citep{riquelme2018deep}. %
We also implement %
an ensemble of 50 neural network models following \cite{lakshminarayanan2017simple}  and  \cite{osband2016deep} as baselines for uncertainty quantification at the $t=0$ timestep. For comparison to methods that ignore prior information $Z^{(a)}$, we include both the UCB algorithm in Section 6 of \citet{Abbasi-YadkoriPaSz11} and \textsc{TS Beta-Bernoulli (Uniform Prior)}, which performs exact Thompson sampling with a misspecified uniform prior.

\subsubsection{Simulation Results}
\label{sec:newsSimulationResults}
\begin{figure}[t]
    \vspace{-4mm}
    \centering
  \begin{minipage}[b]{0.45\textwidth}
    \begin{minipage}[b]{\linewidth}
      \centering
      \includegraphics[trim=5pt 0pt 0pt 0pt, clip,width=\linewidth]{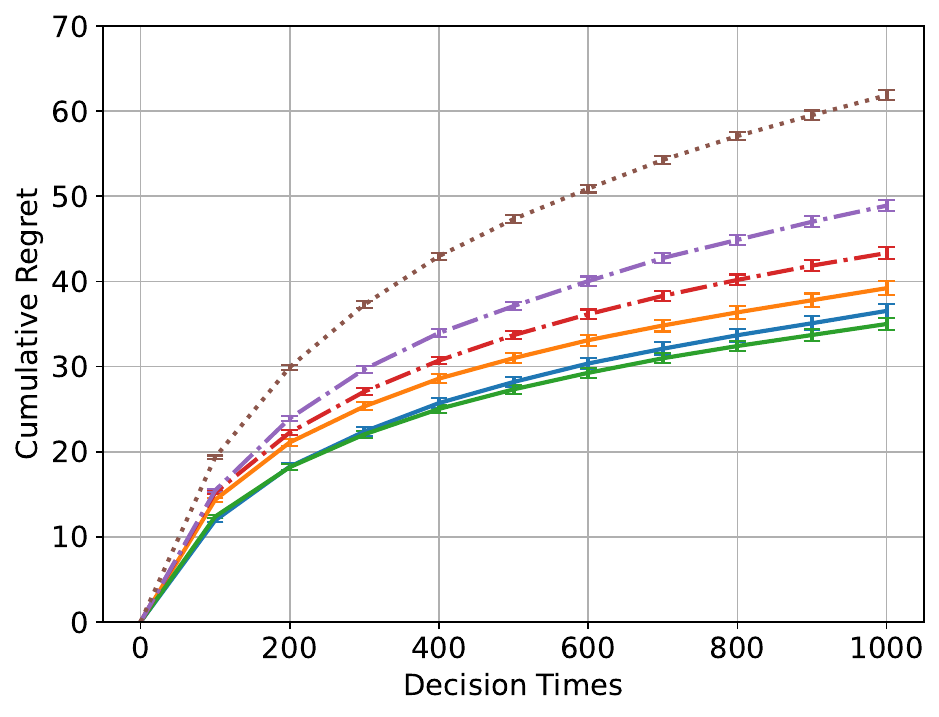}
    \end{minipage}

  \end{minipage}
      \begin{minipage}[b]{0.53\textwidth}
    \centering
    \includegraphics[width=\linewidth]{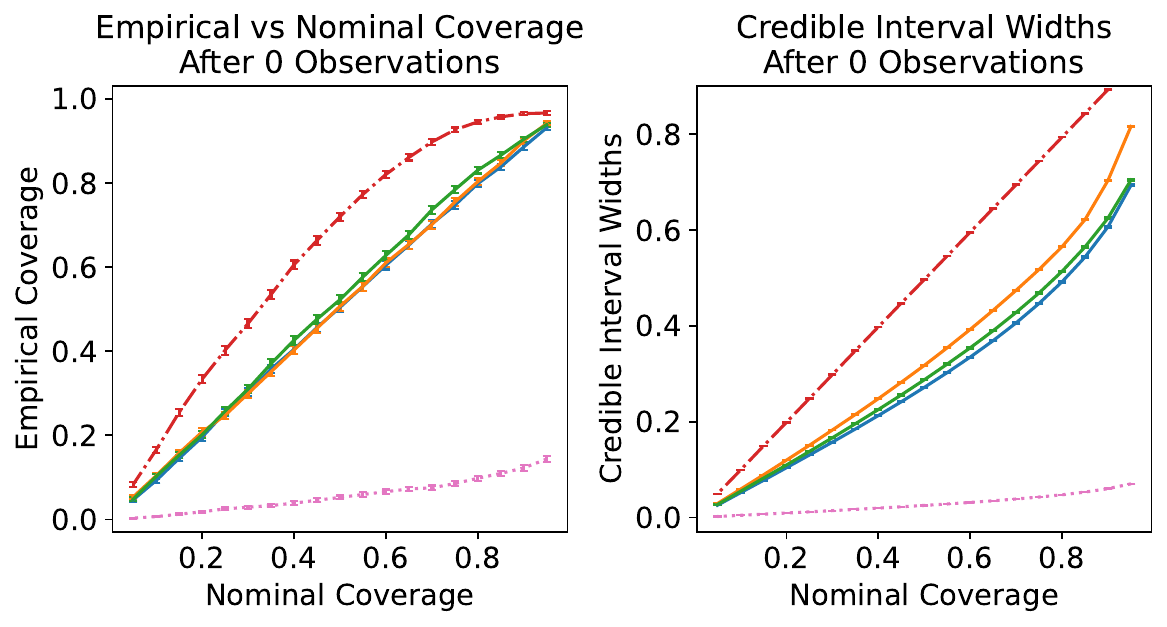}
    \begin{minipage}[b]{\linewidth}
      \centering
      \includegraphics[width=\linewidth]{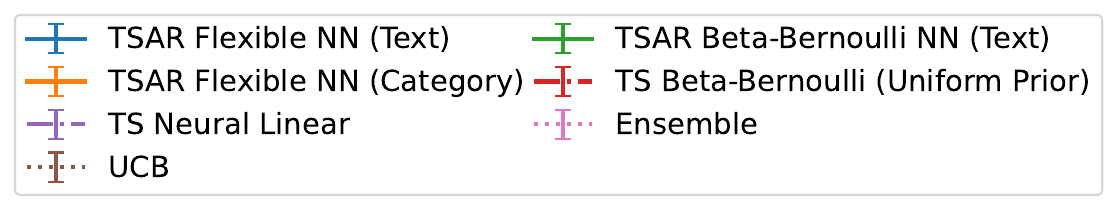}
      
    \end{minipage}
    
  \end{minipage}
    \caption{\bo{Evaluation 
    on news data.} 
    Left: cumulative regret with $|\Aeval| = 10$, averaged over $500$ repetitions.
    Right: evaluating uncertainty quantification (coverage and interval width), averaged over 2280 actions not seen in training.
    Error bars are $\pm 1$ s.e. \vspace{-3mm}
    }
    \label{fig:mind_evals}
    \vspace{-3pt}
\end{figure}

\paragraph{%
PSAR provides accurate uncertainty quantification with language models %
in the loop: Figure \ref{fig:mind_evals} (right).} 
A longstanding challenge in ML is implementing uncertainty quantification with neural networks, especially with large foundation models in the loop \citep{riquelme2018deep,gawlikowski2023survey}. Using the same methodology as Section \ref{sec:synthetic_experiments}, we evaluate uncertainty quantification by constructing intervals for the mean action reward %
across 2,280 articles held out from the offline training data. Notably, all %
\textsc{PSAR}
variants achieve the nominal coverage level out of the box \textit{without any post-hoc adjustments to improve calibration}. Also note that the text-based models produce slightly narrower intervals compared to ones that do not use text information. 

Many of the baseline approaches have extremely poor coverage, including deep ensembling, which is often considered a state-of-the-art approach for uncertainty quantification with neural networks. We believe ensembling performs poorly in our informed exploration setting where, loosely speaking, the relevant uncertainty is the mean action reward for \emph{this} specific action, rather than the mean action reward for actions with this action's prior information $Z^{(a)}$. Ensembling focuses on uncertainty for the latter quantity. Additional details and analysis are provided in Appendix~\ref{app:MIND_appendix}.

\paragraph{TSAR using %
headline text has the best regret, but gains are modest: Figure \ref{fig:mind_evals} (left).} 
\textsc{TSAR} methods that incorporate headline text achieve the lowest regret among all approaches. 
However, comparing our best \textsc{TSAR} variant against \textsc{TS Beta-Bernoulli (Uniform Prior)} demonstrates modest gains from leveraging offline data and article text. While these results confirm such gains are achievable, this dataset likely understates the potential value of unstructured prior information; in some applications, recent work suggests such information alone can drive effective decision-making \citep{yuan2023go}. 
\textsc{TS Neural Linear} perform poorly in terms of regret, showing that simply adding neural-network-based predictions to standard Thompson sampling, for instance through a Gaussian linear model on top of trained neural network outputs, is not enough to ensure good bandit performance.

\subsection{Examining the impact of truncating generation length}
\label{sec:truncationExp}

\bo{Empirically comparing full imputation vs. truncated generation (Figure \ref{fig:full_vs_truncated_imputation}).} 
We empirically compare our proposed method (Algorithm \ref{alg:Thompson_sampling_generic}), i.e., ``full imputation'', with the computationally cheaper version of Thompson sampling with autoregressive generation that limits %
generation to $m=500$ future timesteps, on the news setting (Section~\ref{sec:MIND}). We follow the procedure described in Appendix \ref{app:MIND_appendix} in forming the regret plots: we run $500$ repetitions of each bandit algorithm and in each repetition we draw a new set of $10$ actions/articles from the validation set to represent a ``new task''. We find that the two versions have similar performance. %

\begin{figure}[t]
\centering
\begin{subfigure}[b]{0.45\textwidth}
  \centering
  \includegraphics[width=1\linewidth]{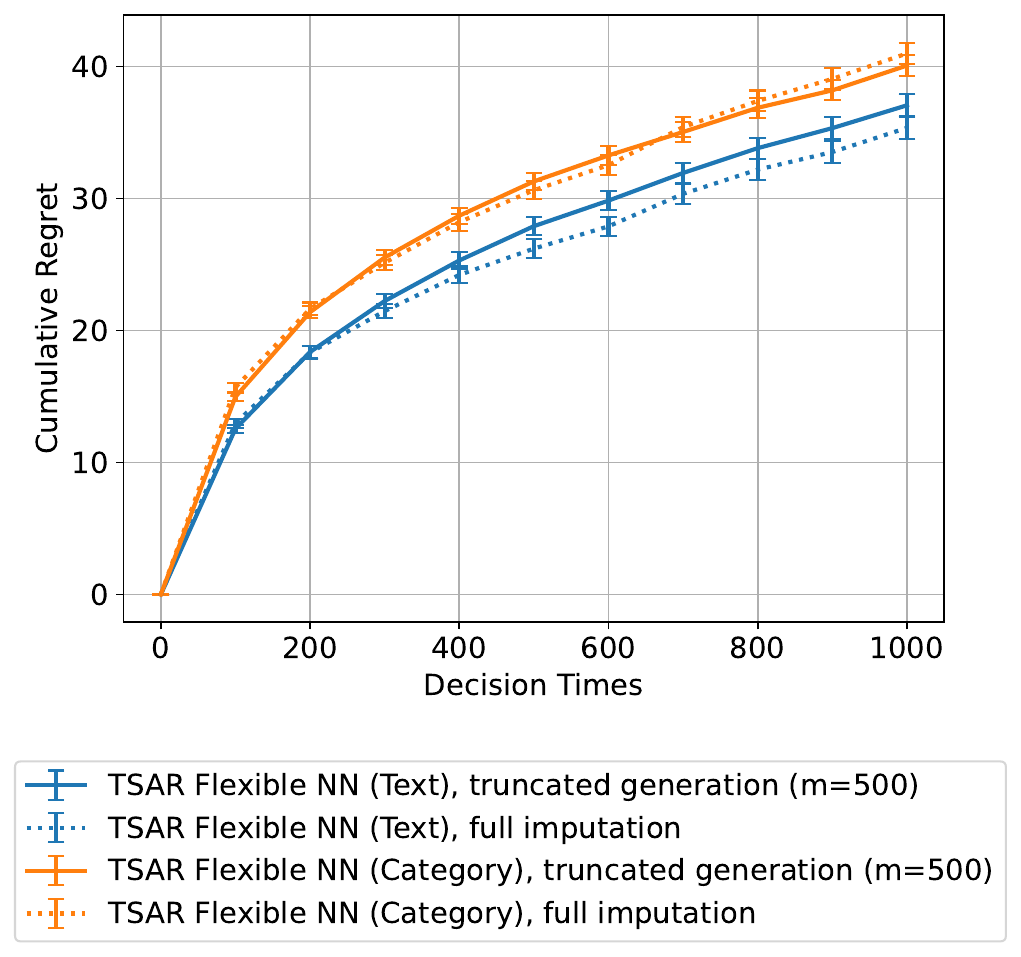}
  
  \caption{\bo{Full imputation vs. truncated generation of future rewards.} Error bars are $\pm 1$ s.e. averaged over $500$ runs.}
  \label{fig:full_vs_truncated_imputation}
\end{subfigure}
\qquad
\begin{subfigure}[b]{0.45\textwidth}
  \centering
  \includegraphics[width=1\textwidth]{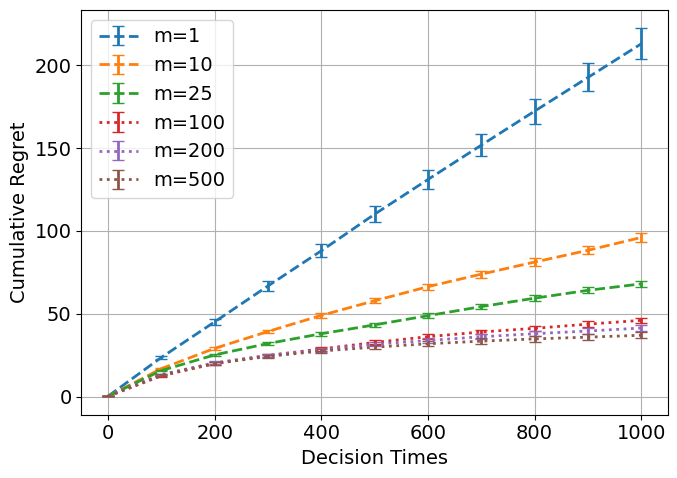}
  \vspace{3mm}
  \caption{\bo{Varying truncated generation length ($|\Aeval|=10$).} Error bars are $\pm 1$ s.e. averaged over $500$ runs.}
  \label{fig:mind_m}
\end{subfigure}
\caption{Generation length truncation experiments.}
\end{figure}

\paragraph{Varying the truncation generation length (Figure \ref{fig:mind_m}).} 
Throughout all our previous experiments (aside from Figure \ref{fig:full_vs_truncated_imputation}), we set $m=500$. In Figure \ref{fig:mind_m}, we examine the impact of varying $m$ on the regret of the algorithm with the \textsc{Flexible NN (text)} sequence model in the news recommendation setting. As before, we follow the procedure described in Appendix \ref{app:MIND_appendix} in forming the regret plots. %
We find that increasing $m$ reduces the regret of the algorithm; however, when $m$ is sufficiently large, the benefit of increasing $m$ further is negligible.

\subsection{Important subtleties in autoregressive generation}\label{subsec:failure_of_other_generation}

Our implementation of autoregressive generation in Section~\ref{sec:pretrainHistorical} makes three critical choices:
\begin{enumerate}[leftmargin=0.2in]
\item Potential outcomes are computed autoregressively, meaning we condition each newly generated outcome on those previously generated. This captures persistent patterns in performance, where actions that have demonstrated strong results are more likely to continue performing well.
\item Autoregressive generation is performed through random sampling rather than selecting the most likely next outcome. This preserves the model's implicit understanding of uncertainty, whereas always choosing the most likely outcome does not. %
\item At each decision-time, instead of sampling a single reward per action, we consider an aggregate statistic: the average of the imputed rewards over time. This approach better captures uncertainty about an action's overall performance, which is most relevant for exploration.
\end{enumerate}

\begin{figure}[t]
    \vspace{-2mm}
    \centering
    \begin{minipage}[b]{0.4\linewidth}
      \centering
      \includegraphics[width=\linewidth]{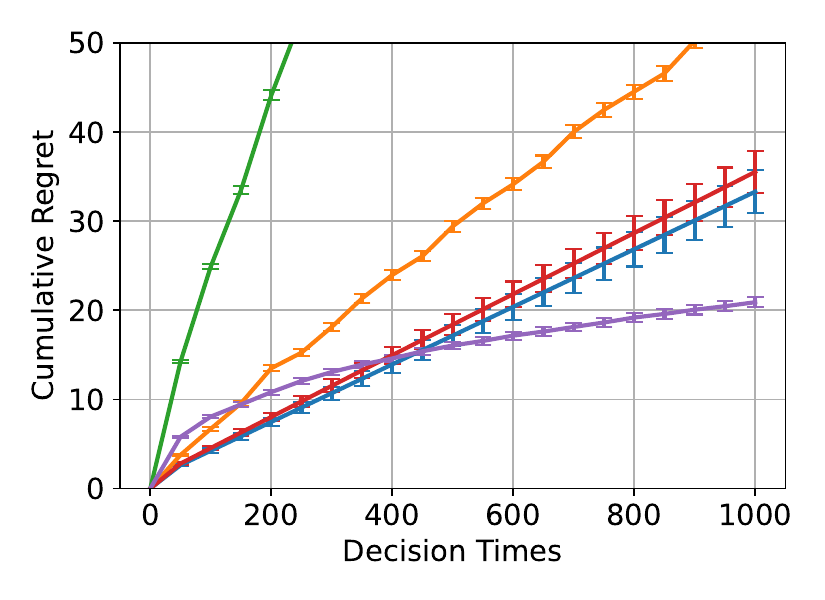}
    \end{minipage}
    \begin{minipage}[b]{0.35\linewidth}
      \centering
        \includegraphics[width=\linewidth]{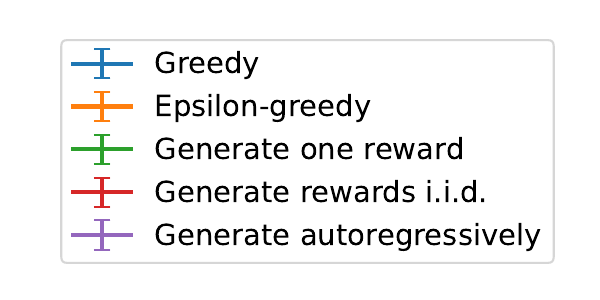}
        \vspace{5mm}
    \end{minipage}
    \vspace{-2mm}
    \caption{\bo{Comparing to alternative sampling approaches.} All approaches use the same pretrained sequence model for decision-making. \textsc{Greedy} selects the action with the greatest predicted mean reward; \textsc{Epsilon-greedy} selects the action with the greatest predicted mean reward 90\% of the time and a random action 10\% of the time;
    \textsc{Generate one reward} attempts to mimic Thompson sampling by selecting the action with the highest sampled one next reward; \textsc{Generate rewards i.i.d.} instead generates many i.i.d. samples from each reward and selects the action with the highest predicted mean. Our approach, \textsc{Generate rewards autoregressively}, uses the sequence model to generate rewards autoregressively and selects the action with the highest empirical mean.
    }
    \label{fig:altSampling}
\end{figure}

These implementation choices are critical to the algorithm's correctness. Figure \ref{fig:altSampling} demonstrates that two alternative approaches, including one recently proposed in several published papers (to be discussed shortly), lead to poor decision-making by either over- or under-exploring. Each tested variant modifies the steps in our proposed algorithm where  
we average autoregressively sampled and observed rewards to form a ``posterior sample" of the mean reward, while keeping the remainder of the algorithm unchanged. See Appendix~\ref{sec:altsampling_details} for additional details on the experimental setup. Our analysis reveals two key findings:    
\begin{enumerate}[leftmargin=*]
\item \textbf{The importance of averaging out idiosyncratic noise.}  The method labeled \textsc{Generate one reward} implements a proposal from recent works~\citep{NguyenGr22, MullerHoArGrHu22, GarneloRoMaRaSaShTeReEs18}. These works propose generating only the outcome/reward from \emph{the next} action draw, and taking the action with the highest sample reward. While this captures overall \textit{predictive uncertainty} in the next outcome---including both uncertainty in an action's overall quality (epistemic uncertainty) and uncertainty in the idiosyncratic noise of a single outcome (aleatoric uncertainty)---it fails in settings with substantial idiosyncratic uncertainty.\footnote{In domains like recommender systems, there is substantial ``noise" in individual user interactions. The aforementioned papers apply their method to benchmark problems with almost no reward noise.} Figure \ref{fig:altSampling} shows this approach incurs linear regret due to endless randomized exploration.
\item \textbf{The importance of autoregressive generation.}   The method labeled \textsc{Generate rewards i.i.d.} averages rewards across many independent (non-autoregressive) draws of the next outcome and takes the action with the highest average sampled reward. With sufficient generations, this average converges to the mean of the predictive distribution, effectively selecting the action currently believed to be best without purposeful exploration. Such ``greedy" algorithms perform poorly in the long run because they fail to gather the data needed to resolve uncertainty about action quality.
\end{enumerate}

\section{Discussion}

By viewing uncertainty through the lens of missing data rather than latent parameters, our work establishes a principled bridge between modern sequence models and exploration algorithms, demonstrating through algorithms, theory, and experiments that success in offline sequence prediction translates to effective online decision-making and uncertainty quantification. We believe our missing data view of uncertainty extends well beyond the setting studied in this work, including to sequences with predictable time-dependence. 
We are optimistic that operationalizing this viewpoint beyond the setting studied in this paper will bear fruit. 

Several important challenges remain open. While enormous investments in both software and hardware are going toward speeding up autoregressive generation (e.g. for language model inference), we still expect that repeatedly generating %
posterior samples online can be %
computationally expensive. Future work could address this computational overhead by distilling exploration policies into neural networks through imitation learning~\citep{NamkoongDaBa20}. Second, our experiments assume access to sufficient historical data for offline learning. In practice, available data may be limited and partially confounded by previous action selection policies. While many deployed systems face and overcome similar challenges, carefully evaluating our methods' robustness to these issues at scale remains important. %

\appendix
\addtocontents{toc}{\protect\setcounter{tocdepth}{0}}

\appendix
\paragraph{Overview of Appendices} 
\begin{itemize}[leftmargin=0.5cm]
    \item Appendix \ref{app:proofs}: Theoretical results
    \item Appendix \ref{sec:exp_details}: Experimental details
    \item Appendix \ref{subsec:bayes-seq-models}: Bayesian sequence models
    \item Appendix \ref{app:finite_pop_TS}: Finite vs infinite population formulations and Thompson Sampling variants
\end{itemize}

\section{Theoretical results}
\label{app:proofs}

\subsection{ Proof of Proposition~\ref{prop:regret_TS_exact}: Regret of exact Thompson Sampling by sequential imputation}
\label{sec:app_exact_ts}
\begin{proof}
Proposition~\ref{prop:regret_TS_exact} follows from \citet{bubeck2015bandit} and \citet{russo2016information}. 
For convenience, write $\mathbb{E}_t[\,\cdot\,] \triangleq \mathbb{E}[\,\cdot \mid \mathcal H_{t-1}]$ for the expectation given history $\mathcal H_{t-1}$ and $I_t\!\left(A^*;\,(A_t, Y_t)\right) \;\triangleq\; I\!\left(A^*;\,(A_t, Y_t) \mid \mathcal H_{t-1}\right)$
for the mutual information between the optimal action $A^* = A^*(\tau)$ and the observed action-outcome
pair $(A_t, Y_t)$, conditional on the history $\mathcal H_{t-1}$. Finally, we define the information ratio at timestep $t$ as $\Gamma_t \;\triangleq\;
    \frac{\left(\mathbb{E}_t \big[R \big(Y_t^{(A^*)} \big) - R(Y_t)\big]\right)^2}
    {I_t\!\left(A^*;\,(A_t,\,Y_t)\right)}$.
\noindent By Proposition~3 in \citet{russo2016information}, %
$\frac{\bigl(\mathbb{E}_t \big[R\big(Y_t^{(A^*)}\big)
          - R(Y_t)\big] \bigr)^2}
         {I_t\!\left(A^*;\,(A_t,\,Y_t)\right)}
    \;\leq\; \frac{|\mathcal{A}|}{2}$.
Rearranging,
\begin{align}
    \label{eq:prop1_proof}
\mathbb{E}_t \left[R\big(Y_t^{(A^*)}\big) - R(Y_t)\right]
    \;\leq\;
    \sqrt{\frac{|\mathcal{A}|\;I_t\!\left(A^*;\,(A_t,\,Y_t)\right)}{2}}.
\end{align}
\begin{lemma}[\citet{bubeck2015bandit} Lemma~13]
\label{lem:bubeck_info_regret_bound}
For any algorithm in any Bayesian bandit problem with a finite action set and arbitrary prior over reward distributions over time, $\mathbb{E}\!\left[\,\sum_{t=1}^{T}
    \sqrt{I_t\!\left(A^*;\,(A_t,\,Y_t)\right)}\,\right]
    \;\leq\; \sqrt{T \log |\mathcal{A}|}$.
\end{lemma}

By summing \eqref{eq:prop1_proof} over timesteps $t$ and taking expectations to obtain regret $\Delta(\pi_{TSAR}, \MC{B})$, and then applying Lemma~\ref{lem:bubeck_info_regret_bound} at each timestep to obtain the desired regret bound
\begin{align*}
    \Delta(\pi_{TS}, \MC{B}) =
    \frac{1}{T}\sum_{t=1}^T \mathbb{E}\!\left[\mathbb{E}_t\!\left[
        R\!\left(Y_t^{(A^*)}\right) - R\!\left(Y_t^{(A_t)}\right)
    \right]\right]
    &\leq \frac{\sqrt{|\mathcal{A}|/2}}{T}\;
    \mathbb{E}\!\left[\sum_{t=1}^T \sqrt{I_t\!\left(A^*;\,(A_t, Y_t^{(A_t)})\right)}\right]\\
    &\leq \frac{\sqrt{|\mathcal{A}|/2}}{T} \cdot \sqrt{T \log |\mathcal{A}|}
    = \sqrt{\frac{|\mathcal{A}| \log |\mathcal{A}|}{2T}}.
\end{align*}
\end{proof}

\vspace{-5mm}
\subsection{Deriving Eqn \eqref{eq:log-loss-is-kl}: KL divergence representation of the offline loss}
\label{app:log-loss-is-kl}
The next lemma is a standard result connecting the excess expected loss of a sequence model $ n$ to its KL divergence from the true sequence model $p^*$. In this section, for notational convenience, we use $p^*$ and $\widehat{p}$ to denote $p_{\TN{per-action}}^*$ and $\widehat{p}_{\TN{per-action}}$, respectively. Recall, the expected loss of a sequence model $\widehat{p}$ is denoted $\ell(\widehat{p})$, defined in \eqref{eq:pop_loss}. To (nearly) minimize loss, $\widehat{p}$ the learner needs to closely approximate the true sequence model $p^*$.
\begin{lemma}
    \label{lem:log-loss-is-kl}
    For any sequence model $\widehat{p}$,
    \[
    \ell(\widehat{p}) = \ell(p^*) + \E_{Z^{(a)} \sim P_Z} \left[D_{\rm KL}\left( p^* \big(Y_1^{(a)}, \ldots, Y_T^{(a)} \mid Z^{(a)} \big) \; \Big\| \; \widehat{p} \big( Y_1^{(a)}, \ldots, Y_T^{(a)} \mid Z^{(a)} \big)  \right) \right]. 
    \]    
\end{lemma}
\begin{proof}
By the definition of the expected loss in  \eqref{eq:pop_loss}, and the chain rule of KL divergence:
    \begin{align*}
        \ell(\widehat{p}) - \ell(p^*) 
        &= \E\left[ -\sum_{t=1}^T \log \widehat{p} \big( Y_t^{(a)} \mid Z^{(a)}, Y_{1:t-1}^{(a)} \big) \right] - \E\left[ -\sum_{t=1}^T \log p^* \big( Y_t^{(a)} \mid Z^{(a)}, Y_{1:t-1}^{(a)} \big) \right] \\
        &=  \E\left[ \E \left[  \log \left\{ \prod_{t=1}^{T} \frac{p^* \big( Y_t^{(a)} \mid Z^{(a)}, Y_{1:t-1}^{(a)} \big)}{\widehat{p} \big( Y_t^{(a)} \mid Z^{(a)}, Y_{1:t-1}^{(a)} \big)} \right\} \, \Bigg| \, Z^{(a)} \right]  \right] 
    \end{align*}
    \begin{align*}
        &=  \E\left[ \E \left[  \log \left\{   \frac{p^* \big( Y_{1}^{(a)}, \dots, Y_{T}^{(a)} \mid Z^{(a)} \big)}{\widehat{p} \big( Y_{1}^{(a)}, \dots, Y_{T}^{(a)} \mid Z^{(a)} \big)} \right\} \, \Bigg| \, Z^{(a)} \right]  \right]\\
        &= \E_{Z^{(a)} \sim P_Z} \left[D_{\TN{KL}} \big( p^*( Y_1^{(a)}, \dots, Y_T^{(a)} \mid Z^{(a)} ) ~ \big\| ~ \widehat{p} ( Y_1^{(a)}, \dots, Y_T^{(a)} \mid Z^{(a)} ) \big)\right].
    \end{align*}
    The second equality uses the chain rule and properties of logarithms. The third equality uses the chain rule of probability. The last equality uses the definition of KL divergence.  
\end{proof}

\vspace{-2mm}
\subsection{Proof of Proposition \ref{prop:posterior_sample}}\label{app:proof_of_posterior_sample}
    Consider the full process of generating a table $\widehat{\tau}_t$. This involves first i) the sequential revelation of $t-1$ (real) potential outcomes determined by selected actions under $\pi$ and ii) the sequential revelation of $T\times |\Aeval|-(t-1)$ (imagined) potential outcomes produced by the autoregressive generation procedure in Algorithm \ref{alg:posterior_sample} (see also Figure \ref{fig:autoregressive_generation}.). 

    We introduce an abstraction which subsumes both kinds of revelation. The revelation of a full potential outcome table proceeds across $|\Aeval|\times T$ rounds; we use the letter $i$ to index these rounds, $a$'s to index rows of the potential outcomes table, and $j$'s to index columns. Let $(A_i, J_i) \in \Aeval \times [T]$ denote the index of the entry selected in round $i$ and $\widehat{\tau}[A_i, J_i]$ to denote an entry itself in sampled potential outcomes table. For $i\leq t-1$, the column is simply $J_i =i$ and $A_i$ is determined adaptively according to the policy $\pi$; for $i\geq t$, the choice of $(A_i, J_i)$ proceeds across missing entries according to the deterministic order of Alg. \ref{alg:posterior_sample}.

    We can represent this procedure abstractly as an adaptive selection rule $\psi = (\psi_1, \ldots, \psi_{|\Aeval|T})$. For any $i$, $\psi_i$ is a function which selects $A_i,J_i$ as a function of the revealed/generated outcomes so far and possibly some exogenous noise. Writing it formally, let the (generalized) history $\HH^+_{i-1} \in  \mathcal{Z}^{|\Aeval|} \times \left( \Aeval \times [T] \times \mathcal{Y}\right)^{i-1}$ denote the sampled unstructured inputs $(Z's)$ together with a  sequence of table indices and entries; take $\xi_i$ to be a random seed, independent of all else. The procedure described above can be thought of as corresponding to a rule which determines selections $(A_i, J_i)=\psi_i(\HH_{i-1}, \xi_i)$ for each $i \leq |\Aeval|T$. While the $\psi$ we are studying has many special properties, our analysis does not rely on these. Instead, our proof leverages only one property of $\psi$: that it never selects a previously revealed entry of the table and hence reveals all entries of the table after $|\Aeval|T$ revelations. We call this such a selection rule \emph{non-repeating}.
    
    The proof compares three procedures for sampling outcomes of selected entries conditioned on prior selections and samples.
       \begin{enumerate}[leftmargin=*]
        \item {\bf True environment}: Conditioned on $\HH^{+}_{i-1}$ and $(A_i, J_i)$, the revealed outcome $\widehat{\tau}[A_i, J_i]$ is drawn according to $p^*$. That is, $\widehat{\tau}[a,j] \mid \HH^{+}_{i-1}, A_i=a, J_i=j \sim p^*\left( \cdotspace \mid Z^{(a)}, (\widehat{\tau}[A_k, J_k])_{k<i, A_k=a} \right)$. We denote probabilities under $\psi$ employed in this environment by $\mathbb{P}_{*,\psi}$. 
        \item {\bf Misspecified environment}: Conditioned on $\HH^{+}_{i-1}$ and $(A_i, J_i)$, the revealed outcome $\widehat{\tau}[A_i, J_i]$ is drawn according to $\widehat{p}$. That is, $\widehat{\tau}[a,j] \mid \HH^{+}_{i-1}, A_i=a, J_i=j \sim \widehat{p}\left( \cdotspace \mid Z^{(a)}, (\widehat{\tau}[A_k, J_k])_{k<i, A_k=a} \right)$. We denote probabilities under $\psi$ employed in this environment by $\mathbb{P}_{\theta,\psi}$. 
          \item  {\bf Hybrid environment}: Conditioned on $\HH^{+}_{i-1}$ and $(A_i, J_i)$, the revealed outcome $\widehat{\tau}[A_i, J_i]$ is drawn according to $p^*$ if $i<t$ and drawn according to $\widehat{p}$ otherwise. We denote probabilities under $\psi$ employed in this environment by $\mathbb{P}_{h,\psi}$.  This represents the procedure used to generate $\widehat{\tau}_{t}$, where the first $t-1$ entries are revealed according the true environment, and the remainder are drawn by using Algorithm \ref{alg:posterior_sample} with $\widehat{p}$ for generation.
    \end{enumerate}
    With the abstractions and notations above, Proposition \ref{prop:posterior_sample} can be rewritten as
    \begin{equation}\label{eq:posterior_sample_bound_rewritten}
     \mathbb{E}_{\psi, *} \left[ D_{\rm KL}\left( \PP_{\psi, *}\left( \widehat{\tau} \in \cdot \mid \HH^+_{t-1} \right) \, \big\|  \, \PP_{\psi, h}\left( \widehat{\tau} \in \cdot  \mid \HH^+_{t-1} \right)   \right) \right] \leq |\Aeval| \left( \ell(\widehat{p}) - \ell(p^*) \right).
    \end{equation}
    We establish this result in three steps. The first bounds the KL divergence on the left hand side of \eqref{eq:posterior_sample_bound_rewritten} by the KL divergence between full histories. 
    \begin{lemma} For any \emph{non-repeating} selection rule $\psi$, 
        \[
             \mathbb{E}_{\psi, *} \left[ D_{\rm KL}\left( \PP_{\psi, *}\left( \widehat{\tau} \in \cdot \mid \HH^+_{t-1} \right) \, \big\|  \, \PP_{\psi, h}\left( \widehat{\tau} \in \cdot  \mid \HH^+_{t-1} \right)   \right) \right] \leq D_{\rm KL}\left( \PP_{\psi, *}\left( \HH_{|\Aeval|T}^+ \in \cdot   \right) \, \Big\|  \, \PP_{\psi, h}\left( \HH_{|\Aeval|T}^+ \in \cdot  \right)   \right).    
        \]
    \end{lemma}
    \begin{proof}
        Since the complete table $\widehat{\tau}$ is a function of the final history $\HH_{|\Aeval|T}^+$ under any non-repeating $\psi$, the data processing inequality yields
        \begin{align*}
        &\mathbb{E}_{\psi, *} \left[ D_{\rm KL}\left( \PP_{\psi, *}\left( \widehat{\tau} \in \cdot \mid \HH^+_{t-1} \right) \, \big\|  \, \PP_{\psi, h}\left( \widehat{\tau} \in \cdot  \mid \HH^+_{t-1} \right)   \right) \right] \\
        \leq & \mathbb{E}_{\psi, *} \left[ D_{\rm KL}\left( \PP_{\psi, *}\left( \HH_{|\Aeval|T}^+ \in \cdot \mid \HH^+_{t-1} \right) \, \big\|  \, \PP_{\psi, h}\left( \HH_{|\Aeval|T}^+ \in \cdot  \mid \HH^+_{t-1} \right)   \right) \right] \\
        =&   D_{\rm KL}\left( \PP_{\psi, *}\left( \HH_{|\Aeval|T}^+ \in \cdot   \right) \, \Big\|  \, \PP_{\psi, h}\left( \HH_{|\Aeval|T}^+ \in \cdot  \right)   \right)  - \underbrace{ D_{\rm KL}\left( \PP_{\psi, *}\left( \HH_{t-1}^+ \in \cdot   \right) \, \big\|  \, \PP_{\psi, h}\left( \HH_{t-1}^+ \in \cdot  \right)   \right) }_{=0} \\
        =& D_{\rm KL}\left( \PP_{\psi, *}\left( \HH_{|\Aeval|T}^+ \in \cdot   \right) \, \Big\|  \, \PP_{\psi, h}\left( \HH_{|\Aeval|T}^+ \in \cdot  \right)   \right). 
    \end{align*}
    The first equality is the chain rule of KL divergence and the second equality uses that the hybrid and true environments coincide for the first $t-1$ rounds.
    \end{proof}
    The next step, stated as a standalone lemma, bounds the KL divergence between the true and hybrid environments by the KL divergence between the true and misspecified environments. 
    \begin{lemma} For any \emph{non-repeating} selection rule $\psi$, 
    \[
        D_{\rm KL}\left( \PP_{\psi, *}\left( \HH_{|\Aeval|T}^+ \in \cdot   \right) \, \Big\|  \, \PP_{\psi, h}\left( \HH_{|\Aeval|T}^+ \in \cdot  \right)   \right) \leq        D_{\rm KL}\left( \PP_{\psi, *}\left( \HH_{|\Aeval|T}^+ \in \cdot   \right) \, \Big\|  \, \PP_{\psi, \theta}\left( \HH_{|\Aeval|T}^+ \in \cdot  \right)   \right).
    \]
    \end{lemma}
    \begin{proof}
    \vspace{-5mm}
    \begin{align*}
    & D_{\rm KL}\left( \PP_{\psi, *}\left( \HH_{|\Aeval|T}^+ \in \cdot   \right) \, \big\|  \, \PP_{\psi, h}\left( \HH_{|\Aeval|T}^+ \in \cdot  \right)   \right) \\ 
    = \quad &   \underbrace{D_{\rm KL}\left( \PP_{\psi, *}\left(  \HH^+_{t-1} \in \cdot  \right) \, \big\|  \, \PP_{\psi, h}\left( \HH^+_{t-1} \in \cdot \right) \right) }_{=0}   \\
    & \quad  + \mathbb{E}_{\psi, *}\left[ D_{\rm KL}\left( \PP_{\psi, *}\left( \HH_{|\Aeval|T}^+ \in \cdot \mid \HH^+_{t-1} \right) \, \Big\|  \, \PP_{\psi, h}\left( \HH_{|\Aeval|T}^+ \in \cdot \mid \HH^+_{t-1} \right)   \right)  \right] \\
    \leq \quad &   \underbrace{D_{\rm KL}\left( \PP_{\psi, *}\left(  \HH^+_{t-1} \in \cdot  \right) \, \big\|  \, \PP_{\psi, \theta}\left( \HH^+_{t-1} \in \cdot \right) \right) }_{\geq 0}   \\
    & \quad  + \mathbb{E}_{\psi, *}\left[ D_{\rm KL}\left( \PP_{\psi, *}\left( \HH_{|\Aeval|T}^+ \in \cdot \mid \HH^+_{t-1} \right) \, \Big\|  \, \PP_{\psi, h}\left( \HH_{|\Aeval|T}^+ \in \cdot \mid \HH^+_{t-1} \right)   \right)  \right] \\
    = \quad &   D_{\rm KL}\left( \PP_{\psi, *}\left(  \HH^+_{t-1} \in \cdot  \right) \, \big\|  \, \PP_{\psi, \theta}\left( \HH^+_{t-1} \in \cdot \right) \right)  \\
    & \quad  + \mathbb{E}_{\psi, *}\left[ D_{\rm KL}\left( \PP_{\psi, *}\left( \HH_{|\Aeval|T}^+ \in \cdot \mid \HH^+_{t-1} \right) \, \Big\|  \, \PP_{\psi, \theta}\left( \HH_{|\Aeval|T}^+ \in \cdot \mid \HH^+_{t-1} \right)   \right)  \right] \\
    =\quad & D_{\rm KL}\left( \PP_{\psi, *}\left( \HH_{|\Aeval|T}^+ \in \cdot   \right) \, \Big\|  \, \PP_{\psi, \theta}\left( \HH_{|\Aeval|T}^+ \in \cdot  \right)   \right).
    \end{align*}
    We make repeated use of the chain rule and the definitions of true, hybrid, and misspecified environment. The inequality uses that the hybrid environment is identical to the true environment for the first $t-1$ samples to conclude a second KL divergence is zero, and the non-negativity of KL divergences to claim that zero is a lower bound on the divergence under the misspecified environment.  
    \end{proof}
    Our next lemma calculates the KL divergence between (generalized) histories under the true and misspecified environments, finding it is \emph{exactly} equal to the number of actions times the excess offline prediction loss of the misspecified sequence model $\widehat{p}$.

    \begin{lemma}
        Under any non-repeating selection rule $\psi$,
        \[
    D_{\rm KL}\left( \PP_{\psi, *}\left( \HH_{|\Aeval|T}^+ \in \cdot   \right) \, \Big\|  \, \PP_{\psi, \theta}\left( \HH_{|\Aeval|T}^+ \in \cdot  \right)   \right) = |\Aeval| \left( \ell(\widehat{p}) - \ell(p^*) \right).
            \]
    \end{lemma}
    \begin{proof}  
    \begin{align*}
       & D_{\rm KL}\left( \PP_{\psi, *}\left( \HH_{|\Aeval|T}^+ \in \cdot   \right) \, \Big\|  \, \PP_{\psi, \theta}\left( \HH_{|\Aeval|T}^+ \in \cdot  \right)   \right) \\
       = &  \sum_{i=1}^{|\Aeval| T} \mathbb{E}_{\psi, *} \left[D_{\rm KL}\Big( \PP_{\psi, *}\left( (A_i, J_i, \widehat{\tau}[A_i, J_i]) \in \cdot \mid \HH_{i-1}^+ \Big) \, \Big\| \,   \mathbb{P}_{\psi, \theta}\Big( (A_i, J_i, \widehat{\tau}[A_i, J_i]) \in \cdot \mid \HH_{i-1}^+ \right) \Big) \right] \\
     = & 0 +  \sum_{i=1}^{|\Aeval| T} \mathbb{E}_{\psi, *} \left[D_{\rm KL}\left( \PP_{\psi, *}\left( \widehat{\tau}[A_i, J_i]) \in \cdot \mid \HH_{i-1}^+, A_i, J_i \right) \, \Big\| \,   \mathbb{P}_{\psi, \theta}\left(\widehat{\tau}[A_i, J_i]  \in \cdot \mid \HH_{i-1}^+, A_i, J_i  \right) \right) \right]\\
     =  & \sum_{i=1}^{|\Aeval| T} \mathbb{E}_{\psi, *} \left[D_{\rm KL}\left( p^*\left( \cdotspace \mid Z^{(A_i)}, (\widehat{\tau}[A_k, J_k])_{k<i, A_k=A_i} \right) \, \Big\| \,   \widehat{p}\left( \cdotspace \mid Z^{(A_i)}, (\widehat{\tau}[A_k, J_k])_{k<i, A_k=A_i} \right) \right)  \right] \\
     =& \sum_{a \in \Aeval}  %
     \left( \sum_{n=1}^{ T} \mathbb{E}_{\psi, *} \left[D_{\rm KL}\left( p^*\left( \cdotspace \mid Z^{(a)}, (\widehat{\tau}[a, J_{k,a}])_{k<n} \right) \, \Big\| \,   \widehat{p}\left( \cdotspace \mid Z^{(a)}, (\widehat{\tau}[a, J_{k,a}])_{k<n} \right) \right)  \right]\right) 
    \end{align*}
    Where the last line introduced the notation $J_{n,a}$ to denote whichever column is revealed/sampled $n^{\rm th}$ among those in row $a$.    The first equality is the chain rule of KL divergence. The second equality uses the chain rule again, together with the fact that the conditional distribution of the selection $(A_i, J_i)$ is determined solely by $\psi$ and not impacted by the distribution used to generate elements of $\widehat{\tau}$, i.e.   
    $\mathbb{P}_{\psi, *}\left( (A_i, J_i) \in \cdot \mid \HH_{i-1}^+ \right)  = \mathbb{P}_{\psi, \theta}\left( (A_i, J_i,) \in \cdot \mid \HH_{i-1}^+ \right)$. The last equality directly applies the definitions of the true and misspecified environments. 
    
    Notice that the final sum iterates over all rows in the table and, within a row, iterates over all columns, but in an unusual adaptive order (Assumption \ref{assump:exchangeable}).  We now replace the adaptive order with a more transparent deterministic, increasing, order. Notice that $\widehat{\tau}[a, J_{1,a}] \mid Z^{(a)}  \sim p^*(\cdotspace \mid Z^{(a)})$. Then $\widehat{\tau}[a, J_{2,a}] \mid Z^{(a)},\widehat{\tau}[a, J_{1,a}] \sim p^*(\cdotspace \mid  Z^{(a)},\widehat{\tau}[a, J_{1,a}])$ and so on, yielding  
    \[
     (\widehat{\tau}[a, J_{k,a}])_{k\leq T}   \mid Z^{(a)} \sim p^*\big(\cdotspace \mid Z^{(a)} \big)
     \]
     To simplify notation, we introduce an equivalent draw $Y_{1:T}^{(a)} \sim p^*(\cdotspace \mid Z^{(a)})$. Thus,
\begin{align*}
    & D_{\rm KL}\left( \PP_{\psi, *}\left( \HH_{|\Aeval|T}^+ \in \cdot   \right) \, \Big\|  \, \PP_{\psi, \theta}\left( \HH_{|\Aeval|T}^+ \in \cdot  \right)   \right) \\
    &= \sum_{a\in \Aeval} \sum_{i=1}^{T} \mathbb{E}_{ (Z^{(a)}, Y_{1:T}^{(a)})\sim p^*} \left[D_{\rm KL}\left( p^*\left( \cdotspace \mid Z^{(a)}, (Y_{k}^{(a)})_{k<i} \right) \, \Big\| \,    \widehat{p}\left( \cdotspace \mid Z^{(a)}, (Y_{k}^{(a)})_{k<i} \right) \right)  \right] 
    \end{align*}
    \begin{align*}
     &=  \sum_{a\in \Aeval}  \E_{Z^{(a)} \sim P_Z} \left[D_{\rm KL}\left( p^* \big(Y_1^{(a)}, \ldots, Y_T^{(a)} \mid Z^{(a)} \big) \; \Big\| \; \widehat{p} \big( Y_1^{(a)}, \ldots, Y_T^{(a)} \mid Z^{(a)} \big)  \right) \right]
     = |\Aeval| \left( \ell(\widehat{p}) - \ell(p^*) \right).
    \end{align*}
    The last equality above uses \eqref{eq:log-loss-is-kl}. Recall $P_Z$ is the marginal distribution of $Z^{(a)}$ under $p^*$. 
\end{proof}

\vspace{-2mm}
\subsection{Proof of Proposition \ref{prop:sim-to-real-gap} (Section \ref{section:regret-bounds}): Bounding the deployment regret in terms of regret on a simulator}
\label{app:proof_of_sim_to_real_gap}

\subsubsection{Implementing the simulator with apriori randomness}\label{subsec:non-exchangeable-sequence-model}
For notational convenience, we use $p^*$ and $\widehat{p}$ to denote $p_{\TN{per-action}}^*$ and $\widehat{p}_{\TN{per-action}}$, respectively.
The simulator corresponding to sequence model $\widehat{p}$, formalized in Definition \ref{def:simulator}, appears to involve sequential draws of outcomes as actions are selected. An alternative, which is useful in the analysis, is to draw all randomness upfront. Doing so in a manner that corresponds to the simulator involves a slightly different potential outcomes table, denoted $\tilde{\tau}$, in which columns are interpreted as indexing arm play-counts rather than time. %

We emphasize that this construction is  {\bf used solely as proof technique, and is not part of the model of the problem.} In particular, it is used for studying simulators corresponding to misspecified non-exchangeable sequence models $\widehat{p}$. Under exchangeable sequence models, like $p^*$, order has no bearing.
\begin{definition}[Potential outcomes table indexed by play count]
\label{def:non-exchangeable-dgp}  Given a (possibly non-exhangeable) sequence model $\widehat{p}$, consider a table 
\[
\tilde{\tau}=\big(\tilde{Y}^{(a)}_1, \ldots, \tilde{Y}_{T}^{(a)}\big)_{a\in \Aeval} 
\]
where, conditioned on $Z^{(a)} \sim P_Z$, each $\big(\tilde{Y}^{(a)}_1, \ldots, \tilde{Y}_{T}^{(a)}\big) \mid Z^{(a)} \sim \widehat{p}$ is drawn independently across $a\in \Aeval$. If $A_t=a$ is selected at time $t$ and this is the $k^{\rm th}$ time that arm is chosen, then $\tilde{Y}_{t} \triangleq \tilde{Y}_k^{(A_t)}$ is revealed.
Let $\HHtilde_{t-1} \triangleq ( \{Z^{(a)}\}_{a \in \Aeval}, A_1, \tilde{Y}_1, \dots, A_{t-1}, \tilde{Y}_{t-1} )$. See Figure \ref{fig:revelation}. %
\end{definition}
\begin{figure}[h]
\centering
\vspace{-5mm}
\begin{subfigure}[b]{0.48\textwidth}
    \centering
     \includegraphics[width=0.65\textwidth]{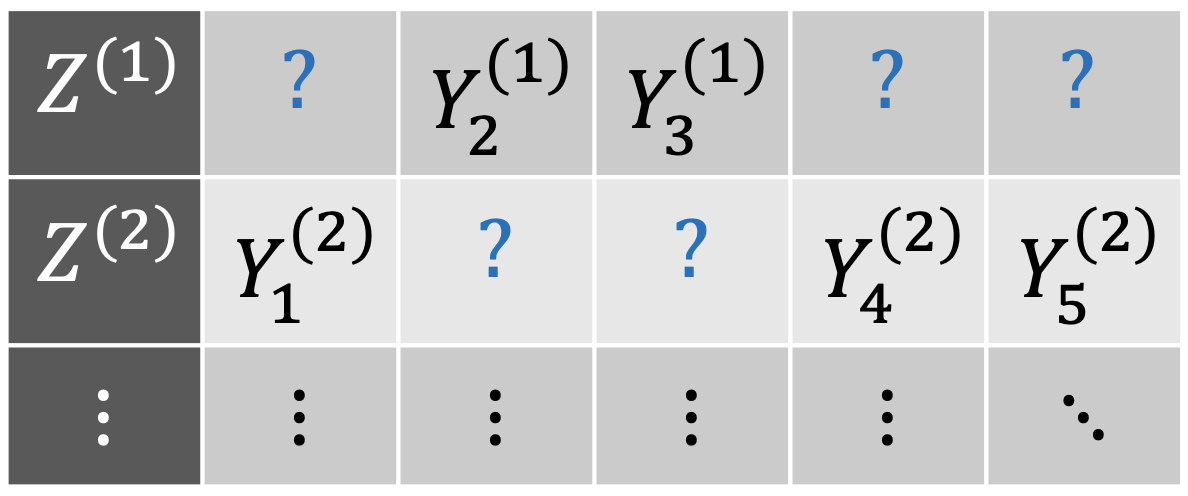}
     \caption{Original revelation order}
\end{subfigure}
\hfill
\begin{subfigure}[b]{0.48\textwidth}
    \centering
     \includegraphics[width=0.65\textwidth]{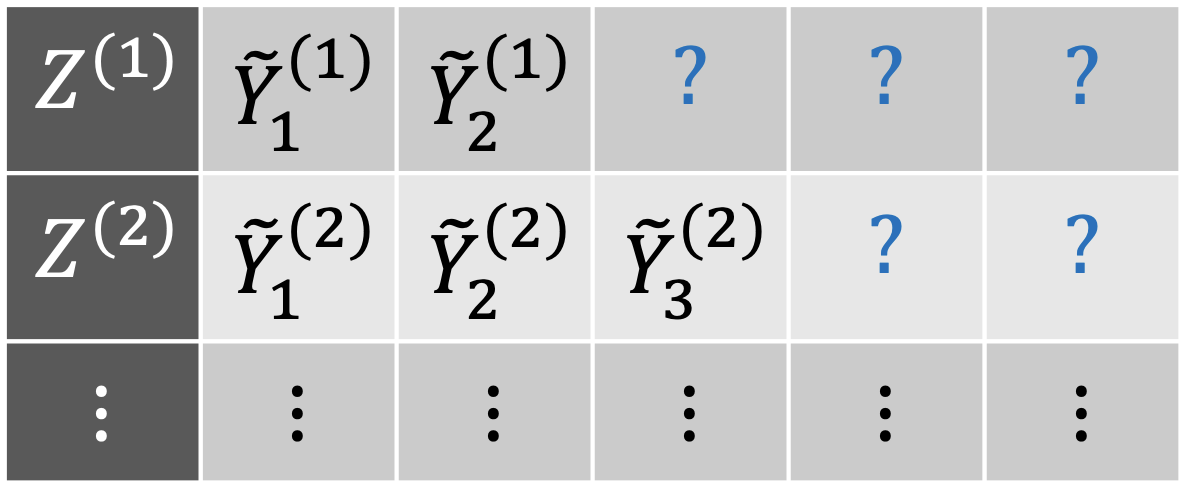}
     \caption{Alternative revelation order.}
 \end{subfigure}
\caption{Suppose the algorithm selects actions $A_1=2, A_2=1, A_3 = 1, A_4=2, A_5=2$. Above, we illustrate the potential outcomes that would be revealed under (a) our original revelation order, and (b) our alternative revelation order according to Definition \ref{def:non-exchangeable-dgp}.}
\label{fig:revelation}
\end{figure}

\vspace{-2mm}
\subsubsection{Completing the proof}
\begin{proof}
Recognize that the regret of any policy $\pi$ is a deterministic function of 
\begin{enumerate*}[label=(\roman*)] 
    \item The potential outcomes $\tilde{\tau}$, which determines the outcome and reward of any possible action selection;
    \item The prior information $Z^{(a)}\sim P_Z$ for $a\in \Aeval$ which may influence action selection;
    \item A random seed $\xi$ that determines any addition randomness in action selection. 
\end{enumerate*}  
That is, for some function $f(\cdot)$ with range $[0,1]$ the per-period regret of any policy $\pi$
under $\widehat{p}$ can be written as 
\[
\Delta( \pi ; \widehat{p}) = \E_{\widehat{p}} \left[f( Z, \tau, \xi ) \right]
\]
where we write $Z = \left(Z^{(a)}\right)_{a\in \Aeval}$. Then 
\begin{align*}
    &\sup_{f : \| f \|_\infty \leq 1} \bigg\{ \E_{p^*} \big[ f\left(  Z, \tilde{\tau}, \xi \right) \big] - \E_{\widehat{p}} \big[ f \left(  Z, \tilde{\tau}, \xi\right) \big] \bigg\} 
    \underbrace{\leq}_{(i)} \sqrt{ \frac{1}{2} D_{\rm KL} \big( \PP_{p^*}( Z, \tilde{\tau}, \xi) ~ \| ~ \PP_{\widehat{p}} ( Z, \tilde{\tau}, \xi) \big) } \\
    \underbrace{=}_{(ii)} &\,  \sqrt{ \underbrace{ \frac{1}{2} D_{\rm KL} \big( \PP_{p^*}\left( (  Z, \xi) \in \cdot \right) ~ \| ~ \PP_{\widehat{p}}\left( ( Z, \xi)\in \cdot \right) \big) }_{=0}
    + \frac{1}{2} D_{\rm KL} \big( \PP_{p^*}( \tilde{\tau} \in \cdot \mid Z,\xi) ~ \| ~ \PP_{\widehat{p}} (\tilde{\tau} \in \cdot \mid Z, \xi) \big) } \\
    \underbrace{=}_{(iii)} &\,  \sqrt{\frac{1}{2} D_{\rm KL} \big( \PP_{p^*}( \tilde{\tau} \in \cdot \mid Z) ~ \| ~ \PP_{\widehat{p}} (\tilde{\tau} \in \cdot \mid Z) \big) }  \\
    \underbrace{=}_{(iv)} &\,  \sqrt{\frac{1}{2}\sum_{a\in \Aeval} D_{\rm KL} \big( \PP_{p^*}( \tilde{Y}_{1:T}^{(a)} \in \cdot \mid Z^{(a)}) ~ \| ~ \PP_{\widehat{p}} (\tilde{Y}_{1:T}^{(a)} \in \cdot \mid Z^{(a)}) \big) }  
    \underbrace{=}_{(v)} \sqrt{ \frac{|\Aeval|}{2} \cdot \left( \ell(p^*) - \ell(\widehat{p}) \right) }
\end{align*}
Here \begin{enumerate*}[label=(\roman*)] 
    \item holds by Fact 9 in \citep{russo2016information} (which uses Pinsker's inequality).
    \item  holds by the chain rule for Kullback Liebler Divergence.
    \item holds because the draw of $(Z,\xi)$ does not depend on the sequence model used and because $\xi$ is independent of $(Z,\tilde{\tau})$. 
    \item  holds because $( Z^{(a)}, \tilde{Y}_{1:T}^{(a)} )$ are i.i.d. across $a \in \Aeval$. 
    \item holds by Lemma \ref{lem:log-loss-is-kl}.
\end{enumerate*}
\end{proof}

\vspace{-5mm}
\subsection{Proof of Corollary \ref{thm:psarRegret}}\label{app:proof_of_TS_regret}
\noindent  In light of Proposition \ref{prop:sim-to-real-gap}, the following proof is largely a review of an information-theoretic analysis of Thompson sampling due to \cite{russo2016information}. It was observed by \cite{bubeck2015bandit,bubeck2016multi} that this analysis can be applied without modification to analyze regret with respect to the  best fixed action ($A^*$) even in nonstationary environments (e.g. non-exchangeable models $\widehat{p}$ as in Definition \ref{def:non-exchangeable-dgp}.) In this section, for notational convenience, we use $p^*$ and $\widehat{p}$ to denote $p_{\TN{per-action}}^*$ and $\widehat{p}_{\TN{per-action}}$, respectively.

\begin{proof}
It follows from Proposition \ref{prop:sim-to-real-gap} that for any sequence model $\widehat{p}$,
\begin{align*}
    \Delta \big( \psar(\widehat{p} ) ; p^* \big) \leq 
    \Delta \big( \psar(\widehat{p}) ; \widehat{p} \big) 
    + \sqrt{ \frac{|\Aeval|}{2} \big\{ \ell(\widehat{p}) - \ell(p^*) \big\} }.
\end{align*}
Our goal is to show $\Delta\big( \psar(\widehat{p}) ; \widehat{p} \big) \leq \sqrt{ \frac{ |\Aeval| \cdot \log(|\Aeval|) }{2 T} }$. 
The key step is to recognize that under the data generating process corresponding to $\widehat{p}$ (see Definition \ref{def:simulator} and especially Definition \ref{def:non-exchangeable-dgp}), if Algorithm \ref{alg:posterior_sample} is applied to the input sequence model $\widehat{p}$ and input history $\HH_{t-1}$, then 
\begin{equation*}
    \PP_{\widehat{p}}\left( \widehat{\tau}_t \in \cdot \mid \HH_{t-1} \right) =  \PP_{\widehat{p}}\left( \tilde{\tau} \in \cdot \mid \HH_{t-1} \right). 
\end{equation*}
This follows by definition as Algorithm \ref{alg:posterior_sample} is constructed to sample directly from the posterior distribution of missing $\tilde{Y}$'s given past observations. 

With the definition of the optimal action $A^*=A^*(\tilde{\tau}) = \argmax_{a\in \Aeval} \sum_{t=1}^{T} R\left(\tilde{Y}_t^{(a)}\right)$, it follows immediately that Algorithm \ref{alg:Thompson_sampling_generic} satisfies
$\PP\left( A_t = a \mid \HH_{t-1}\right) = \PP_{\widehat{p}}\left( A^* = a \mid \HH_{t-1}\right) ~\text{for all } a \in \Aeval$.
In words, the sampling distribution of the action $A_t$ is \emph{matched} to the posterior distribution of $A^*$. From here, 
we use 
 Thompson sampling regret bound techniques from \citet{russo2016information}.
By the proof of Proposition 1 of \cite{russo2016information} (which is general and applies to all policies $\pi$), %
\begin{align*}
    \label{eqn:condZRegret}
    \tilde{\Delta} \big( \psar(\widehat{p} ); \widehat{p} \big)
    = \E_{\widehat{p}} \bigg[ \max_{a\in \Aeval} \bigg\{ \frac{1}{T} \sum_{t=1}^{T} R(Y_t^{(a)}) \bigg\} - \frac{1}{T} \sum_{t=1}^{T} R(\tilde{Y}_{t}) \bigg] 
    &\leq \sqrt{ \frac{ H_{\widehat{p}} (A^* ) \Gamma }{T}} 
    \leq \sqrt{ \frac{  \log(|\Aeval|) \Gamma }{T} }
\end{align*}
where $H_{\widehat{p}} (A^*) \leq \log(|\Aeval|)$ is the Shannon entropy of $A^*$ under the data generating process defined by $\widehat{p}$, and $\Gamma$ is a constant upper bound on the ``information ratio'' such that
\begin{align*}
    \Gamma \geq \max_{t \in [1 \colon T]} \bigg\{ \frac{ \big(\E_t \big[ \tilde{Y}_t^{(A^*)} - \tilde{Y}_t \big] \big)^2 }{ I_t \big( A^*; (A_t, \tilde{Y}_t) \big) } \bigg\} \quad \TN{w.p.~} 1,
\end{align*}
Above we use $\tilde{Y}_t^{(A^*)} \gets Y_{N_t^{(A^*)}}^{(A^*)}$ where $N_{t}^{(a)} \triangleq \sum_{i=1}^{t} \ind(A_t = a)$ denotes the number of times arm $a$ was played up to and including period $t$.
Above, we use $\E_t[\cdotspace] \triangleq \E_{p^*, \psar(\widehat{p})}[\cdotspace \mid \tilde{\HH}_{t-1}]$ to denote that expectations are conditioned on the history $\HHtilde_{t-1}$, and  $I_t \big( A^*; (A_t, \tilde{Y}_t) \big)$ to denote the mutual information between $A^*$ and $(A_t, \tilde{Y}_t)$ conditional evaluated under a base measure ($p^*, \psar(\widehat{p})$) that conditions on $\tilde{\HH}_{t-1}$.

The proof of Proposition 5 of \cite{russo2016information} shows that  one can choose $\Gamma \leq | \Aeval | / 2$ w.p. $1$. As observed in \cite{bubeck2015bandit,bubeck2016multi}, this proof relies only on the probability matching property in Lemma \ref{lem:prob_matching} and hence applies in our setting. 
Combining our results implies
\[
\tilde{\Delta} \big( \psar(\widehat{p} ); \widehat{p} \big)
\leq \sqrt{ \frac{  \log(|\Aeval|) \cdot |\Aeval| }{2 T} }.
\]
\end{proof}

\vspace{-5mm}
\subsection{Proof of Corollary \ref{cor:screening}}
\label{app:screening}

\begin{remark}
The proof technique is more general than the Corollary \ref{cor:screening} statement in two respects.  First, the pointwise bound~\eqref{eq:pointwise} holds for any measurable $S(z) \subseteq \cA$---not just the $1/T$-threshold---with a penalty proportional to $\PP(A^* \notin S(z)\mid Z\!=\!z)$.  Second, the same argument applies to the ``regret under simulator'' term $\Delta(\pi;\what{p}_{\mathrm{per\text{-}action}})$ in Proposition~\ref{prop:sim-to-real-gap}, yielding a version for approximate models; the statement is conceptually identical but notationally heavier since all probabilities are then evaluated under $\what{p}_{\mathrm{per\text{-}action}}$.
\end{remark}

\begin{proof}
Fix a realization $Z = z$ and put $B(z) = \cA \setminus S(z)$.  We bound the cumulative regret $\MC{R}_T := \sum_{t=1}^T \big( R \big(Y_t^{(A^*)} \big) - R \big(Y_t^{(A_t)} \big) \big)$ conditioned on $Z\!=\!z$, then average over $Z$ at the end.

Define the \emph{$z$-coarsened optimal arm} $\tilde{A}^* := A^* \cdot \mathbf{1}(A^* \in S(z)) + a_0 \cdot \mathbf{1}(A^* \notin S(z))$, taking values in $S(z) \cup \{a_0\}$ for a dummy action $a_0$.  Let $p_{t,a} := \PP(A^*\!=\!a \mid \cH_{t-1}, Z = z)$, $\E_t[\cdotspace] := \E[\cdotspace \mid \cH_{t-1}, Z = z]$, and $I_t(\cdotspace;\cdotspace)$ for mutual information under $\PP(\,\cdotspace \mid \cH_{t-1}, Z = z)$.

The conditional one-step regret decomposes as 
\begin{align*}
    \Delta_t := \E_t \Big[ R \big(Y_t^{(A^*)}\big) - R\big(Y_t^{(A_t)}\big) \Big] = \sum_{a \in \MC{A}} p_{t,a}\,\delta_{t,a}
\end{align*}
where $\delta_{t,a} = \E_t \big[R(Y_t^{(a)})\mid A^*\!=\!a \big] - \E_t \big[R(Y_t^{(a)}) \big]$; see \citet[Proposition~2]{russo2016information}.  Splitting $\Delta_t$ over $S(z)$ and $B(z)$ we have that $\Delta_t = \Delta_t^S + \Delta_t^B$, where $\Delta_t^S = \sum_{a \in \MC{A}} p_{t,a}\delta_{t,a} \mathbf 1(a\in S(z))$ and $\Delta_t^B = \sum_{a \in \MC{A}} p_{t,a}\delta_{t,a} \mathbf 1(a\in B(z))$.

\paragraph{Implausible
arms.} Since $\delta_{t,a} \le 1$ and since $P(A^*\!=\!a \mid Z\!=\!z) \le 1/T$ for each $a \in B(z)$,
\begin{align}
    \E \Bigg[\sum_{t=1}^T \Delta_t^B \, \bigg| \, Z\!=\!z \Bigg] \le T \cdot \PP(A^* \in B(z)\mid Z\!=\!z) \le |\cA| - |S(z)|.
    \label{eqn:deltaBbound}
\end{align}

\paragraph{Candidate arms.} %
The following holds by an argument akin to that for \citet[Proposition~3]{russo2016information}, which utilizes the Cauchy-Schwartz inequality, Pinsker's inequality (Fact 9) and \citet[Proposition 2]{russo2016information}.
\begin{align*}
    \E_t \left[ \Delta_t^S \right]^2 
    &= \bigg( \sum_{a \in S(z)} \PP_t(\tilde{A}^* = a ) \Big( \E_t \big[ R(Y_t^{(a)}) \mid \tilde{A}^* = a \big] - \E_t \big[ R(Y_t^{(a)}) \big] \Big) \bigg)^2 
    \leq \tfrac{|S(z)|}{2}\,I_t(\tilde{A}^*; (A_t,Y_t) ).
\end{align*}
Thus, we have that
\begin{align}
    \E \left[ \sum_{t=1}^T \Delta_t^S \,\Bigg|\, Z = z \right]
    &= \E \left[ \sum_{t=1}^T \E_t \left[ \Delta_t^S \right] \Bigg|\, Z = z \right] 
    \leq \E \left[ \sum_{t=1}^T \sqrt{ \tfrac{|S(z)|}{2}\,I_t(\tilde{A}^*; (A_t,Y_t)) } \, \Bigg|\, Z = z \right] \nonumber \\
    &\underbrace{\leq}_{(a)}  \sqrt{ \E \left[ \sum_{t=1}^T  \tfrac{|S(z)| T}{2}\,I_t(\tilde{A}^*; (A_t,Y_t)) \, \Bigg|\, Z = z \right] } \nonumber 
    \underbrace{\leq}_{(b)} \sqrt{ \tfrac{|S(z)| T}{2}\,H(\tilde{A}^* \mid Z = z ) } \nonumber \\
    &\leq \sqrt{\tfrac{|S(z)|\,T}{2}\,\log(|S(z)|+1)}
    \label{eqn:deltaSbound}
\end{align}
Above, inequality (a) holds by Cauchy-Schwartz over $t$, and inequality (b) holds since \\
$\E \left[ \sum_{t=1}^T I_t(\tilde{A}^*; (A_t,Y_t)) \right] \leq H(\tilde{A}^* \mid Z = z)$ by the chain rule for mutual information \cite[Fact~5]{russo2016information} and the non-negativity of entropy \cite[Fact~1]{russo2016information}. The final inequality holds because the coarsened best action $\tilde{A}^*$ takes at most $|S(z)|+1$ different values.
Combining \eqref{eqn:deltaBbound} and \eqref{eqn:deltaSbound},
\begin{align}
    \label{eq:pointwise}
  \E \big[\MC{R}_T \mid Z\!=\!z \big] 
  &= \E\Bigg[\sum_{t=1}^T \Delta_t^B \;\Bigg|\; Z\!=\!z \Bigg] + \E\Bigg[\sum_{t=1}^T \Delta_t^S \;\Bigg|\; Z\!=\!z \Bigg] \nonumber \\
  &\;\le\; \bigl(|\cA|-|S(z)|\bigr) + \sqrt{\tfrac{|S(z)|\,T}{2}\,\log(|S(z)|+1)}\,.
\end{align}

\paragraph{Averaging over $Z$.}  Divide~\eqref{eq:pointwise} by $T$ and take expectations.  Jensen's inequality ($\sqrt{\cdot}$ is concave) and $\log(|S(z)|+1) \le \log|\cA|$ give
\[
  \Delta(\pi_{\rm TSAR})
  \;=\; \tfrac{1}{T}\,\E[R_T]
  \;\le\; \frac{|\cA| - \E[|S(Z)|]}{T} + \sqrt{\frac{\E[|S(Z)|]\cdot\log|\cA|}{2T}}.
\]
\end{proof}

\vspace{-5mm}
\subsection{The tightness of our regret bound with misspecified priors}\label{app:lower_bound_instance}

In this section, for notational convenience, we use $p^*$ and $\widehat{p}$ to denote $p_{\TN{per-action}}^*$ and $\widehat{p}_{\TN{per-action}}$, respectively. In Example \ref{ex:prior-sensitivity} we showed the following bound on regret under Thompson sampling applied with a misspecified prior:
    \begin{align}\label{eq:mispecified-prior-regret-repeat}
        \Delta \big( \psar(\widehat{p} ); \, p^* \big) \leq 
        \underbrace{ \sqrt{ \frac{ |\Aeval| \log (|\Aeval|) }{2 T} } }_{\TN{Regret bound for Thompson sampling}}
        + \underbrace{ \sqrt{ \frac{ |\Aeval| }{2} D_{\rm KL}\big(  \nu_{*}  \; \big\| \; \nu_{\theta}  \big)}  }_{\TN{Penalty for mispecified prior}}.
    \end{align}
As noted below Example \ref{ex:prior-sensitivity}, our upper bound suggests that misspecified priors have persistent effects that do not diminish with $T$. The per-period expected regret remains inflated by a non-vanishing KL divergence term. Here we show this is unavoidable in general. 

\begin{example}
    Now we consider a special case of Example \ref{ex:prior-sensitivity} where the KL divergence term in \eqref{eq:mispecified-prior-regret-repeat} is nearly constant (just $O(\log T)$) and yet per-period regret is also inflated by a constant. 
    Let the latent variable $\eta^{(a)} \in \{1/4, 1/2, 3/4\}$ take on one of three values and take the likelihood $q(\cdot \mid \eta^{(a)}) = {\rm Bern}(\eta^{(a)})$  to be a Bernoulli distribution with success probability $\eta^{(a)}$. Take $R(y)=y$ so there is no distinction between rewards and outcomes. Let $Z^{(a)} \in \{ \mathtt{Safe}, \mathtt{Risky} \}$ take one of two values with equal probability. 
    \begin{enumerate}[leftmargin=*]
        \item When $Z^{(a)}=\mathtt{Safe}$ the arm is known to have 50\% success probability. That is, $\eta^{(a)}=1/2$ with probability 1 (under both $\nu_*$ and $\nu_\theta$).
        \item When $Z^{(a)}=\mathtt{Risky}$, the value of the latent success probability $\eta^{(a)}$ is uncertain. Under $p^*$, it is equally likely to take on a large or small value, with $\nu_*( \{3/4\} \mid \mathtt{Risky})=1/2= \nu_*( \{1/4\} \mid \mathtt{Risky})$. Under the misspecified process, risky arms are unlikely to have high success probability,  with $\nu_\theta( \{3/4\} \mid \mathtt{Risky})= 1/T^2$ and $\nu_\theta( \{1/4\} \mid \mathtt{Risky})=1-1/T^2$.
    \end{enumerate}
     Following the calculation in Example \ref{ex:prior-sensitivity}, one can conclude 
    \begin{align*}
    \ell(\widehat{p})-\ell(p^*)
    &\leq \underbrace{P_Z(\mathtt{Risky})}_{=1/2}\left( (1/2)\log\Big( \frac{1/2}{1/T^2} \Big) + (1/2)\log\Big( \frac{1/2}{1-1/T^2} \Big)  \right) 
    =O(\log(T)). 
    \end{align*}
    The idea of this construction is that $\mathtt{Risky}$ arms have a good chance of high upside under $p^*$, so an optimal policy has strong incentives to explore them in order to learn their true quality. In contrast, Thompson sampling with the misspecified model $\widehat p$ selects actions according to posterior samples from $\widehat p$, not from $p^*$. Under this misspecified posterior, each risky arm has probability only $1/T^2$ of having success probability $3/4$ and hence of beating the safe arms. 

    Since safe arms have known success probability $1/2$, a risky arm can be optimal in a posterior sample only if its sampled success probability exceeds $1/2$. In this construction, the only risky value exceeding $1/2$ is $\eta=3/4$. Therefore, misspecified Thompson sampling can select a risky arm only if at least one risky arm is sampled with $\eta=3/4$. Under $\widehat p$, this event has probability at most $|\Aeval|/T^2$ by a union bound. Hence, when $|\Aeval|=O(1)$, the probability of selecting a risky arm in any such period is $O(1/T^2)$. By a union bound, the probability that misspecified Thompson sampling ever plays a risky arm over $T$ periods is only $O(1/T)$.

    Regret is evaluated under the true model $p^*$. Under $p^*$, with constant probability there exists a risky arm with success probability $3/4$. On this event, the optimal expected reward is $3/4$, while misspecified Thompson sampling plays safe arms with probability $1-O(1/T)$ and obtains expected reward $1/2$. Therefore, for $|\Aeval|=O(1)$, the per-period regret of misspecified Thompson sampling is bounded below by a positive constant:
\[
\Delta \big( \psar(\widehat{p});\, p^* \big) \geq \Omega(1).
\]
\end{example}

\section{Experimental details}
\label{sec:exp_details}

\subsection{Synthetic Experiments: Mixture Beta-Bernoulli}
\label{app:synthetic_appendix}

\paragraph{Data Generating Process.}
We action attributes $Z^{(a)} = \big( Z^{(a)}_1, Z^{(a)}_2 \big) \in \real^2$ where $Z^{(a)}_1, Z^{(a)}_2 \iidsim \TN{Uniform}(0, 0.25)$. 
We sample $Y_{1:T}^{(a)}$ by first sampling a click rate $\mu^{(a)} \in [0, 1]$ from a mixture:  %
\begin{align*}
    \mu^{(a)} \mid Z^{(a)} \sim 
    \begin{cases}
        \TN{Beta} \big( 25 \cdot \frac{1}{4} Z_1^{(a)} + 1, ~ 25 \cdot (1-\frac{1}{4}Z_1^{(a)}) + 1 \big) & \TN{w.p. } 1/2 \\
        \TN{Beta} \big( 25\cdot (1-\frac{1}{4}Z_2^{(a)})  + 1, ~ 25\cdot \frac{1}{4}Z_2^{(a)} + 1 \big) & \TN{w.p. } 1/2
    \end{cases}
\end{align*}
Then, outcomes are sampled as $Y_1^{(a)}, \dots, Y^{(a)}_T \mid \mu^{(a)}, Z^{(a)} \iidsim \TN{Bernoulli}(\mu^{(a)})$. We use the identity reward mapping $R(y) = y$.

\paragraph{Additional training details.} 
The training and validation datasets contain 2500 and 1000 actions each, respectively. 
During offline learning (Section~\ref{sec:pretrainHistorical}), we train up to sequence lengths of $500$. 
Hyperparameters and early stopping is chosen with the validation set. 
\begin{itemize}[leftmargin=0.5cm]
    \item \textsc{Flexible NN.} This sequence model $p_\theta$ is based on an MLP. This model takes $Z^{(a)}$ and a summary statistic that represents the history of outcomes for that action as input, and outputs a probability in $[0,1]$. Since outcomes $Y_t^{(a)}$ are binary, we use the following summary statistic: $\big( \frac{1}{N^{(a)}} \sum_{t'=1}^{t-1} Y_{t'} \ind_{A_{t'} = a}, ~ \frac{1}{ 1 + N^{(a)} } \big)$, where $N^{(a)}\triangleq\sum_{t'=1}^{t-1} \ind_{A_{t'} = a}$. %
    \item \textsc{Beta-Bernoulli NN.} This is a sequential model based on the closed-form posterior predictive for a Beta-Bernoulli model; see \eqref{eqn:betaBernoulliPostpred2}. $\alpha_\theta(Z^{(a)})$ and $\beta_\theta(Z^{(a)})$ are each parameterized by separate MLPs that take in $Z^{(a)}$ and output a non-negative scalar. 
\end{itemize}
The MLPs above have three layers, width $50$, and ReLU activations. Models are trained for 1000 epochs with AdamW (learning rate 0.001, weight decay 0.01, batch size 500).

\subsubsection{Additional details on Figure \ref{fig:altSampling}: Comparing to alternative sampling approaches.}
\label{sec:altsampling_details}
Figure \ref{fig:altSampling} considers the synthetic decision-making setting described in this section. 
Figure \ref{fig:altSampling} considers the synthetic decision-making setting described in this section. All decision-making algorithms in Figure \ref{fig:altSampling} 
use the same pretrained sequence model $p_\theta$, with the \textsc{Flexible NN} architecture and trained as described in Section~\ref{sec:synthetic_experiments} and Appendix~\ref{app:synthetic_appendix}.

\begin{itemize}[leftmargin=*]
    \item \textsc{Greedy.} Select $A_t = \argmax_{a \in \MC{A}} \, \widehat{\mu}_{t,a}(p_\theta)$ where $\widehat{\mu}_{t,a}(p_\theta) = \int_y R(y) \cdot p_\theta( Y_t^{(a)} = y \mid \HH_{t-1}) dy$. For binary outcomes, $\widehat\mu_{t,a}(p_\theta)=p_\theta(Y_t^{(a)}=1\mid \mathcal H_{t-1})$. 
    \item \textsc{Epsilon-greedy.} Select the action chosen by \textsc{Greedy} 90\% of the time, and a uniformly random action for the remaining 10\% of the time. 
        \item \textsc{Generate one reward.} For each action $a \in \MC{A}$, sample $\widehat{Y}_t^{(a)} \sim p_\theta( Y_t^{(a)} \in \cdotspace \mid \HH_{t-1})$ and select $A_t = \argmax_{a \in \MC{A}} \, R (\widehat{Y}_t^{(a)})$.
    \item \textsc{Generate rewards i.i.d.} For each action $a \in \MC{A}$, sample 
    $\widehat{Y}_t^{(a)}, \dots, \widehat{Y}_{t+m}^{(a)} \iidsim p_\theta( Y_t^{(a)} \in \cdotspace \mid \HH_{t-1})$ for $m=500$ and select $A_t = \argmax_{a \in \MC{A}} \, \frac{1}{m} \sum_{i=0}^m R (\widehat{Y}_{t+m}^{(a)})$. %

    \item \textsc{Generate rewards autoregressively (TSAR).} Probabilistically impute the table of potential outcomes using Algorithm \ref{alg:posterior_sample}, but truncated to 500 imputed values as described in Section \ref{sec:truncationExp}, to form $\widehat{\tau}_t$. Then select $A_t=A^*(\widehat{\tau}_t)$. %

\end{itemize}

\subsection{News Recommendation Experiment Details}
\label{app:MIND_appendix}

\paragraph{Dataset preprocessing.} 
We process the MIcrosoft News Dataset (MIND) as follows:\footnote{This dataset is free to download for research purposes at \href{https://msnews.github.io/}{https://msnews.github.io/}.}
\begin{enumerate}[leftmargin=0.5cm]
\item Collect all articles from the MIND ``large'' dataset (training split only) \citep{wu2020mind}. Remove any article with fewer than 100 total impressions. 
\item For each article $a \in \MC{A}$, we use $\mu^{(a)}$ to denote its empirical click rate. We normalize the click rates to be centered around 0.5 in a way that preserves their relative rank to speed up the learning. %
The new success probabilities are as follows for each $a \in \MC{A}$:
\begin{align*}
    \mu^{(a)} \gets \begin{cases}
        \mu_0^{(a)} & \TN{if~} \mu_0^{(a)} \in \{ 0, 1 \} \\
        \text{logit}^{-1} \left(\text{logit}(\mu_0^{(a)}) - \bar{\mu}_0 \right) & \TN{otherwise}
    \end{cases}.
\end{align*}
Above, $\bar{\mu}_0 \triangleq \frac{1}{|\MC{A}|} \sum_{a' \in \MC{A}} \text{logit}(\mu_0^{(a')})$ and $\TN{logit}(x) \triangleq \log \frac{x}{1-x}$. 
See Figure~\ref{fig:transform_probs} for comparison of the success click rates before and after the transformation. 
\begin{figure}[h]
\centering
\vspace{-3mm}
\includegraphics[width=0.35\textwidth]{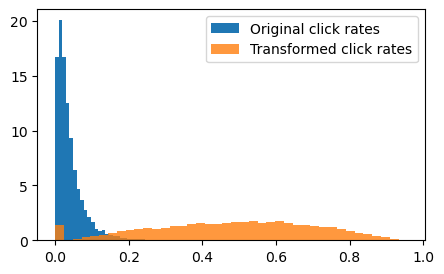}
\caption{Original and transformed click rates.}
\label{fig:transform_probs}
\vspace{-3mm}
\end{figure}
\item \vspace{-5mm} Randomly select 20\% of the remaining articles to be in the validation set (2280 articles); the rest are in the training set (9122 articles). %
\end{enumerate}

\paragraph{Additional training details.} 
During training, we train to sequence lengths of $500$. Hyperparameters and
early stopping epochs are chosen using the validation dataset. 
\begin{itemize}[leftmargin=*]
    \item \textsc{Flexible NN (text).} This model is very similar to the \textsc{Flexible NN} model in Appendix~\ref{app:synthetic_appendix}, except the MLP head of the neural network from before is fed as input a DistilBERT \citep{sanh2019distilbert} embedding of text data $Z^{(a)}$. The MLP linear layers have width 100 and the summary statistics are repeated 100 times. %
    \item \textsc{Beta-Bernoulli NN (text).} This is very similar to the Beta-Bernoulli posterior predictive sequence model in Appendix~\ref{app:synthetic_appendix}, with the exception that in place of a two-dimensional $Z^{(a)}$, the MLP head of the neural network from before is fed as input a DistilBERT \citep{sanh2019distilbert} embedding of text data $Z^{(a)}$. On top of the one DistilBERT embedding are two separate MLP heads for $\alpha(Z^{(a)})$ and $\beta(Z^{(a)})$, which are trained together. The summary statistics are repeated 100 times and MLP layers have width 50.
    \item \textsc{Flexible NN (category).} The model and training procedure are the same as for  \textsc{Flexible NN} from Appendix~\ref{app:synthetic_appendix}, except it uses a one-hot new category vector for $Z^{(a)}$.
\end{itemize}
The models \textsc{Flexible NN (text)} and \textsc{Beta-Bernoulli NN (text)}, which both incorporate the DistilBERT model, are each trained for 500 epochs with a learning rate of 1e-5 on MLP weights and 1e-8 on the DistilBERT weights using the AdamW optimizer.

\paragraph{Ensembles (Figure~\ref{fig:mind_evals}).}  
The ensembling approach used in the uncertainty quantification plots in Figure~\ref{fig:mind_evals} (right) takes a base model that is a DistilBERT that is trained to take in article headlines $Z^{(a)}$, which is fed into an MLP with a sigmoid head that predicts binary outcomes.%
\footnote{MLP width 100, 3 layers, learning rate 1e-5 on the head and 1e-8 on DistilBERT, weight decay 0.01, AdamW optimizer.} To save computation, we freeze the DistilBERT weights and just ensemble the weights of the MLP head (50 copies) using bootstrapped data. %

\paragraph{Additional randomized prior results.} 
In Figure~\ref{fig:randomized_prior_calibration}, we include a ``randomized prior'' ensembling approach which initializes each neural network model in the ensemble in a particular way to encourage diversity \citep{osband2018randomized}. We modify the ensembling procedure described earlier so that each of the 50 MLP heads incorporates their own prior networks, which we denote using $g_{1},\ldots,g_{50}$, which are MLP heads with the same architecture with randomly initialized weights. Note that the prior network weights are \textit{not} trained. Let $f_{\theta_1},\ldots,f_{\theta_{50}}$ denote the trainable MLP heads. The model predicts Bernoulli success probabilities using $\text{sigmoid}(f_{\theta_i}(\varphi(Z^{(a)}))+\sigma \cdot g_{i}(\varphi(Z^{(a)})))$, where %
$\varphi(Z^{(a)})$ denotes the trained DistilBERT embedding of the text headline $Z^{(a)}$. $\sigma$ is a prior scaling term hyperparameter.
\begin{figure}[H]
\includegraphics[width=0.45\linewidth]{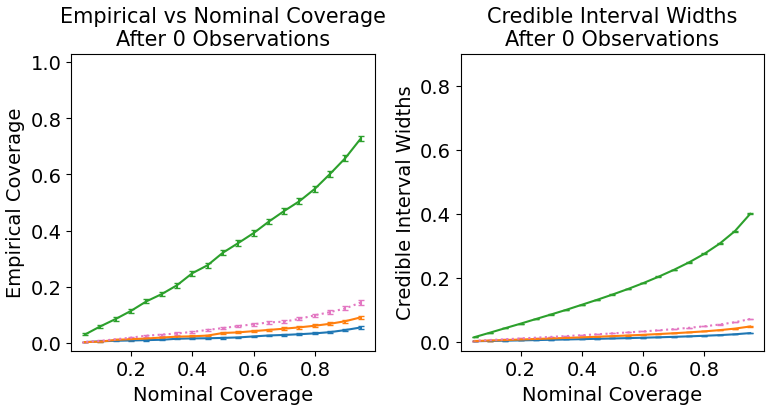}
\hspace{5mm}
\includegraphics[width=0.3\linewidth]{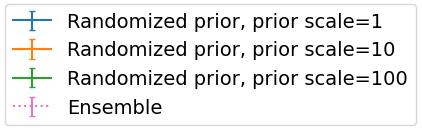}
\caption{Comparison of randomized prior ensembling methods \cite{osband2018randomized} after 0 observations.}
\vspace{-2mm}
\label{fig:randomized_prior_calibration}
\end{figure}
\vspace{-3mm}
Although tuning the prior scaling for the randomized prior can achieve approximately calibrated uncertainty (in contrast to regular ensembling), posterior sampling with autoregressive generation (PSAR) is advantageous over ensembling with randomized priors because i) PSAR simply updates its posterior beliefs by conditioning on the history, while ensembling must update models by fine-tuning the models via gradient updates on the new data, and ii) the pretraining procedure for PSAR automatically calibrates uncertainty well through standard sequence loss modeling.

\subsection{Baseline Bandit Algorithms}
\label{sec:bandit_baselines}

\paragraph{TS Beta Bernoulli (Uniform Prior).}
This algorithm uses Thompson sampling with a Beta-Bernoulli model with a uniform prior \citep{russo2020tutorial}. Note that unlike \textsc{TSAR} using the \textsc{Beta-Bernoulli NN}, the prior here does not depend on action attributes $Z^{(a)}$. 

\paragraph{TS Neural Linear.}
We implement a simplified version of ``neural linear'' \citep{riquelme2018deep,snoek2015scalable}. 
Here, we assume that for each arm, data are generated independently as 
$$Y^{(a)}_t=W^{(a)\top} \beta^{(a)} +\epsilon^{(a)}_t$$
where $\epsilon^{(a)}_t\sim N(0,\sigma^2)$ are drawn iid across arms $a$ and timesteps $t$, and $W^{(a)}:=[g(Z^{(a)}),1]^\top$
where $g(\cdot)$ is a model trained to predict reward from features $Z^{(a)}$ by using historical data (details below). 
Additionally, we assume
$$(\sigma^{(a)})^2\sim \text{InverseGamma}(a,b)\quad\text{ and }\beta^{(a)} \mid (\sigma^{(a)})^2\sim N([1,0]^\top, (\sigma^{(a)})^2 \mathbf{I}_2).$$
We set $a=3$ and $b=0.5$ since outcomes for each arm are Bernoulli and thus can have variance at most 1/4, and an Inverse Gamma with parameters $a=3$ and $b=0.5$ have reasonable mass under 1.4 while also being somewhat diffuse so that it does not make too strong an assumption on $\sigma^{(a)}$. 

For our synthetic experiments, we train $g$ using the same architecture and training procedure as we do for \textsc{Flexible NN} from Appendix \ref{app:synthetic_appendix} except it does not take in a sufficient statistic of the history. Similarly, for our news recommendations experiment, we train $ g$ using the same architecture and training procedure as we do for \textsc{Flexible NN} from Appendix \ref{app:MIND_appendix} (which incorporates a DistilBERT model) except it does not take in a sufficient statistic of the history. %

\paragraph{UCB.}
For UCB we use the multi-arm bandit algorithm described in Section 6 of \cite{Abbasi-YadkoriPaSz11}. We set the failure probability $\delta = 0.1$ and use sub-Gaussian parameter $0.5$ (since we have binary rewards).

\subsection{Modeling best action rather than modeling action reward}
\label{app:DPT}
\citet{lee2023incontext} introduce Decision Pretrained Transformers as an alternative approach of using sequence models to implement Thompson sampling. Their procedure trains a sequence model to predict the best action, mimicking an expert/optimal policy, given current state and recent interactions. This section aims to make a fair empirical comparison between our \textsc{TSAR} algorithm and the approach introduced in \citet{lee2023incontext}. %

\paragraph{Experiment details.}
In order to make the most fair comparisons between modeling best actions (which we call \textsc{TS Action}) and modeling rewards (\textsc{TSAR}), we use nearly identical model architecture and training procedures. The only adjustments we make are i) while the \textsc{TSAR} sequence models output a distribution over the next outcome $Y$, the TS Action sequence models output a softmax distribution over the set of action candidates, ii) the training target is set to the best action, and iii) we must sample offline trajectories collected by a behavior policy to use for training (``in-context datasets''). For collecting these offline trajectories we follow a procedure identical to that used in \citet{lee2023incontext} for the multi-armed bandit setting. %
We train these action-based sequence models for 2500 epochs and select the best learning rate out of $\{10^{-1}, 10^{-2},10^{-3},10^{-4}\}$.

\begin{figure}[H]
    \centering
    \begin{subfigure}{0.32\textwidth}
        \centering
        \includegraphics[width=\linewidth]{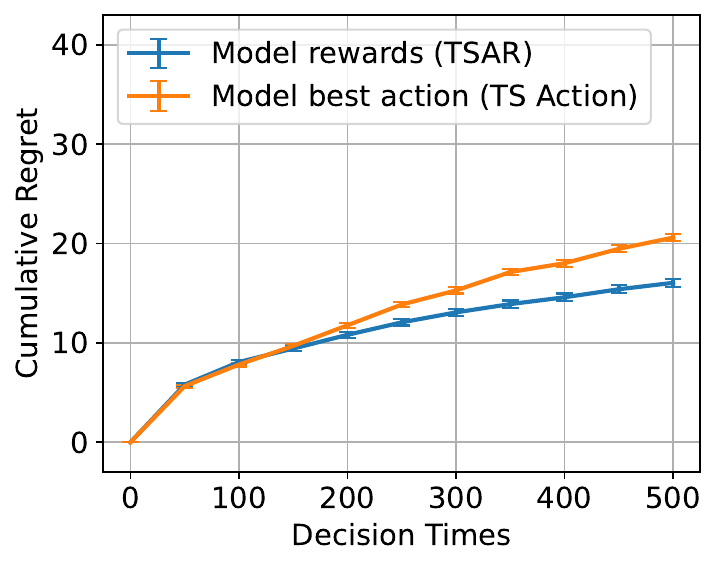} 
        \caption{$|\Ahist| = 2500$}
        \label{fig:dpt_psar_2500}
    \end{subfigure}
    \hfill
            \begin{subfigure}{0.32\textwidth}
        \centering
        \includegraphics[width=\linewidth]{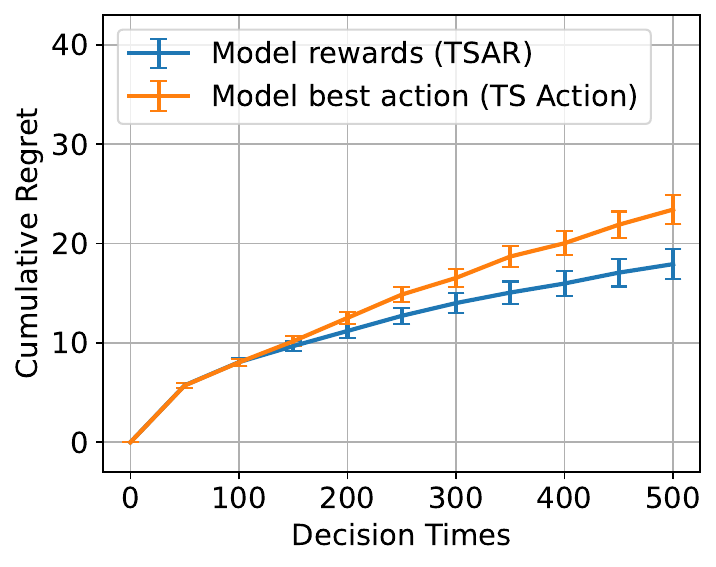} 
        \caption{$|\Ahist| = 100$}
        \label{fig:dpt_psar_100}
    \end{subfigure}
    \hfill
        \begin{subfigure}{0.32\textwidth}
        \centering
        \includegraphics[width=\linewidth]{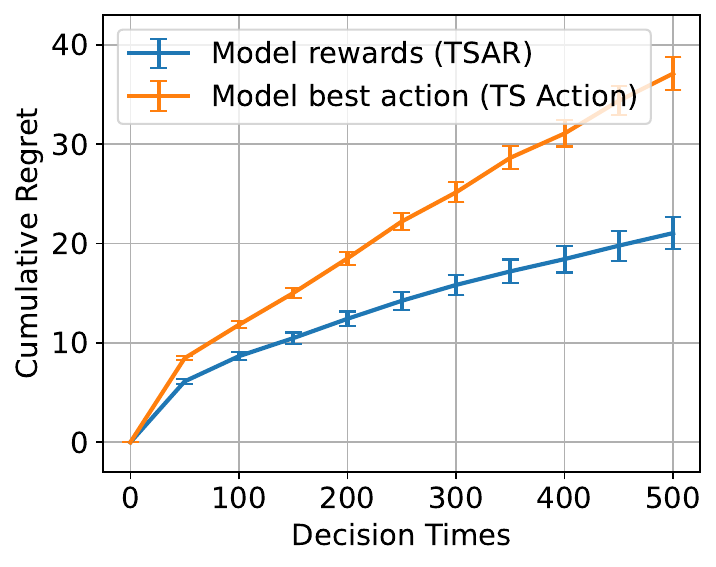}
        \caption{$|\Ahist| = 25$}
        \label{fig:dpt_psar_25}
    \end{subfigure}
    \caption{Synthetic setting: Regret for \textsc{TSAR} versus \textsc{TS Action} for different amounts of pretraining data.}
    \label{fig:dpt}
\end{figure}

\paragraph{Results.}
When training with $|\Ahist| = 2500$ actions in the offline pretraining dataset, there is little difference between \textsc{TSAR} and \textsc{TS-Action}. In contrast, \textsc{TSAR} is much better than \textsc{TS Action} when the number of offline training arms is drastically reduced, e.g. to 25 (Figure~\ref{fig:dpt_psar_25}). We hypothesize that this gap in performance is because the best action provides less information than the rewards for each action. For example, one action may have better rewards on average than another, but while \textsc{TSAR} can comprehend how much better the rewards are for action vs the other, \textsc{TS Action} only knows relative ordering. 
Further understanding the gap in performance is an interesting direction for future work.

\section{Bayesian sequence models}
\label{subsec:bayes-seq-models}

This section connects proposed offline learning procedure (for example, as in Section~\ref{sec:pretrainHistorical}) to empirical Bayes methods (Type-II maximum likelihood), which is an approach used to fit a Bayesian model to observed data \citep{pml1Book,casella1985introduction,normand1999meta}. This connection is precise in cases where $p_{\theta}$ is the posterior predictive distribution of some Bayesian model and $\theta$ consists of hyperparameters of the prior.  In other settings, where we fit parameters of a  more flexible (and not necessarily exchangeable) sequence model is used, this connection may help readers view our offline sequence model learning procedure as ``learning to form prior beliefs given $Z$, using historical data''.

Standard Bayesian mixture models define exchangeable sequence models. This is formalized in Example \ref{ex:bmm} below. 
\vspace{-2mm}
\begin{example}[Bayesian models define exchangeable sequence models]
    \label{ex:bmm}
    The model places a prior on a latent variable $\eta^{(a)}$, where $ \eta^{(a)} \sim P( \eta^{(a)} \in \, \cdot \, \mid Z^{(a)} )$, then $Y_1^{(a)}, \dots, Y_T^{(a)} \mid \eta^{(a)} \iidsim P( Y_t^{(a)} \in \, \cdot \, \mid \eta^{(a)})$.
    The $p^*$ sequence model associated with the above data generating process satisfies Assumption \ref{assump:exchangeable} ($p^*$ can be computed via Bayes rule).
\end{example}

To keep the presentation simple, we illustrate empirical Bayes in a simple case. Consider a Beta-Bernoulli Bayesian model without $Z$'s, which was introduced in \eqref{eq:ppd}:
\begin{align}
    \label{eqn:betaBernoulli}
    \mu^{(a)} \sim \TN{Beta}(\alpha^*, \beta^*) ~~\TN{then}~~
    Y_1^{(a)}, Y_2^{(a)}, \dots \big| \, \mu^{(a)} \iidsim \TN{Bernoulli}(\mu^{(a)}).
\end{align}
It is well known that the above Bayesian model~\eqref{eqn:betaBernoulli} is associated with the following class of autoregressive sequence models, i.e., the posterior predictive distribution:
\begin{align}
\label{eqn:betaBernoulliPostpred2}
p_{\theta} \big( Y_{t+1}^{(a)} = 1 \mid Y_{1:t}^{(a)} \big) =  \frac{ \alpha  + \sum_{i=1}^{t} Y_{i}^{(a)} }{ \alpha + \beta + t } \quad \TN{where} \quad \theta = (\alpha, \beta). 
\end{align}
Here, the true sequence model $p^*$ is exactly $p_\theta$ evaluated at $\theta = (\alpha^*, \beta^*)$. 
For the above sequence model class $p_\theta$, minimizing our negative log likelihood sequence training criterion \eqref{eq:train_loss} is equivalent to maximizing the marginal likelihood of the data, the criterion that is used in Empirical Bayes (Type-II maximum likelihood) to fit prior distributions to observed data \citep{pml1Book,casella1985introduction,normand1999meta}.

\section{Finite vs infinite population formulations and Thompson Sampling variants}
\label{app:finite_pop_TS}

This section discusses the intimate connections between (large) finite-population formulations that were discussed in the main body of the paper and infinite-population formulations that are more common in the Bayesian bandit literature. We do this in the special case of the Bayesian mixture model from Example \ref{ex:bmm}.

We emphasize that {\bf from our perspective, the main advantages or disadvantages of the finite population view are conceptual.} In terms of advantages: (1) the definitions do not require any explicit assumptions around mixture modeling or latent variables, and (2) the finite nature of the problem lets us visualize the procedure as in Figure  \ref{fig:autoregressive_generation}, without abstract reference to limits across infinite sequences.

\subsection{Review of Thompson sampling in infinite populations}
Thompson sampling is most often defined for a Bayesian mixture model, e.g., as in Example \ref{ex:bmm}. 
Following that example, we consider in the subsection the canonical example of exchangeable sequences: a mixture model wherein the outcomes are i.i.d. conditioned on a latent variable $\eta^{(a)}$. 
That is, $p^*(Y_1^{(a)}, \ldots, Y_t^{(a)} \mid Z^{(a)}) = \intop  \prod_{t=1}^{T} P(Y_t^{(a)} \mid \eta^{(a)}=u) P(\eta^{(a)} = u\mid  Z^{(a)}) du$.
The unknown latent variable represents the decision-maker's uncertainty about an action's performance. 

The literature typically defines the ``true arm means'' given $\eta^{(a)}$ as 
\[ 
\mu_{\infty}^{(a)} = \intop R(y) \cdot P(y \mid \eta^{(a)})\,dy.
\]
The subscript highlights that this has the interpretation of a long-run average reward across an infinite population of users (or infinite set of rounds). By the law of large numbers (applied conditional on %
$\eta^{(a)}$,
one has $\mu_{\infty}^{(a)} = \lim_{T\to \infty} \frac{1}{T} \sum_{t=1}^{T} R( Y_t^{(a)} )$.
The true best arm is defined as $A^*_\infty \in \argmax_{a\in \Aeval} ~ \mu_{\infty}^{(a)}$.
Randomness in the latent parameters $(\eta^{(a)})$ means $\mu_{\infty}^{(a)}$ and $A^*_\infty$ are random variables whose realizations are uncertain even given the history $\HH_{t-1}.$ Thompson sampling selects an action by probability matching on $A^*_{\infty}$, defined by the property $\PP(A_t=a \mid \HH_{t-1}) = \PP(A^*_{\infty} = a \mid \HH_{t-1}) \quad \text{for all } a \in \Aeval$.
The per-period Bayesian regret over $T$ periods is defined as 
\begin{equation}\label{eq:infinite_pop_regret}
\E\left[ \frac{1}{T}\sum_{t=1}^{T}\left( R \big( Y^{(A_{\infty}^*)}_t \big) - R \big( Y^{(A_t)}_t \big) \right)  \right]
\end{equation}

\subsection{Thompson sampling in finite populations}
One can define the true mean of a finite population as $\mu^{(a)}_T = \frac{1}{T} \sum_{t=1}^{T} R \big( Y_t^{(a)} \big)$.
Note in the main paper we use the notation $\mu^{(a)}_T$ to denote the above mean, but we change the notation here to make the contrast to $\mu_\infty^{(a)}$ more clear.
The true best arm for this finite population is defined as $A^* \in \argmax_{a\in \Aeval} ~ \mu^{(a)}_T$.
As in Lemma \ref{lem:prob_matching}, Thompson sampling selects an action by probability matching on the (finite-population) optimal action $A^*$, defined by the property $\PP(A_t=a \mid \HH_{t-1}) = \PP(A^* = a \mid \HH_{t-1}) \quad \text{for all } a \in \Aeval$.
The per-period Bayesian regret over $T$ periods is defined as 
\begin{equation}\label{eq:finite_pop_regret}
\E\left[ \frac{1}{T}\sum_{t=1}^{T}\left( R \big(Y^{(A^*)}_t \big) - R \big( Y^{(A_t)}_t \big) \right)  \right]
\end{equation}
It is not hard to show that \eqref{eq:finite_pop_regret} is a more stringent notion of regret than in \eqref{eq:infinite_pop_regret}, since \\ 
$ \frac{1}{T}\sum_{t=1}^{T} R \big( Y^{(A^*)}_t \big) \geq \frac{1}{T}\sum_{t=1}^{T} R \big( Y^{(A^*_\infty)}_t \big)$ by definition of $A^*$. Both definitions are widely used, with the more stringent finite-population version being more common in the adversarial bandit literature; see \cite{LattimoreSz19}.

\subsection{The gap between finite and infinite population formulations is small}
We analyze the gap between the two formulations in the case of a mixture model. 
By a sub-Gaussian maximal inequality
    \begin{align*}
     \E\left[ \max_{a\in \Aeval} \left| \mu_{\infty}^{(a)} -  \mu_T^{(a)} \right| \right] 
    = \E\left[ \E\left[ \max_{a\in \Aeval} \left| \mu_{\infty}^{(a)} - \mu_T^{(a)} \right|  \bigg| \, \{ \eta^{(a)} \}_{a\in \Aeval}  \right]\right] 
    \leq \sqrt{ \frac{2 \log (| \Aeval|)}{T} }.
    \end{align*}
To justify the last inequality, note that since the function $R$ takes values in $[0, 1]$, $R \big( Y_t^{(a)} \big) - \mu^{(a)}_{\infty}$ is sub-Gaussian with variance proxy $1$, conditional on $\eta^{(a)}$ (by Hoeffding's Lemma). Since it is the average of independent sub-Gaussian random  variables, 
$\mu_{\infty}^{(a)} - \mu_T^{(a)}$ is sub-Gaussian with variance proxy $\frac{1}{T}$, conditional on $\{ \eta^{(a)} \}_{a\in \Aeval}$. The last step follows then from applying the subgaussian maximal inequality,  conditional on $\{ Z^{(a)} \}_{a \in \MC{A}}, \{ \eta^{(a)} \}_{a\in \Aeval}$.

It follows easily that the infinite population optimum $A^*_{\infty}$ is near optimal for finite populations:
\[    0 \leq \E\left[ \max_{a\in \Aeval}  \mu_{T}^{(a)} -  \mu^{(A^*_\infty)}_T  \right]   
    \leq 2\sqrt{ \frac{2 \log (| \Aeval|)}{T} }.
\]
Analogously, the finite population optimum is near-optimal in infinite populations: \\
$0 \leq \E\left[ \max_{a\in \Aeval} ~ \mu_{\infty}^{(a)} -  \mu^{(A^*)}_\infty  \right]   
    \leq 2\sqrt{ \frac{2 \log (| \Aeval|)}{T} }$.
Supported by this theory, we do not focus on the distinction between $A^*$ and $A^*_\infty$.

\bibliography{bib,bib-hong}

\end{document}